\documentclass{article}

\usepackage{microtype}
\usepackage{graphicx}
\usepackage{subfigure}
\usepackage{booktabs} 

\usepackage{breakurl}

\usepackage{hyperref}

 \usepackage[accepted]{icml2025}

\usepackage{amsmath}
\usepackage{amssymb}
\usepackage{mathtools}
\usepackage{amsthm}

\usepackage[capitalize,noabbrev]{cleveref}

\DeclareMathOperator{\Hom}{Hom}

\DeclareMathOperator{\FundRep}{FundRep}

\usepackage{mathdots}
\usepackage{bm}
\usepackage{nicematrix}
\usepackage{xcolor}

\usepackage{tabularray}
\usepackage{comment}

\usepackage[title]{appendix}
\usepackage{appendix}
\usepackage{xurl}

\usepackage{ytableau}
\usepackage{tikzit}

\tikzstyle{node}=[fill=white, draw=black, shape=circle, minimum size=1mm, ultra thick]
\tikzstyle{small box}=[fill=white, draw=black, shape=rectangle, minimum height=0.5cm, minimum width=0.5cm, ultra thick]
\tikzstyle{weyl}=[fill=white, draw={rgb,255: red,0; green,0; blue,109}, shape=rectangle, minimum height=0.5cm, minimum width=0.5cm]
\tikzstyle{filled_node}=[fill=black, draw=black, shape=circle, minimum size=1mm, ultra thick]

\tikzstyle{thick}=[-, ultra thick]
\tikzstyle{blue_thick}=[-, ultra thick, draw=blue]
\tikzstyle{dashes}=[-, dashed, draw={rgb,255: red,191; green,191; blue,191}, dash pattern=on 2mm off 1mm, fill={rgb,255: red,244; green,228; blue,0}]
\tikzstyle{thick_arrow}=[ultra thick, ->]
\tikzstyle{dash_1}=[-, dashed]
\tikzstyle{dash_2}=[-, dashed, fill={rgb,255: red,246; green,235; blue,255}]
\tikzstyle{dash_3}=[-, dashed, fill={rgb,255: red,229; green,255; blue,181}]
\tikzstyle{dash_4}=[-, dashed, fill={rgb,255: red,255; green,209; blue,153}]
\tikzstyle{red_thick}=[-, ultra thick, draw=red]
\tikzstyle{dash_5}=[-, dashed, fill={rgb,255: red,225; green,255; blue,254}]
\tikzstyle{jellyfish}=[-, ultra thick, fill={rgb,255: red,34; green,48; blue,255}]

\usepackage{graphicx}

\theoremstyle{plain}
\newtheorem{theorem}{Theorem}[section]
\newtheorem{proposition}[theorem]{Proposition}
\newtheorem{lemma}[theorem]{Lemma}
\newtheorem{corollary}[theorem]{Corollary}
\theoremstyle{definition}
\newtheorem{definition}[theorem]{Definition}

\theoremstyle{remark}
\newtheorem{remark}[theorem]{Remark}
\newtheorem{example}[theorem]{Example}

\usepackage[most]{tcolorbox}

\newcommand{\norm}[1]{\left\lVert#1\right\rVert}

\usepackage[textsize=tiny]{todonotes}

\icmltitlerunning{Compact Matrix Quantum Group Equivariant Neural Networks}

\begin{document}

\twocolumn[
\icmltitle{Compact Matrix Quantum Group Equivariant Neural Networks}




\begin{icmlauthorlist}
\icmlauthor{Edward Pearce--Crump}{yyy}
\end{icmlauthorlist}

\icmlaffiliation{yyy}{Department of Computing, Imperial College London, United Kingdom}

\icmlcorrespondingauthor{Edward Pearce--Crump}{ep1011@ic.ac.uk}

\icmlkeywords{Machine Learning, ICML}

\vskip 0.3in
]



\printAffiliationsAndNotice{}  

\begin{abstract}
	Group equivariant neural networks have proven effective in modelling a wide
	range of tasks where the data lives in a classical geometric space and exhibits
	well-defined group symmetries.
	However, these networks are not suitable for learning from data 
	that lives in a non-commutative geometry,
	described formally by non-commutative $C^{*}$-algebras,
	since the $C^{*}$-algebra of 
	continuous functions on a compact matrix group is commutative.
	To address this limitation, we derive the existence of
	a new type 
	of equivariant neural network, 
	called compact matrix quantum group equivariant neural networks, 
	which encode 
	symmetries that are described by
	compact matrix quantum groups.
	We characterise the weight matrices that appear in these neural networks 
	for the easy compact matrix quantum groups, 
	which are defined by set partitions.
	As a result, we obtain new characterisations of equivariant weight matrices 
	for some compact matrix groups that have not appeared previously
	in the machine learning literature.
\end{abstract}

\section{Introduction} \label{qcintroduction}

Spaces in classical (commutative) geometry are given by 
sets of points with additional structure,
and their symmetries are formally described by groups.
For example, the compact matrix group $SO(3)$ is the group of all rotations 
of $\mathbb{R}^{3}$ that preserve the origin.
A class of neural networks called \textit{group equivariant neural networks}
was conceived to take advantage of group symmetries that are inherent in data 
by encoding them as an inductive bias in their architectures.
Many group equivariant neural networks have been designed using 
representations of groups as their layer spaces,
and characterisations of the weight matrices that appear in them 
have been found for a number of important groups 
\citep{deepsets, ravanbakhsh17a, maron2018, finzi, villar2021scalars, pearcecrumpB, pearcecrumpJ, pearcecrump, godfrey}.
These networks have proven to perform well across a wide range of tasks, including, but not restricted to, particle physics \citep{bogatskiy2020}, natural language processing \citep{gordon2020}, the generation of molecules \citep{satorras21a}, and computer vision \citep{chatzipantazis2023}.

However, when we consider tasks that involve learning from data that lives in a
\textit{non-commutative geometry},
such as 
learning 
properties of random variables in free probability theory
\citep{voiculescu1985, voiculescu1992},
physical phenomena in particle physics \citep{bhowmick2014, vansuijlekom2014},
and properties of exactly solvable models in statistical mechanics \citep{podles1995},
a group equivariant neural network is no longer a valid choice of model.
This is because
non-commutative geometry, as developed by \citet{Connes1995},
is formally described using non-commutative $C^{*}$-algebras,
which generalize commutative algebras $C(X)$
of continuous functions on compact Hausdorff spaces $X$,
and one can show that the $C^{*}$-algebra $C(G)$ of continuous functions 
on a compact matrix group $G$
is indeed commutative 
(Appendix \ref{exCstaralgebras}). 
In particular, in non-commutative geometry,
the classical notion of a space is replaced by an algebraic structure, 
thus making
traditional group-based symmetries inadequate 
\citep{Budzyński1995, Majid2000, Sitarz2013}.
Instead, a more general notion of symmetry is required.
Quantum groups, which have been developed by a number of authors 
\citep{Jimbo1985, drinfeld1986, woronowicz1987, woronowicz1991, woronowicz1998},
are an important class of mathematical objects for describing these symmetries.

Given that there exist symmetries that are not captured by groups,
this implies that there is a need for a new type of equivariant neural network.
In this paper, we 
focus on quantum groups that were constructed 
using a topological approach by \citet{woronowicz1987}.
Woronowicz generalised compact matrix groups 
to create a class of quantum groups
called 
\textit{compact matrix quantum groups}.
We describe and motivate this construction in Section \ref{noncommgeometry},
and provide a comprehensive background on the prerequisite mathematical material,
namely, general topology, $C^{*}$-algebras,
and operator theory, in the Appendix.

The main contributions of our work are as follows.
We broaden the scope of equivariant neural networks beyond classical symmetries
by defining a new class of neural networks 
that encode, as an inductive bias, 
symmetries 
described by compact matrix quantum groups.
We call these neural networks 
\textit{compact matrix quantum group equivariant neural networks}.
In fact, we mathematically derive their existence using Woronowicz's version of 
Tannaka-Krein duality for compact matrix quantum groups \citep{woronowicz1988}, 
which is introduced in Section~\ref{tannakakreinsection}.
After proving that these neural networks exist,
we characterise their weight matrices for
a large class of compact matrix quantum groups known as \textit{easy}.
We focus on easy compact matrix quantum groups 
not only because they have
been studied extensively in the literature 
\citep{banica, banica2010, weber2013, freslonweber2016, raum2016, tarrago2016, weber2017, tarrago2018, gromada2020, freslon2023},
but also because the way in which they are defined, 
namely, by set partitions,
makes it possible to obtain a full characterisation of the equivariant weight matrices 
for these compact matrix quantum groups.
We note that our contributions are primarily theoretical in nature: to demonstrate the 
practical potential of these neural networks, further work is needed to extend
the characterisation 
that we have found
for the easy compact matrix quantum groups
to the equivariant non-linear layers so that they can be implemented.

We note as a corollory of our work that
the new class of neural networks generalises and encompasses all
neural networks that are equivariant to compact matrix groups, since
we show, in Section \ref{noncommgeometry},
that compact matrix groups are, and can be expressed as, compact matrix quantum groups.
As a result, we recover characterisations of the weight matrices
that have appeared in \citet{ravanbakhsh17a, maron2018, pearcecrumpB, pearcecrump, godfrey},
and obtain characterisations of the equivariant weight matrices for three
compact matrix groups that have not appeared in the machine learning literature before,
namely the hyperoctahedral group $H_n$, the bistochastic group $B_n$, 
and the unitary group $U(n)$.


A word on notation: we use $[n]$ to refer to the set of elements $\{1, \ldots, n\}$, and, for any set $A$, we denote the set of $n \times n$ matrices with entries in $A$ by $M_n(A)$.


\section{Related Work}

A few prior related works \citep{hashimoto22, hataya2023, hashimoto24a}
use $C^*$-algebras to design new machine learning architectures, though their approach
is motivated by different objectives.
While they generalize complex-valued layer spaces and parameters to function-valued ones,
our approach extends the concept of group symmetry in group equivariant neural networks to a certain type of quantum group symmetry, thus expanding the range of equivariant models. 
Consequently, their models need to be trained with a custom gradient descent method, 
whereas 
our networks can still be trained 
with standard backpropagation as both the layer spaces and weight matrices remain
complex-valued.
Crucially, while both approaches can be used to construct 
group equivariant neural networks, their models do not encode quantum group symmetry, 
which is central to our framework.
Other works, such as those by \citet{ruhe23b, ruhe23a},
use Clifford algebras to develop equivariant
neural networks that encode symmetries related to rotations and reflections. 
However, their approach seems loosely related to ours, in that while both involve
non-commutative algebras, their networks capture symmetries in classical geometries 
whereas ours looks at symmetries in non-commutative geometries.
Finally, our approach differs from that of \citet{hoffmann2020}, where
they look to learn the algebraic structure from their network. 
By contrast, we show in Theorem \ref{easycstar} that
the algebraic structure for easy compact matrix quantum group equivariant neural networks
can be derived mathematically from the weight matrices themselves.



\section{Non-Commutative Geometry} \label{noncommgeometry}

In this section, we first introduce non-commutative geometry 
and then motivate the existence of compact matrix quantum groups 
by studying certain properties of compact matrix groups $G(n) \subseteq GL(n)$.

Two ideas were important in the development of non-commutative geometry.
The first was the paradigm shift 
from thinking in terms of points in a space (geometry)
to considering functions on spaces (algebra)
as the primary method for solving problems in geometry and topology 
\citep{Budzyński1995, Majid2000, Sitarz2013}.
Indeed, whilst one can show 
(Appendix \ref{exCstaralgebras}) 
that the algebra $C(X)$ of 
continuous, complex-valued functions on a compact Hausdorff space $X$ 
is a commutative, unital, $C^{*}$-algebra,
a famous theorem in $C^{*}$-algebra theory by \citet{gelfand1943}
shows that every commutative, unital, $C^{*}$-algebra is, in fact, an algebra of continuous
functions $C(X)$ for some compact Hausdorff space $X$.
In this way the geometric notion of a (compact Hausdorff) space $X$ is dualised to the algebra of continuous functions on $X$.


The second important development was the abstraction of spaces 
entirely by allowing unital $C^{*}$-algebras to be non-commutative 
\citep{Jimbo1985, drinfeld1986, woronowicz1987}.
In doing so, this created a notion of geometry through these algebras, namely by
treating non-commutative $C^{*}$-algebras
as if they were algebras of continuous functions 
on (non-existant) \textit{non-commutative spaces},
as an abstract extension of the theorem by \citet{gelfand1943}.
Consequently, non-commutative $C^{*}$-algebras are
the foundation of non-commutative geometry,
and, under this framework, the notion of a space (of points) is no longer meaningful. 

But with these developments, a new notion of symmetry for non-commutative 
$C^{*}$-algebras was needed, since, for any 
compact matrix group $G(n) \subseteq GL(n)$, 
one can show 
(Appendix \ref{exCstaralgebras}) 
that the algebra $C(G(n))$ 
is a commutative, unital $C^{*}$-algebra.

However, we can study the properties of $C(G(n))$ 
to motivate the definition of a compact matrix quantum group, as given by
\citet{woronowicz1987},
to describe this new type of symmetry.
Firstly, $G(n)$ is 
not just
a compact Hausdorff space, it is a group,
and so it comes with a group law $G(n) \times G(n) \rightarrow G(n)$
where $(g_1, g_2) \mapsto g_1g_2$.
We can express this group law in the language of $C^{*}$-algebras as follows:
define
$\Delta: C(G(n)) \mapsto C(G(n) \times G(n))$ 
under the mapping
$f \mapsto ((g_1, g_2) \mapsto f(g_1g_2))$.
Furthermore, since $C(G(n) \times G(n))$ is $^*$-isomorphic to $C(G(n)) \otimes C(G(n))$
under $((g_1, g_2) \mapsto f_1(g_1)f_2(g_2)) \cong f_1 \otimes f_2$,
we see that 
\begin{equation}
	\Delta: C(G(n)) \mapsto C(G(n)) \otimes C(G(n))
\end{equation}
is the expression of the group law of $G(n)$ in the language of $C^{*}$-algebras.

We generalise this construction by replacing $C(G(n))$ by a 
(potentially non-commutative) $C^{*}$-algebra $A$, to define a
compact matrix quantum group in the style of \citet{woronowicz1987}.

\begin{definition} \label{CMQGDefn}
	Let $A$ be a unital $C^{*}$-algebra, and let $u_{i,j} \in A$, for all $i, j \in [n]$, for some positive integer $n$. 
	Let $u$ be the $n \times n$ matrix whose $(i,j)$-entry is $u_{i,j}$, that is, 
	$u$ is an element of the unital $C^{*}$-algebra $M_n(A)$.
	The pair $(A, u)$ is said to be
	a \textbf{compact matrix quantum group}
	if
	\begin{enumerate}
		\item $A$ is the universal $C^{*}$-algebra $C^{*}(u_{i,j}, 1 \leq i,j \leq n)$,
		\item $u$ and $u^\top = (u_{j,i})$ are invertible matrices, and
		\item the comultiplication map $\Delta: A \rightarrow A \otimes_{\min} A$ defined by
			\begin{equation} \label{deltadefn}
				\Delta(u_{i,j})
				\coloneqq
				\left(
				\sum_k 
				u_{i,k} \otimes u_{k,j}
				\right)
			\end{equation}
		is a $*$-homomorphism. 
	\end{enumerate}

\end{definition}

\begin{remark}
	In motivating Definition \ref{CMQGDefn},
	we did not specify the tensor product in $C(G(n)) \otimes C(G(n))$
	for constructing $\Delta$ because $C(G(n))$ is a nuclear $C^{*}$-algebra
	\citep[Proposition 10.10]{courtney2023notes},
	implying that all norms that complete the algebraic tensor product 
	$C(G(n)) \odot C(G(n))$ into a $C^{*}$-algebra give the same one.
	However, in general, the algebraic tensor product of two $C^{*}$-algebras, 
	$A \odot B$,
	can be completed in many ways to form a $C^{*}$-algebra. 
	One method is to complete with respect to the minimal norm $\norm{\cdot}_{\min}$,
	giving the $C^{*}$-algebra $A \otimes_{\min} B$.
	We review these concepts in Section \ref{Cstarbackground} of the Appendix.
\end{remark}

Note that for a compact matrix group $G(n) \subseteq GL(n)$, points 1 and 2 of 
Definition \ref{CMQGDefn} make sense. 
Indeed, if we define functions
$u_{i,j} : G(n) \rightarrow \mathbb{C}$ for $i,j \in [n]$
such that $u_{i,j}(g) = g_{i,j}$, then,
by the Stone-Weierstrass Theorem (Appendix \ref{stoneweierstrass})
and the compactness of $G(n)$,
the $u_{i,j}$ generate $C(G(n))$.
Also, the matrices $u \coloneqq (u_{i,j})$ and $u^\top$ are invertible
since $u^{-1}(g) = u(g^{-1})$.
Furthermore, by defining the comultiplication map $\Delta$ for $C(G(n))$ by 
\begin{equation} \label{deltaCGndefn}
	\Delta(u_{i,j})(g,h) \coloneqq u_{i,j}(gh)
\end{equation}
we obtain the following result:
\begin{proposition}[Appendix: Section \ref{missingproofs}] \label{compmatgroup}
	Compact matrix groups $G(n) \subseteq GL(n)$ form a 
	special class of compact matrix quantum groups. 
\end{proposition}
Moreover, 
\citet[Theorem 1.5]{woronowicz1987} 
also showed, using the theorem by \citet{gelfand1943},
that if $(A, u)$ is a commutative compact matrix quantum group 
for some positive integer $n$,
then $A \cong C(G(n))$ for some compact matrix group $G(n) \subseteq GL(n)$.
Combining this with Proposition \ref{compmatgroup}, we obtain 
the Fundamental Theorem of Compact Matrix Quantum Groups: 
\begin{theorem} [Fundamental Theorem of Compact Matrix Quantum Groups]
	\label{cmqgfundamental}
	Let $(A, u)$ be a compact matrix quantum group for some $n \in \mathbb{N}$.
	Then $A$ is commutative if and only if $A \cong C(G(n))$ for some compact
	matrix group $G(n) \subseteq GL(n)$.
\end{theorem}

Hence, the Fundamental Theorem shows
	that if $(G(n), u)$ is non-commutative, 
	then we obtain ``true'' compact matrix quantum groups.
	Bearing this theorem in mind,
	convention dictates that we
	denote $A$ by $C(G(n))$, 
	even if $A$ is non-commutative,
	and so we often refer to $G(n)$, 
	or sometimes the pair $(G(n), u)$, as the compact matrix quantum group.

We now provide an example of a ``true'' compact matrix quantum group to show 
that they do indeed exist.

\begin{example} \label{exquantumgroups}
	Recall that each element of the symmetric group $S_n$, 
	considered as a subgroup of $GL(n)$,
	can be thought of as an $n \times n$ matrix
	having exactly one entry of $1$ 
	in each row and in each column, with all other entries being $0$.
	Writing the elements of $g \in S_n$ as $g_{i,j}$, this condition
	can be expressed as the following relations:
	\begin{equation} \label{relationsSnI}
		g_{i,j}^2 = g_{i,j} = g_{i,j}^*, \;
		\sum_{i=1}^{n} g_{i,j} = \sum_{j=1}^{n} g_{i,j} = 1
	\end{equation}
	and
	\begin{equation} \label{relationsSnII}
		g_{i,j}g_{k,l} = g_{k,l}g_{i,j}
	\end{equation}
	By the universal property of universal $C^*$-algebras 
	(Appendix \ref{unipropCstar}),
	we have that $C(S_n)$ is the universal $C^*$-algebra
	$C^*(E \mid R)$
	with generators $E$ given by $\{u_{i,j}, 1 \leq i,j \leq n\}$
	and relations $R$ given in (\ref{relationsSnI}) and (\ref{relationsSnII}) (replacing $g_{i,j}$ with $u_{i,j}$).
	\citet{wang1998} 
	defined the symmetric quantum group 
	as a compact matrix quantum group $S_n^+$ by \textit{liberating} $C(S_n)$
	of its commutativity relations to obtain $C(S_n^+)$
	as the universal $C^*$-algebra
	$C^*(E \mid R)$ where $E$ is the same as for $S_n$ but now the set of 
	relations $R$ is given by (\ref{relationsSnI}) only.
	As a result, the symmetric quantum group $S_n^+$ is
	our first example of a compact matrix quantum group 
	that is not a group. 
\end{example}


\section{Representation Theory of Compact Matrix Quantum Groups}

Analogous to compact matrix groups, compact matrix quantum groups 
have a representation theory. 
We now motivate the definition of a 
representation of a compact matrix quantum group.
Recall that for a compact matrix group $G(n)$, 
an $m$-dimensional representation is a choice of $m$-dimensional vector space $V$ 
over $\mathbb{C}$ and group homomorphism $\rho_V: G(n) \rightarrow GL(V)$.
By choosing a basis of $V$, we can replace $GL(V)$ by $GL(m)$.
The representation $\rho_V$ can itself be represented by functions 
$v_{i,j} \in C(G(n))$ 
for $1 \leq i, j \leq m$
that map $g \in G(n)$ to the $(i,j)$ entry of $\rho_V(g)$, 
since $\rho_V$ is continuous by the compactness of $G(n)$.
We claim the following result.
\begin{lemma}
	If $\Delta$ is the comultiplication on 
$C(G(n))$ defined in (\ref{deltaCGndefn}), then
	the homomorphism property of the representation satisfies
\begin{equation} \label{comultprop}
	\Delta(v_{i,j})(g, h)
	=
	\left(
	\sum_k 
	v_{i,k} \otimes v_{k,j}
	\right)
	(g, h)
\end{equation}
\end{lemma}

\begin{proof}
Indeed, we have that $\Delta(v_{i,j})(g, h)$ is equal to
\begin{align} 
	& v_{i,j}(gh)
	=
	[\rho_V(gh)]_{i,j}
	= 
	\sum_k 
	[\rho_V(g)]_{i,k}
	[\rho_V(h)]_{k,j} =  \nonumber \\
	&
	\sum_k 
	v_{i,k}(g)
	v_{k,j}(h)
	=
	\left(
	\sum_k 
	v_{i,k} \otimes v_{k,j}
	\right)
	(g, h) \nonumber \qedhere
\end{align}
\end{proof}

As a result, we have the following definition. 
\begin{definition}
	Let $(G(n), u)$ be a compact matrix quantum group. 
	Then an \bm{$m$}\textbf{-dimensional representation} of $G(n)$ 
	is an element $v \in M_m(C(G(n)))$ such that (\ref{deltadefn}) holds.
	If a representation $v$ of $G(n)$ has a matrix inverse, we say that 
	$v$ is a \textbf{non-degenerate} representation, and if $v$ is unitary,
	that is, $vv^* = 1 = v^*v$, then we call $v$ a \textbf{unitary} representation.
\end{definition}

\begin{remark}
	The matrix $u$ given in the definition of a compact matrix 
	quantum group $(G(n), u)$ is an $n$-dimensional representation 
	called the \textbf{fundamental representation} of $G(n)$, and
	by Definition \ref{CMQGDefn}, it is a non-degenerate representation of $G(n)$.
\end{remark}

Similar to groups, we can also define the tensor product and 
complex conjugation of representations, 
starting with the following definition.
\begin{definition}
	Let $v \in M_m(C(G(n)))$ and $w \in M_p(C(G(n)))$.
	Then the \textbf{tensor product} $v \otimes w \in M_{mp}(C(G(n)))$ is simply 
	the Kronecker product of matrices, 
	and the \textbf{complex conjugate} is $\bar{v} = (v_{i,j}^{*}) \in M_m(C(G(n)))$.
\end{definition}

\begin{proposition}[Appendix: Section \ref{missingproofs}] \label{tensconjcmqg}
	If $v, w$ are representations of a compact matrix quantum group $(G(n), u)$, 
	then both the tensor product and complex conjugate 
	are also representations of $G(n)$.
\end{proposition}

The following lemma will be useful when it comes to considering
the fundamental representation of a compact matrix quantum group.

\begin{lemma}[Appendix: Section \ref{missingproofs}] \label{conjunitary}
	Let $(G(n),u)$ be a compact matrix quantum group. Then
	$u^\top$ is unitary if and only if $\bar{u}$ is unitary.
\end{lemma}

We can also define the analogous concept of equivariance 
for compact matrix quantum groups.
\begin{definition} \label{quantumgroupequiv}
	Let $v \in M_m(C(G(n)))$ and $w \in M_p(C(G(n)))$ be representations of 
	a compact matrix quantum group $G(n)$.
	A map $\phi: \mathbb{C}^m \rightarrow \mathbb{C}^p$ is said to be 
	\textbf{$G(n)$-equivariant} (also known as an \textbf{intertwiner}) 
	if $\phi v = w\phi$.
	The set of all such \textit{linear} maps is denoted by 
	$\Hom_{G(n)}(v,w)$, and it is, in fact, a vector space.
\end{definition}

\begin{remark}
	It is worth noting that the expression $\phi v = w\phi$ given in 
	Definition \ref{quantumgroupequiv} has a different interpretation 
	relative to its sister definition for group equivariance.
	The representation $v \in M_m(C(G(n)))$ \textit{coacts} on the space 
	$\mathbb{C}^m$ in the following sense: if $x \in \mathbb{C}^m$, then 
	$vx \in \mathbb{C}^m \otimes C(G(n))$.
	Hence $v$ can be thought of as a map 
	$\mathbb{C}^m \rightarrow \mathbb{C}^m \otimes C(G(n))$ 
	or as an element of $M_m(\mathbb{C}) \otimes C(G(n))$.
	In fact, any representation on a vector space $V$ can be described by a 
	map $V \rightarrow V \otimes C(G(n))$.
	Consequently, $\phi \in \Hom_{G(n)}(v,w)$ if and only if 
	it satisfies the commutative square:
	\begin{equation}
		\scalebox{0.6}{\begin{tikzpicture}
	\begin{pgfonlayer}{nodelayer}
		\node [style=none] (0) at (0, 1.25) {};
		\node [style=none] (1) at (0, -1.75) {};
		\node [style=none] (2) at (4, -2.75) {};
		\node [style=none] (3) at (6.75, 1.25) {};
		\node [style=none] (4) at (0, 2.25) {\Large$\bm{\mathbb{C}^{m}}$};
		\node [style=none] (8) at (1, -2.75) {};
		\node [style=none] (9) at (6.75, -1.75) {};
		\node [style=none] (10) at (1, 2.25) {};
		\node [style=none] (11) at (4, 2.25) {};
		\node [style=none] (12) at (8, 2.25) {\Large$\bm{\mathbb{C}^{m} \otimes C(G(n))}$};
		\node [style=none] (13) at (8, -2.75) {\Large$\bm{\mathbb{C}^{p} \otimes C(G(n))}$};
		\node [style=none] (14) at (0, -2.75) {\Large$\bm{\mathbb{C}^{p}}$};
		\node [style=none] (15) at (2.5, 2.75) {\large$\bm{v}$};
		\node [style=none] (16) at (2.5, -2.25) {\large$\bm{w}$};
		\node [style=none] (17) at (7.75, -0.25) {\large$\bm{\phi}$};
		\node [style=none] (18) at (-1, -0.25) {\large$\bm{\phi}$};
	\end{pgfonlayer}
	\begin{pgfonlayer}{edgelayer}
		\draw [style={thick_arrow}] (10.center) to (11.center);
		\draw [style={thick_arrow}] (3.center) to (9.center);
		\draw [style={thick_arrow}] (8.center) to (2.center);
		\draw [style={thick_arrow}] (0.center) to (1.center);
	\end{pgfonlayer}
\end{tikzpicture}}
	\end{equation}
	where the right-hand arrow maps $x \otimes f \mapsto \phi{x} \otimes f$.
\end{remark}

\begin{definition}
	We say that two representations 
	$v \in M_m(C(G(n)))$ and $w \in M_m(C(G(n)))$
	are \textbf{equivalent} if there is an
	invertible linear map 
	$\phi: \mathbb{C}^m \rightarrow \mathbb{C}^m$
	that intertwines $v$ and $w$.
\end{definition}

One can show 
\citep[Proposition 2.2.3]{gromada2020}
that every non-degenerate representation is equivalent to a unitary representation.
Hence, going forward, we always choose 
the fundamental representation $u$ of a compact matrix quantum group 
and its transpose to be unitary.
By Lemma \ref{conjunitary}, we therefore have that both $u$ and $\bar{u}$ are unitary.


\section{Tannaka--Krein Duality for Two-Coloured Representation Categories} 
\label{tannakakreinsection}

Woronowicz's version of
Tannaka--Krein duality 
is an important result in the theory of compact matrix quantum groups.
Informally, the duality says that we can construct a compact matrix quantum group 
just by knowing its fundamental representation category.
We will use this duality to derive compact matrix quantum group equivariant neural networks
in Section \ref{CMQGNNs}.
We begin with some definitions.

\begin{definition} \label{twocolouredset}
	Consider the \textbf{two-coloured set}
	$\{\circ, \bullet\}$ 
	consisting of a white point and a black point.
	For any non-negative integer $k$,
	we can construct a \textbf{word} $w$ of length $k$ 
	as a string of $k$ colours from
	$\{\circ, \bullet\}$. 
	If $k = 0$, we define $\varnothing$ to be the empty word.
	The \textbf{concatenation} of two words, $w_1, w_2$, is the concatenation of their two strings and is written as $w_1 \cdot w_2$.
	We define a \textbf{homomorphism on words}, $w \mapsto \bar{w}$, by first defining it on the individual colours by
	$\bar{\circ} \coloneqq \bullet, \bar{\bullet} \coloneqq \circ$,
	and then applying it element wise to the word $w$.
	If $w$ is a word, then 
	its \textbf{involution} $w^{*}$ is the word read backwards together with its colours inverted.
\end{definition}

\begin{example}
	If $w_k \coloneqq \circ \bullet \bullet \circ \bullet \, \circ$ 
	is a word of length $6$ and
	$w_l \coloneqq \bullet \circ \bullet \bullet \circ$ is a word of length $5$, then
	$w_k \cdot w_l = 
	\circ \bullet \bullet \circ \bullet \circ \bullet \circ \bullet \bullet \circ$,
	$\bar{w_k} \coloneqq \bullet \circ \circ \bullet \circ \, \bullet$, and
	$w_k^{*} \coloneqq \bullet \circ \bullet \circ \circ \, \bullet$.
\end{example}

We can use words of colours to create tensor products of the fundamental representation of a compact matrix quantum group as follows.

\begin{definition}
	Let $(G(n), u)$ be a compact matrix quantum group. 
	Let $u^{\circ} \coloneqq u$ and $u^{\bullet} \coloneqq \bar{u}$.
	Then, for any word $w$ formed from the two-coloured set
	$\{\circ, \bullet\}$,
	we define $u^{\otimes w}$ to be the corresponding tensor product of representations $u^{\circ}$ and $u^{\bullet}$.

	Moreover, if $w_k$ and $w_l$ are words of lengths $k, l$ respectively, 
	then we define $\FundRep_{G(n)}(w_k, w_l)$ to be 
	$\Hom_{G(n)}(u^{\otimes w_k}, u^{\otimes w_l})$, that is, the vector space
	\begin{equation}
	\{\phi: (\mathbb{C}^n)^{\otimes k} \rightarrow (\mathbb{C}^n)^{\otimes l} \mid \phi{u^{\otimes w_k}} = u^{\otimes w_l}\phi\}
	\end{equation}
\end{definition}

\begin{example}
	If $w$ is the word $\circ \bullet \bullet \, \circ$, then
	$u^{\otimes w} = u \otimes \bar{u} \otimes \bar{u} \otimes u$.
\end{example}

We can also construct so-called representation categories using words of colours as follows.
\begin{definition} \label{twocolourrepcat}
	A \textbf{two-coloured representation category} $\mathcal{C}$ 
	is a collection 
	$\mathcal{C}(w_k, w_l)$, where $w_k, w_l$ are two words constructed from 
	$\{\circ, \bullet\}$ 
	of lengths $k, l$ respectively, 
	that are subspaces of the set of all linear maps 
	$(\mathbb{C}^n)^{\otimes k} \rightarrow (\mathbb{C}^n)^{\otimes l}$ and which satisfy
	the following axioms.

	\begin{enumerate}
		\item If $\phi_1 \in \mathcal{C}(w_k, w_l)$, 
			$\phi_2 \in \mathcal{C}(w_q, w_m)$,
			then $\phi_1 \otimes \phi_2 \in \mathcal{C}(w_k \cdot w_q, w_l \cdot w_m)$.
		\item If $\phi_1 \in \mathcal{C}(w_k, w_l)$, 
			$\phi_2 \in \mathcal{C}(w_l, w_m)$,
			then $\phi_2 \circ \phi_1 \in \mathcal{C}(w_k, w_m)$.
		\item If $\phi \in \mathcal{C}(w_k, w_l)$, 
			then $\phi^{*} \in \mathcal{C}(w_l, w_k)$.
		\item For every word $w$ (having some length $k$), we have $1_n^{\otimes k} \in \mathcal{C}(w, w)$, and
		\item The map $R : 1 \mapsto \sum_{i=1}^{n} e_i \otimes e_i$
			is in both 
			$\mathcal{C}(\varnothing, \circ\bullet)$ 
			and
			$\mathcal{C}(\varnothing, \bullet\circ)$.
	\end{enumerate}
\end{definition}

\begin{definition} \label{onecolouredrep}
	A \textbf{one-coloured representation category} is a two-coloured representation category where $u^{\circ} = u^{\bullet}$.
	Hence, the words correspond precisely to natural numbers
	because, for any natural number $k$, there is only one word of length $k$
	in this case.
\end{definition}

The following two theorems describe the relationship between two-coloured representation categories and compact matrix quantum groups.

\begin{theorem}[Appendix: Section \ref{missingproofs}]
	\label{cmqgtwocoloured}
	If $(G(n),u)$ is a compact matrix quantum group, 
	then $\FundRep_{G(n)}$ is a two-coloured representation category.
\end{theorem}

\begin{theorem}[\textbf{Woronowicz--Tannaka--Krein Duality}, \cite{woronowicz1988}]
	\label{tannakakrein}
	Let $\mathcal{C}$ be a two-coloured representation category.
	Then there exists a unique compact matrix quantum group $(G(n), u)$ 
	such that $\FundRep_{G(n)} = \mathcal{C}$.
\end{theorem}

\begin{remark}[Appendix \ref{tannakakreinfull}] \label{woronremark}
	There is a second part to Theorem \ref{tannakakrein}, which states
	how to construct the unique compact matrix quantum group $(G(n), u)$
	from $\mathcal{C}$. We provide this part in the Appendix 
	for brevity.
\end{remark}


\section{Compact Matrix Quantum Group Equivariant Neural Networks}
\label{CMQGNNs}

In this section we present the major contribution of this paper: 
we derive a new class of neural networks
which we term \textit{compact matrix quantum group equivariant neural networks}
using Woronowicz's formulation of Tannaka--Krein duality.
In the following definition, when referring to a word of length $k_l$, 
we write $w_l$ instead of $w_{k_l}$ to suppress multiple subscripts.

\begin{definition} \label{quantumGneuralnetwork}
	Let $(G(n), u)$ be a compact matrix quantum group with fundamental representation $u \in M_n(C(G(n)))$.
	A \textbf{compact matrix quantum group equivariant neural network} for $G(n)$
	is a function $f_{\mathit{NN}}$
	that is a composition of some $L$ layer functions
	\begin{equation}
		f_{\mathit{NN}} \coloneqq f_L \circ \ldots \circ f_{l} \circ \ldots \circ f_1
  	\end{equation}
	such that the $l^{\text{th}}$ layer function is equivariant to $G(n)$
	and is a function between representation spaces of $G(n)$ 
	under tensor products and complex conjugations of $u$
	\begin{equation}
		f_l: 
		((\mathbb{C}^{n})^{\otimes k_{l-1}}, u^{\otimes w_{l-1}})
		\rightarrow 
		((\mathbb{C}^{n})^{\otimes k_{l}}, u^{\otimes w_{l}})
	\end{equation}
	that is itself a composition
  	\begin{equation} \label{quantumGlayerfidefn}
	  	f_l \coloneqq \sigma_l \circ \phi_l
  	\end{equation}
	of a learnable, linear function
	in $\FundRep_{G(n)}(w_{l-1}, w_l)$	
	\begin{equation} \label{quantumGlayerlinear}
		\phi_l : 
		((\mathbb{C}^{n})^{\otimes k_{l-1}}, u^{\otimes w_{l-1}})
		\rightarrow 
		((\mathbb{C}^{n})^{\otimes k_{l}}, u^{\otimes w_{l}})
	\end{equation}
	that is equivariant to $G(n)$,
	together with a fixed, non-linear activation function 
	\begin{equation}	
		\sigma_l: 
		((\mathbb{C}^{n})^{\otimes k_{l-1}}, u^{\otimes w_{l}})
		\rightarrow 
		((\mathbb{C}^{n})^{\otimes k_{l}}, u^{\otimes w_{l}})
	\end{equation}
	that is also equivariant to $G(n)$.
	By choosing the standard basis of $\mathbb{C}^{n}$,
	the learnable, linear functions $\phi_l$ that are equivariant to $G(n)$
	become \textbf{matrices} that have \textbf{weights} that can be \textbf{learned}.
	As usual, we call these matrices \textbf{(equivariant) weight matrices}.
\end{definition}

\begin{remark}
	A compact matrix quantum group equivariant neural network 
	is well-defined since the fundamental representation category $\FundRep_{G(n)}$
	of a compact matrix quantum group $(G(n), u)$
	is a two-coloured representation category by
	Theorem \ref{cmqgtwocoloured}.
\end{remark}

We immediately obtain the following corollary for compact matrix groups
$G(n) \subseteq GL(n)$.

\begin{corollary} \label{cmgfromcmqgnn}
	Let $f_{NN}$ be a compact matrix quantum group equivariant neural network for a compact matrix \textbf{group} $(G(n),u)$.
	If all of the words that are used in the network only consist of the white point $\circ$, then $f_{NN}$ is, in fact, a compact matrix group equivariant neural network for $G(n)$.
	Moreover, every compact matrix group equivariant neural network is a compact matrix quantum group equivariant neural network.
\end{corollary}


\section{Easy Compact Matrix Quantum Groups}
\label{constructingCMQG}

Having just derived the existence of a compact
matrix quantum group equivariant neural network,
we show, in this section, that we can characterise 
the weight matrices that appear in these neural networks
for the so-called ``easy" compact matrix quantum groups.
A compact matrix quantum group is said to be \textbf{easy} if its two-coloured
representation category is determined under a functor by a 
so-called two-coloured category of partitions.
We first explore two-coloured categories of partitions, then define the functor,
and finally construct both the easy compact matrix quantum group and characterise
its equivariant weight matrices.

Our first goal is to define a two-coloured category of partitions. 
We need a few preliminary definitions.

\begin{definition}
	For all non-negative integers $l$ and $k$, 
	a \textbf{set partition} $\pi$ of $[l+k]$ is a partition of the set
	$[l+k]$ into a disjoint union of subsets, each of which is called a \textbf{block}.
	A \textbf{(one-coloured) $(k, l)$--partition diagram} $d_\pi$ 
	represents a set partition $\pi$ as a diagram, 
	where we place $l$ white nodes on the top row, 
	labelled left to right by $1, \dots, l$,
	where we place $k$ white nodes on the bottom row, 
	labelled left to right by $l+1, \dots, l+k$, and
	the edges between the vertices correspond to the 
	connected components of the set partition $\pi$.
\end{definition}

\begin{definition}
A \textbf{(two-coloured) $(w_k, w_l)$--partition diagram} is a $(k,l)$--partition diagram 
together with two-coloured words $w_k, w_l$ of lengths $k,l$ respectively such that
the vertices on the top row have the same colours as the word $w_l$, from left to right, and the vertices on the bottom row have the same colours as the word $w_k$, 
	from left to right.
We use the same notation $d_\pi$ to refer to the two-coloured version, since $w_k$ and $w_l$ will be known in making reference to the diagram.
We define the \textbf{two-coloured partition vector space} $P_{w_k}^{w_l}(n)$ to be the $\mathbb{C}$-linear span of the set of all $(w_k,w_l)$--partition diagrams.
\end{definition}

\begin{example} \label{twocolsetpart}
	Suppose that $\pi$ is the set partition of $[5 + 6]$ given by
\begin{equation} 
	\{1, 6 \mid 2, 3 \mid 4, 8 \mid 5, 9 \mid 7, 10 \mid 11\}
\end{equation}
Then 
	\begin{equation}
		\begin{aligned}
			\scalebox{0.5}{\begin{tikzpicture}
	\begin{pgfonlayer}{nodelayer}
		\node [style={filled_node}] (64) at (1.5, 0) {};
		\node [style={filled_node}] (66) at (7.5, 0) {};
		\node [style={filled_node}] (67) at (10.5, 0) {};
		\node [style={filled_node}] (75) at (3, -4) {};
		\node [style={filled_node}] (76) at (6, -4) {};
		\node [style={filled_node}] (78) at (12, -4) {};
		\node [style=none] (84) at (15, -5.25) {\Large$11$};
		\node [style=none] (94) at (1.5, 1.25) {\Large$1$};
		\node [style=none] (95) at (4.5, 1.25) {\Large$2$};
		\node [style=none] (96) at (7.5, 1.25) {\Large$3$};
		\node [style=none] (97) at (10.5, 1.25) {\Large$4$};
		\node [style=none] (98) at (13.5, 1.25) {\Large$5$};
		\node [style=none] (99) at (0, -5.25) {\Large$6$};
		\node [style=none] (100) at (3, -5.25) {\Large$7$};
		\node [style=none] (101) at (6, -5.25) {\Large$8$};
		\node [style=none] (102) at (9, -5.25) {\Large$9$};
		\node [style=none] (103) at (12, -5.25) {\Large$10$};
		\node [style=none] (104) at (-0.5, 0.5) {};
		\node [style=none] (105) at (-0.5, -4.5) {};
		\node [style=none] (106) at (15.5, -4.5) {};
		\node [style=none] (107) at (15.5, 0.5) {};
		\node [style=node] (108) at (4.5, 0) {};
		\node [style=node] (109) at (13.5, 0) {};
		\node [style=node] (110) at (9, -4) {};
		\node [style=node] (111) at (15, -4) {};
		\node [style=node] (112) at (0, -4) {};
	\end{pgfonlayer}
	\begin{pgfonlayer}{edgelayer}
		\draw [style={dash_2}] (105.center) to (106.center);
		\draw [style={dash_2}] (106.center) to (107.center);
		\draw [style={dash_2}] (107.center) to (104.center);
		\draw [style={dash_2}] (104.center) to (105.center);
		\draw [style=thick, bend left=45, looseness=0.50] (75) to (78);
		\draw [style=thick] (112) to (64);
		\draw [style=thick, bend right=60] (108) to (66);
		\draw [style=thick] (109) to (110);
		\draw [style=thick] (76) to (67);
	\end{pgfonlayer}
\end{tikzpicture}}
		\end{aligned}
	\end{equation}
	is the two-coloured $(w_k, w_l)$--partition diagram $d_\pi$
where
$w_k$ is the word 
$\circ \bullet \bullet \circ \bullet \, \circ$ and
$w_l$ is the word
$\bullet \circ \bullet \bullet \circ$.
\end{example}

We can now define a category for two-coloured partition diagrams, as follows:

\begin{definition} \label{cattwopart}
	$\mathcal{P}(n)$
	is the category whose objects are the two-coloured words
	and, for any pair of objects $w_k$ and $w_l$,
	the morphism space $\Hom_{\mathcal{P}(n)}(w_k,w_l)$
	is defined to be $P_{w_k}^{w_l}(n)$.

	We provide definitions of the vertical, horizontal and involution
	operations on morphisms in Section \ref{missingproofs} of the Appendix.
	The unit object is the empty word $\varnothing$.
\end{definition}

One can show the following result.

\begin{proposition}[Appendix: Section \ref{missingproofs}] \label{twocolmonoidal}
	$\mathcal{P}(n)$
	is a strict $\mathbb{C}$-linear monoidal category.
\end{proposition}

We achieve our first goal with the following definition.
\begin{definition} \label{twocolcatpart}
	A \textbf{two-coloured category of partitions} $\mathcal{K}(n)$ is any subcategory of $\mathcal{P}(n)$ such that 
	\begin{enumerate}
		\item if $w_k, w_l$ are objects in $\mathcal{K}(n)$, then
			$\Hom_{\mathcal{K}(n)}(w_k,w_l)$ 
			is a subspace of $\Hom_{\mathcal{P}(n)}(w_k,w_l)$, 
		\item the identity partition diagram is an element of 
			$\Hom_{\mathcal{K}(n)}(\circ,\circ)$,
		\item the top-row pair partition diagram corresponding to the set partition $\{1,2\}$ of $\{1,2\}$ superimposed, on the one hand,
			with the word $\circ\bullet$,
			and on the other,
			with the word $\bullet\circ$,
			is in 
			$\Hom_{\mathcal{K}(n)}(\varnothing,\circ\bullet)$
			and
			$\Hom_{\mathcal{K}(n)}(\varnothing,\bullet\circ)$, respectively,
			and
		\item the morphism spaces are closed
			under the vertical, horizontal and involution operations.
	\end{enumerate}
	By Proposition \ref{twocolmonoidal},
	it is immediate that a \textbf{two-coloured category of partitions} 
	$\mathcal{K}(n)$ is a strict $\mathbb{C}$-linear monoidal category.
\end{definition}

We now introduce the critical functor for our purposes: 
we show that this functor takes two-coloured $(w_k,w_l)$--partition diagrams $d_\pi$ 
that live in a two-coloured category of partitions $\mathcal{K}(n)$
to linear maps 
$\phi_\pi: (\mathbb{C}^n)^{\otimes k} \rightarrow (\mathbb{C}^n)^{\otimes l}$ 
that live in a two-coloured representation category.
By choosing the standard basis for $\mathbb{C}^{n}$, we obtain 
$n^l \times n^k$ matrices instead, which are denoted by $D_\pi$
in what follows.

\begin{definition} \label{partmatfunctor}
Suppose that $d_\pi$ is a $(w_k,w_l)$--partition diagram, where $w_k, w_l$ are two-coloured words of lengths $k, l$ respectively.
We define $D_\pi$ as follows.

Associate the indices $i_1, i_2, \dots, i_l$ with the vertices in the top row of $d_\pi$ and $j_1, j_2, \dots, j_k$ with the vertices in the bottom row of $d_\pi$.
	Then, if $S_\pi((I,J))$ is defined to be the set
	of tuples $(I,J) \in [n]^{l+k}$ such that
		\begin{equation} \label{Snindexingset}
			\text{if } x,y 
			\text{ are in the same block of } \pi, \text{then } i_x = i_y 
		\end{equation}
	(where we have momentarily replaced the elements of $J$ by $i_{l+m} \coloneqq j_m$ for all $m \in [k]$),
	 we define
	\begin{equation} \label{mappeddiagbasisSn}
		D_\pi
		\coloneqq
		\sum_{I \in [n]^l, J \in [n]^k}
		\delta_{\pi, (I,J)}
		E_{I,J}
	\end{equation}
	where
	\begin{equation}
		\delta_{\pi, (I,J)}
		\coloneqq
		\begin{cases}
			1 & \text{if } (I,J) \in S_\pi((I,J)) \\
			0 & \text{otherwise}
		\end{cases}
	\end{equation}
	and
	$E_{I,J}$ is the $n^l \times n^k$ matrix having a $1$ in the $(I,J)$ position and is $0$ elsewhere.
	We extend the definition of the map $d_\pi \mapsto D_\pi$ linearly to $P_{w_k}^{w_l}(n)$.
\end{definition}

\begin{example}
	If $d_\pi$ is the two-coloured 
$(\circ \bullet \bullet \circ \bullet \, \circ, \bullet \circ \bullet \bullet \circ)$--partition diagram given in Example \ref{twocolsetpart},
then, if $n = 3$, for example, we see from 
	\begin{equation}
		\begin{aligned}
			\scalebox{0.5}{\begin{tikzpicture}
	\begin{pgfonlayer}{nodelayer}
		\node [style={filled_node}] (64) at (1.5, 0) {};
		\node [style={filled_node}] (66) at (7.5, 0) {};
		\node [style={filled_node}] (67) at (10.5, 0) {};
		\node [style={filled_node}] (75) at (3, -4) {};
		\node [style={filled_node}] (76) at (6, -4) {};
		\node [style={filled_node}] (78) at (12, -4) {};
		\node [style=none] (84) at (15, -5.25) {\huge$\bm{2}$};
		\node [style=none] (94) at (1.5, 1.25) {\huge$\bm{1}$};
		\node [style=none] (95) at (4.5, 1.25) {\huge$\bm{2}$};
		\node [style=none] (96) at (7.5, 1.25) {\huge$\bm{2}$};
		\node [style=none] (97) at (10.5, 1.25) {\huge$\bm{1}$};
		\node [style=none] (98) at (13.5, 1.25) {\huge$\bm{3}$};
		\node [style=none] (99) at (0, -5.25) {\huge$\bm{1}$};
		\node [style=none] (100) at (3, -5.25) {\huge$\bm{2}$};
		\node [style=none] (101) at (6, -5.25) {\huge$\bm{1}$};
		\node [style=none] (102) at (9, -5.25) {\huge$\bm{3}$};
		\node [style=none] (103) at (12, -5.25) {\huge$\bm{2}$};
		\node [style=none] (104) at (-0.5, 0.5) {};
		\node [style=none] (105) at (-0.5, -4.5) {};
		\node [style=none] (106) at (15.5, -4.5) {};
		\node [style=none] (107) at (15.5, 0.5) {};
		\node [style=node] (108) at (4.5, 0) {};
		\node [style=node] (109) at (13.5, 0) {};
		\node [style=node] (110) at (9, -4) {};
		\node [style=node] (111) at (15, -4) {};
		\node [style=node] (112) at (0, -4) {};
	\end{pgfonlayer}
	\begin{pgfonlayer}{edgelayer}
		\draw [style={dash_2}] (105.center) to (106.center);
		\draw [style={dash_2}] (106.center) to (107.center);
		\draw [style={dash_2}] (107.center) to (104.center);
		\draw [style={dash_2}] (104.center) to (105.center);
		\draw [style=thick, bend left=45, looseness=0.50] (75) to (78);
		\draw [style=thick] (112) to (64);
		\draw [style=thick, bend right=60] (108) to (66);
		\draw [style=thick] (109) to (110);
		\draw [style=thick] (76) to (67);
	\end{pgfonlayer}
\end{tikzpicture}}
		\end{aligned}
	\end{equation}
that the $(1, 2, 2, 1, 3 \mid 1, 2, 1, 3, 2, 2)$-entry of $D_\pi$ is $1$,
whereas, we see that, from 
	\begin{equation}
		\begin{aligned}
			\scalebox{0.5}{\begin{tikzpicture}
	\begin{pgfonlayer}{nodelayer}
		\node [style={filled_node}] (64) at (1.5, 0) {};
		\node [style={filled_node}] (66) at (7.5, 0) {};
		\node [style={filled_node}] (67) at (10.5, 0) {};
		\node [style={filled_node}] (75) at (3, -4) {};
		\node [style={filled_node}] (76) at (6, -4) {};
		\node [style={filled_node}] (78) at (12, -4) {};
		\node [style=none] (84) at (15, -5.25) {\huge$\bm{2}$};
		\node [style=none] (94) at (1.5, 1.25) {\huge$\bm{1}$};
		\node [style=none] (95) at (4.5, 1.25) {\huge$\bm{1}$};
		\node [style=none] (96) at (7.5, 1.25) {\huge$\bm{2}$};
		\node [style=none] (97) at (10.5, 1.25) {\huge$\bm{1}$};
		\node [style=none] (98) at (13.5, 1.25) {\huge$\bm{3}$};
		\node [style=none] (99) at (0, -5.25) {\huge$\bm{2}$};
		\node [style=none] (100) at (3, -5.25) {\huge$\bm{2}$};
		\node [style=none] (101) at (6, -5.25) {\huge$\bm{1}$};
		\node [style=none] (102) at (9, -5.25) {\huge$\bm{3}$};
		\node [style=none] (103) at (12, -5.25) {\huge$\bm{1}$};
		\node [style=none] (104) at (-0.5, 0.5) {};
		\node [style=none] (105) at (-0.5, -4.5) {};
		\node [style=none] (106) at (15.5, -4.5) {};
		\node [style=none] (107) at (15.5, 0.5) {};
		\node [style=node] (108) at (4.5, 0) {};
		\node [style=node] (109) at (13.5, 0) {};
		\node [style=node] (110) at (9, -4) {};
		\node [style=node] (111) at (15, -4) {};
		\node [style=node] (112) at (0, -4) {};
	\end{pgfonlayer}
	\begin{pgfonlayer}{edgelayer}
		\draw [style={dash_2}] (105.center) to (106.center);
		\draw [style={dash_2}] (106.center) to (107.center);
		\draw [style={dash_2}] (107.center) to (104.center);
		\draw [style={dash_2}] (104.center) to (105.center);
		\draw [style=thick, bend left=45, looseness=0.50] (75) to (78);
		\draw [style=thick] (112) to (64);
		\draw [style=thick, bend right=60] (108) to (66);
		\draw [style=thick] (109) to (110);
		\draw [style=thick] (76) to (67);
	\end{pgfonlayer}
\end{tikzpicture}}
		\end{aligned}
	\end{equation}
	the $(1, 1, 2, 1, 3 \mid 2, 2, 1, 3, 1, 2)$-entry of $D_\pi$ is $0$.
\end{example}

Consequently, we obtain the following key result,
which makes it possible for us to construct a two-coloured representation category
from any two-coloured category of partitions.

\begin{theorem} 
	\label{quantumpartfunctor}
	For any two-coloured category of partitions $\mathcal{K}(n)$,
	we can form a category $\mathcal{C}_{\mathcal{K}}(n)$
	whose objects are the same as those in $\mathcal{K}(n)$,
	and, for any pair of objects $w_k$ and $w_l$,
	the morphism space $\Hom_{\mathcal{C}_{\mathcal{K}}(n)}(w_k,w_l)$
	is the image of $\Hom_{\mathcal{K}(n)}(w_k,w_l)$
	under $d_\pi \mapsto D_\pi$, that is
	$\Hom_{\mathcal{C}_{\mathcal{K}}(n)}(w_k,w_l)$ is defined to be
	\begin{equation} \label{easyspanningset}
		\{D_\pi \mid d_\pi \in \Hom_{\mathcal{K}(n)}(w_k,w_l)\}
	\end{equation}
	Then the map $d_\pi \mapsto D_\pi$ defines a 
	strict $\mathbb{C}$--linear monoidal functor
	\begin{equation} \label{qufunctor}
		\Theta : \mathcal{K}(n) \rightarrow \mathcal{C}_{\mathcal{K}}(n)
	\end{equation}
	such that $\mathcal{C}_{\mathcal{K}}(n)$ is a two-coloured representation category.
	[Appendix: Section \ref{missingproofs}].
\end{theorem}

Hence we can use Woronowicz-Tannaka-Krein duality given in Theorem \ref{tannakakrein}
to not only construct the easy compact matrix quantum group but also to characterise
the weight matrices that appear in the neural network that was derived in Definition
\ref{quantumGneuralnetwork}, as follows.

\begin{theorem} \label{easycstar}
	A two-coloured category of partitions 
	$\mathcal{K}(n)$
	determines a unique compact matrix quantum group $(G(n), u)$.
	We call these compact matrix quantum groups \textbf{easy},
	and they can be constructed as the universal $C^{*}$-algebra
	$C^{*}(E \mid R)$, where $E = \{u_{i,j}, i,j \in [n]\}$ and
	$R = \{D_\pi{u^{\otimes w_k}} = {u^{\otimes w_l}}D_\pi 
		\text{ for all } d_\pi \in \mathcal{K}(n)(w_k,w_l)\}$.
	The weight matrix $W$ for the learnable linear function 
	(\ref{quantumGlayerlinear})	
	in the neural network that is equivariant to $G(n)$ is given by
	\begin{equation} \label{cmqgweightmatrix}
		W = \sum_{\pi \mid d_\pi \in \mathcal{K}(n)(w_{l-1},w_l)} w_\pi D_\pi
	\end{equation}
\end{theorem}


\begin{figure*}[tb]
	\begin{tcolorbox}[colback=teal!10, colframe=teal!30, coltitle=black, 
		title={\bfseries Procedure 1: How to Calculate
		the Weight Matrix of
		an Equivariant Linear Layer Function from 
		$((\mathbb{C}^{n})^{\otimes k}, u^{\otimes w_{k}})
		\rightarrow 
		((\mathbb{C}^{n})^{\otimes l}, u^{\otimes w_{l}})$
		for an Easy Compact Matrix Quantum Group $(G(n), u)$.},
	fonttitle=\bfseries]
	Assume that 
		$\mathcal{K}(n)$
		is
		the two-coloured category of partitions 
		that defines the easy compact matrix quantum group $(G(n), u)$ as
		in Theorem \ref{easycstar}. \\

		Perform the following steps:
	\begin{enumerate}
		\item Calculate all of the two-coloured $(w_k, w_l)$--partition diagrams
			$d_\pi$ that live in $\mathcal{K}(n)$.
			These diagrams form a basis of the morphism space $\Hom_{\mathcal{K}(n)}(w_k,w_l)$.
	\item Apply the function $d_\pi \mapsto D_\pi$ to 
		each diagram to obtain its associated spanning set matrix $D_\pi$.
	\item Attach a weight $w_\pi \in \mathbb{C}$ to each matrix $D_\pi$.
	\item Finally, calculate $\sum w_\pi D_\pi$ to give the overall weight matrix.
	\end{enumerate}
	\end{tcolorbox}
  	\label{qgroupsummaryprocedure}
\end{figure*}

\section{Examples} \label{CMQGExamples}

We noted in the Introduction that 
easy compact matrix quantum groups have been well studied in the literature
and many classes of them have been characterised in full.
We introduce a few examples, 
characterise their weight matrices,
and refer the reader to the referenced works for more details.

Note in the following that it is enough to provide the spanning set of matrices
(\ref{easyspanningset}) for an easy compact matrix quantum group, since, 
by (\ref{cmqgweightmatrix}), 
the weight matrix is a weighted linear combination of these
matrices.

Firstly, we are able to characterise the equivariant weight matrices 
for some compact matrix groups
that have not appeared previously in the machine learning literature by
combining Theorem \ref{cmqgfundamental}
and Corollary~\ref{cmgfromcmqgnn}. 
Namely, if we consider one-coloured partition categories
that contain the swap partition diagram 
	\begin{equation} \label{swappart}
		\begin{aligned}
			\scalebox{0.5}{\begin{tikzpicture}
	\begin{pgfonlayer}{nodelayer}
		\node [style=none] (94) at (0, 1.25) {\Large$1$};
		\node [style=none] (95) at (4, 1.25) {\Large$2$};
		\node [style=none] (96) at (0, -5.25) {\Large$3$};
		\node [style=none] (97) at (4, -5.25) {\Large$4$};
		\node [style=none] (104) at (-0.5, 0.5) {};
		\node [style=none] (105) at (-0.5, -4.5) {};
		\node [style=none] (106) at (4.5, -4.5) {};
		\node [style=none] (107) at (4.5, 0.5) {};
		\node [style=node] (108) at (0, 0) {};
		\node [style=node] (109) at (4, 0) {};
		\node [style=node] (111) at (4, -4) {};
		\node [style=node] (112) at (0, -4) {};
	\end{pgfonlayer}
	\begin{pgfonlayer}{edgelayer}
		\draw [style={dash_2}] (105.center) to (106.center);
		\draw [style={dash_2}] (106.center) to (107.center);
		\draw [style={dash_2}] (107.center) to (104.center);
		\draw [style={dash_2}] (104.center) to (105.center);
		\draw [style=thick] (108) to (111);
		\draw [style=thick] (112) to (109);
	\end{pgfonlayer}
\end{tikzpicture}}
		\end{aligned}
	\end{equation}
	then, by Theorem \ref{easycstar},
we obtain easy compact matrix groups (as the $C^*$-algebra is commutative) 
and spanning sets of matrices (\ref{easyspanningset}), which allow us to
characterise their weight matrices by (\ref{cmqgweightmatrix}).
The groups themselves were characterised by \citet{banica}.
We highlight two examples here that are new, and provide
more examples in the Appendix that recover characterisations 
that were known to the machine learning
community before. 
\begin{example}
	The first is the \textbf{hyperoctahedral group} $H_n$, the symmetry group
	of the hypercube. \citet{banica} showed that the spanning set of matrices
	(\ref{easyspanningset}) is the image under $d_\pi \mapsto D_\pi$ 
	of all $(k,l)$--partition diagrams whose blocks have even size.
\end{example}
\begin{example}
	The second is the \textbf{bistochastic group} $B_n$, 
	which is the group of orthogonal
	matrices having sum $1$ in each row and column. Again, 
	\citet{banica} showed that the spanning set of matrices
	(\ref{easyspanningset}) is the image under $d_\pi \mapsto D_\pi$ 
	of all $(k,l)$--partition diagrams whose
	blocks have size one or two.
\end{example}

Secondly, in the case where the partition categories are one-coloured but
they do not contain the swap partition diagram (\ref{swappart}), we obtain true
compact matrix quantum groups 
(as the $C^*$-algebra is non-commutative). Again we highlight two examples.
\begin{example}
	\citet{wang1998} discovered the \textbf{symmetric quantum group} $C(S_n^+)$, which
	we introduced in Section \ref{noncommgeometry}.
	\citet{banica} showed that the spanning set of matrices (\ref{easyspanningset}) 
for $C(S_n^+)$ is the image under $d_\pi \mapsto D_\pi$ 
	of all non-crossing $(k,l)$--partition diagrams (i.e. no two connected components cross).
\end{example}
\begin{example}
	\cite{wang1995} originally discovered 
	the \textbf{orthogonal quantum group} $C(O(n)^+)$ as a universal $C^*$-algebra.
	\citet{banica} showed that the spanning set of matrices (\ref{easyspanningset}) 
for $C(O(n)^+)$ is the image under $d_\pi \mapsto D_\pi$ of 
	all non-crossing $(k,l)$--partition diagrams whose
	blocks come in pairs, and also recovered $C(O(n)^+)$ itself in this way (using the result of Theorem \ref{easycstar}).
\end{example}

Finally, for two-coloured partition categories, the most important examples are
the \textbf{unitary group} $U(n)$ and the \textbf{unitary quantum group} $U(n)^+$. 
These depend on the two-coloured words
$w_k$ and $w_l$ that are chosen, but \citet{tarrago2016, tarrago2018} showed that,
for $U(n)$, the spanning set of matrices is the image under $d_\pi \mapsto D_\pi$ 
of all $(w_k,w_l)$--partition diagrams whose blocks come in pairs 
such that if two vertices of a block are in the same row, 
then they have different colours, otherwise they have the same colours, and for
$U(n)^+$ (defined by Theorem \ref{easycstar}), it is the same except all crossing partition diagrams are first removed.


\section{Conclusion}

We have derived the existence of compact matrix quantum group equivariant
neural networks for learning from data that lives in a non-commutative geometry
and has 
symmetries that are described by compact matrix
quantum groups. We have characterised the linear layers of these neural networks
for the easy compact matrix quantum groups. As future work, we suggest that it 
would be good to extend this characterisation to the equivariant non-linear layers,
which would enable us to demonstrate in practice what these neural networks 
promise in theory.




\section*{Impact Statement}

This paper presents work whose goal is to advance the field of 
Machine Learning. There are many potential societal consequences 
of our work, none which we feel must be specifically highlighted here.


\nocite{*}
\bibliography{bibfile}
\bibliographystyle{icml2025}

\newpage
\appendix
\onecolumn

\section{General Topology} \label{gentopology}

We introduce some concepts 
from general topology --- also known as point set topology ---
that are used in the theory of $C^{*}$-algebras (Appendix: Section \ref{Cstarbackground}).
General topology provides an abstract mathematical framework
for studying the properties of a wide range of mathematical structures.
The mathematical structures in general topology are 
the so-called topological spaces, namely, sets
containing a family of subsets that satisfy some nice properties,
together with the so-called continuous functions,
which are the structure-preserving maps between topological spaces.
The family of subsets that is chosen for a set is called a topology, and
a set can have different topologies assigned to it.
In this way the notion of a continuous function depends on the topology that
is assigned to the domain and the codomain of the function.
Our treatment is only an overview of the main ideas and results:
we will mostly state these without proof and so we refer the reader to 
any standard reference, such as 
\citet{Bourbaki1995} or \citet{Munkres2018},
for more details.

We define the first key concept of general topology.
A \textbf{topology} on a set $X$ is a family of subsets $\tau_X$ of $X$ such that
\begin{enumerate}
	\item Both the empty set and $X$ are elements of $\tau_X$;
	\item Any union of elements of $\tau_X$ is an element of $\tau_X$;
	\item Any intersection of finitely many elements of $\tau_X$ is an element of $\tau_X$.
\end{enumerate}
We call a set $X$ together with a topology $\tau_X$ on $X$ a \textbf{topological space},
written as a pair $(X, \tau_X)$.
When the topology of a topological space $(X, \tau_X)$ is understood, 
we often just refer to the topological space as $X$.
We sometimes call the elements of $X$ \textbf{points}.
We call the elements of $\tau_X$ the \textbf{open sets} of $(X, \tau_X)$.
A subset of $(X, \tau_X)$ is said to be \textbf{closed} 
if its complement is an open set.
Note that by changing the topology that is associated with a set, 
the definition of the open sets changes.

We said that the maps of interest between topological spaces are the 
continuous functions, which are defined as follows.
A function $f : (X, \tau_X) \rightarrow (Y, \tau_Y)$ between topological spaces
is said to be \textbf{continuous} if $f$
is a function on the underlying sets such that if $U$ is an open set of $Y$,
then its preimage, $f^{-1}(U)$, is an open set of $X$.
One can show using this definition that $f$ is continuous if and only if
the preimage of a closed set in $Y$ is a closed set in $X$.
The continuous functions are important because they preserve certain properties 
of topological spaces, many of which will be defined in what follows.

One useful notion that helps to define topologies on sets is that of a \textbf{base},
or sometimes \textbf{basis}.
A base $B$ for a topological space $(X, \tau_X)$ is a collection of open sets in $\tau_X$
such that every open set can be written as a union of elements in $B$.
We say that $B$ \textbf{generates the topology} $\tau_X$.

An important class of topological spaces are \textbf{metric spaces}, 
which come with a notion of distance
between any pair of points in the set.
A metric space is a set $X$ together with a function 
$d_X: X \times X \rightarrow \mathbb{R}$, called a \textbf{metric},
such that the following axioms hold:
\begin{enumerate}
	\item $d_X(x,y) \geq 0$ for all $x, y \in X$ and $d_X(x,y) = 0$ if and only if $x = y$;
	\item $d_X(x,y) = d_X(y,x)$ for all $x, y \in X$;
	\item $d_X(x,y) \leq d_X(x,z) + d_X(z,y)$ for all $x, y, z \in X$.
\end{enumerate}
A metric space is often denoted by the pair $(X,d_X)$ or simply by $X$ when
the metric $d_X$ is understood.
Some of the most important sets in a metric space are the \textbf{open balls}.
For any point $x$ in a metric space $(X,d_X)$ and for any real number $\epsilon > 0$,
the open ball of radius $\epsilon$ around $x$, written $B(x, \epsilon)$,
is defined to be the set $\{y \in X \mid d_X(x,y) < \epsilon\}$.
When a metric space is given a topology that is induced by the metric,
known as the \textbf{metric topology},
then the metric space together with the metric topology becomes a topological space.
The metric topology is generated by the base 
that is given by all of the open balls that are defined by the metric.

A particularly important class of metric spaces are the 
\textbf{normed vector spaces}, 
which are commonly refered to simply as \textbf{normed spaces}.
To define them, we first need to define the concept of a \textbf{norm}
on a vector space.
A norm on a complex vector space $X$ is a function 
$\norm{\cdot} : X \rightarrow \mathbb{R}$ such that
\begin{enumerate}
	\item $\norm{x} \geq 0$ for all $x \in X$
		and if $\norm{x} = 0$ then $x = 0$;
	\item $\norm{{\lambda}x} \leq \lvert \lambda \rvert \norm{y}$ for all 
		 $x \in X$ and $\lambda \in \mathbb{C}$;
	\item $\norm{x + y} \leq \norm{x} + \norm{y}$ for all $x, y \in X$;
\end{enumerate}
If $X$ is an algebra, then a norm on $X$ must also satisfy the submultiplicativity 
axiom, namely
\begin{enumerate}
	\setcounter{enumi}{3}
	\item $\norm{xy} \leq \norm{x}\norm{y}$ for all $x, y \in X$ 
\end{enumerate}
A normed space is a complex vector space (or algebra) $X$ 
that is equipped with a norm $\norm{\cdot} : X \rightarrow \mathbb{R}$.
Normed spaces $X$ become metric spaces under the metric that is induced by the norm, namely
\begin{equation}
	d_X(x,y) \coloneqq \norm{x - y} \text{ for all } x, y \in X
\end{equation}
In particular, one can show that $d_X$ satisfies the axioms of a metric on $X$.

One nice property for a topological space to have is the \textbf{Hausdorff property}.
The Hausdorff property defines a notion of separation between points in a topological
space that will be particularly useful when we consider limits of sequences below.
To define the Hausdorff property, we first need the definition of a 
\textbf{neighbourhood}.
A neighbourhood of a point $p$ in a topological space $(X, \tau_X)$
is a subset $V$ of $X$ such that $V$ contains an open set $U$ that contains $p$,
that is, $p \in U \subseteq V \subseteq X$. Note that $V$ does not necessarily
need to be an open set itself.
We say that a topological space $(X, \tau_X)$ is \textbf{Hausdorff}
if, for any two distinct points $x, y \in X$,
there exists a neighbourhood $V$ of $x$ and a neighbourhood $W$ of $y$
such that $V$ and $W$ are disjoint.
In particular, one can show that every metric space $(X, d_X)$ is Hausdorff.

We now look at the \textbf{closure} of a set in a topological space $(X, \tau_X)$.
Every subset $V$ of a topological space $(X, \tau_X)$ has a closure 
which is denoted by $\overline{V}$:
it is the intersection of all closed sets in $(X, \tau_X)$ that contain $V$.
This definition is not particularly helpful, so to make it more useful
we need the following definition of a \textbf{limit point}.
For a subset $V$ of a topological space $(X, \tau_X)$,
a limit point of $V$ is a point $x \in X$ such that
every neighbourhood of $x$ contains a point of $V$ that is not $x$.
One can then show that the closure of $V$ is the union of $V$ and
the set containing all limit points of $V$.
We say that a subset $V$ of a topological space $(X, \tau_X)$ is \textbf{dense} in $X$
if its closure is $X$, that is, $\overline{V} = X$.
Moreover, one can show that the 
closure of a closed set is the set itself.
Hence if $V$ is dense in $X$ and $V$ is closed, then $V = X$.
All of these definitions apply immediately in the case where $X$ is a metric space.

We continue now with metric spaces and consider sequences of points in a metric space.
A sequence $(x_n)$ in a metric space $(X,d_X)$ is said to be \textbf{Cauchy}
if for all $\epsilon > 0$, there exists a positive integer $N$ such that
for all $m, n \geq N$, $d_X(x_m, x_n) < \epsilon$.
A sequence $(x_n)$ in a metric space $(X,d_X)$ is said to \textbf{converge} 
to a limit $x \in X$
if for all $\epsilon > 0$, there exists a positive integer $N$ such that
for all $n \geq N$, $d_X(x_x, x) < \epsilon$. 
In particular, one can show that if a sequence converges in a metric space,
then the limit must be unique.
This is a direct consequence of a metric space
being Hausdorff.
Moreover, continuous functions preserve the convergence of sequences, in that if 
$f: (X,d_X) \rightarrow (Y,d_Y)$ is a continuous function between metric spaces 
and $(x_n)$ is a sequence
that converges to $x \in X$, then the sequence $(f(x_n))$ in $Y$ 
converges to $f(x) \in Y$.
Particularly nice metric spaces are those that are said to be complete:
a metric space $(X,d_X)$ is \textbf{complete} if every Cauchy sequence in $(X,d_X)$
converges.
In particular, a normed space that is complete with the metric induced by the norm
is known as a 
\textbf{Banach algebra}, which is the starting point for the material
that appears in Section \ref{Cstarbackground} of the Appendix.
One further useful result is that a subset of a 
complete metric space is complete if and only if it is closed.

A notable point is that closed sets
have an alternative characterisation when $X$ is a metric space.
Firstly, note that if $(x_n)$ is a sequence that lives in a subset
$V$ of $X$ that converges, then its limit $x$ is not necessarily in $V$ 
(but it is in $X$).
However, one can show that a subset $V$ of a metric space $(X,d_X)$ is \textbf{closed}
if and only if every sequence $(x_n)$ in $V$ converges to a limit in $V$.



We said that complete metric spaces are particularly nice: 
they allow us to solve problems iteratively using Cauchy sequences,
since we know that they have a limit.
However, even if a metric space is incomplete, 
one can create a complete metric space from it
in a process known as \textbf{completion}.
We first need the following definition.
A map $i : (X, d_X) \rightarrow (Y, d_Y)$ is said to be an \textbf{isometry}
if it satisfies
\begin{equation}
	d_Y(i(x_1), i(x_2)) 
	= d_X(x_1, x_2)
\end{equation}
for all $x_1, x_2 \in X$.
A surjective isometry between metric spaces is called an 
\textbf{isomorphism of metric spaces}, or simply an \textbf{isomorphism}.
Hence, we have that 
a metric space 
$(\tilde{X}, d_{\tilde{X}})$ is called 
the completion of a metric space $(X,d_X)$ if the following
conditions are satisfied:
\begin{itemize}
	\item there is an isometry $i : X \rightarrow \tilde{X}$
	\item the image $i(X)$ is dense in $\tilde{X}$, and
	\item the space $(\tilde{X}, d_{\tilde{X}})$ is complete.
\end{itemize}
In fact, one can show that 
every metric space has a unique completion, up to isomorphism.
In particular, every incomplete metric space can be completed.

One of the most important concepts in general topology is 
\textbf{compactness}, which is a notion of finiteness for topological spaces.
In general topology, compact topological spaces
play an analogous role to finite dimensional vector spaces in the theory of vector spaces
or to finite sets in the theory of sets. 
The motivation for compactness of a topological space $X$ comes from
wanting to understand the conditions that $X$ needs to satisfy such that
any 
continuous map $f : X \rightarrow \mathbb{C}$ is bounded.
Recall that a function $f : X \rightarrow \mathbb{C}$ is \textbf{bounded}
if there exists a non-negative real number $M$ such that
$\lvert f(x) \rvert \leq M$ for all $x \in X$.

One way in which a continuous function $f : X \rightarrow \mathbb{C}$ 
will be bounded is if it is bounded on a finite number of subsets of $X$
whose union is $X$ itself. 
This provides sufficient motivation for the following definitions.
If $(X, \tau_X)$ is a topological space, then a \textbf{cover} of $X$
is a family $(U_i)_{i \in I}$ of subsets of $X$ such that 
$\bigcup_{i \in I} U_i = X$.
The cover is \textbf{finite} if the indexing set $I$ is finite,
and the cover is \textbf{open} if $U_i$ is open for each $i \in I$.
If $(U_i)_{i \in I}$ is a cover of $X$, then we say that
$(U_j)_{j \in J}$ is a subcover of $(U_i)_{i \in I}$ if $J \subseteq I$
and $(U_j)_{j \in J}$ is itself a cover of $X$.
Consequently, a topological space $(X, \tau_X)$ is \textbf{compact} if every open
cover of $X$ has a finite subcover.
Hence, every continuous function $f : X \rightarrow \mathbb{C}$
on a compact topological space $X$ is bounded.
We study $C(X)$, the space of continuous functions $f : X \rightarrow \mathbb{C}$,
in the case where $X$ is a compact Hausdorff space in Section \ref{Cstarbackground}.
It is also worth noting that every finite topological space is compact 
since any open cover must already be finite.

One can also show that every closed subspace of a compact space is compact,
where the subspace is given the \textbf{subspace topology}:
if $V$ is a subspace of a topological space $(X, \tau_X)$, then the
subspace topology $\tau_V$ on $V$ is defined by 
$\{U \cap V \mid U \in \tau_X\}$.
Although it is not true that every compact subspace of a compact space is closed,
we do have that every compact subspace of a Hausdorff space is closed.
Moreover, if $(X, \tau_X)$ is compact and
$f: (X, \tau_X) \rightarrow (Y, \tau_Y)$ is a continuous function,
then $f(X)$ is compact.
And if $(X, \tau_X)$ is compact and $f: X \rightarrow \mathbb{C}$ is a continuous
function, then $f(X)$ is bounded.
For metric spaces $(X,d_X)$, one can show that 
compactness is equivalent to sequential compactness,
which means that every sequence in $X$ has a convergent subsequence.
Furthermore, if the metric space is a Euclidean space, say 
$\mathbb{R}^n$ or $\mathbb{C}^n$, then one can show that a subspace
of a Euclidean space is compact if and only if that subspace is closed and bounded.

Some of the most important compact topological spaces that we consider
are \textbf{compact groups}.
Firstly, a \textbf{topological group} is 
a topological space $(G, \tau_G)$ that has a group structure defined on it
such that the group operations of multiplication and inversion
\begin{equation}
	(g, h) \mapsto gh, \quad g \mapsto g^{-1}	
\end{equation}
are both continuous functions.
A topological group is called a compact group
if the underlying topological space is compact.
Note that we always choose our topological groups to be Hausdorff, 
and so this property is always assumed when we refer to a topological group.
Note also that, in the theory of topological groups, a subgroup
of a topological group is always assumed to be closed.
In particular, a closed subgroup of a compact group is a compact group.

In the main paper, we study compact matrix groups.
The most notable example of a compact matrix group is
the \textbf{orthogonal group} $O(n)$ of $n \times n$ matrices $M$ 
such that $M \mapsto M^\top{M} = I$.
To show that $O(n)$ is compact, it is enough to show that it is closed and bounded,
since $O(n)$ is a subset of the Euclidean space $\mathbb{R}^{n \times n}$.
Given that one can show that singleton sets in a metric space are closed 
and that polynomials are continuous functions,
$O(n)$ is closed since it is the preimage 
of the singleton set $\{I\}$ 
under the continuous map 
$f : \mathbb{R}^{n \times n} \rightarrow \mathbb{R}^{n \times n}$
that is given by $M \mapsto M^\top{M}$.
$O(n)$ is bounded since the Euclidean norm of an orthogonal $n \times n$ matrix
is $\sqrt{n}$.
Moreover, any finite group --- in particular, 
the \textbf{symmetric group} $S_n$ ---
is compact.


\section{An Overview of $C^{*}$-Algebras and Operator Theory} \label{Cstarbackground}

We provide an overview of the key results 
that appear in the theory of $C^{*}$-algebras,
which is an area of mathematics that lies 
at the intersection of algebra and topology.
The theory has its roots in the development of quantum
mechanics in the 1920s, 
where the non-commutativity
of bounded, linear operators on a Hilbert space is used to model physical observables.
\citet{gelfand1943} developed an abstract definition of 
$C^{*}$-algebras that did not require the use of bounded, linear operators on a Hilbert space.
Forty-five years later, 
\citet{woronowicz1987, woronowicz1988} used $C^{*}$-algebras
to develop compact matrix quantum groups.
Since a proper treatment of the theory of $C^{*}$-algebras would amount to several 
hundreds of pages, we focus solely on the key concepts 
that also appear in the theory of compact matrix quantum
groups given in Section \ref{noncommgeometry}; 
notably, tensor products of $C^{*}$-algebras and the 
construction of the universal $C^{*}$-algebra from a set of generators and relations
on those generators.
When we quote a result without proof, we provide a reference to the literature 
where its proof can be found.
We also rely on the background material on general topology
that appeared in Section \ref{gentopology} throughout. 
For a more comprehensive treatment on $C^{*}$-algebras, 
we recommend any of the following references:
\citet{blackadar2006operator, courtney2023notes,
	gromada2020, lin2001,
	maassen2021, speich} and \citet{rordam2000}.

\begin{definition}
	We build up to the definition of a $C^{*}$-algebra with the following definitions.
	\begin{itemize}
		\item An \textbf{algebra} $A$ over $\mathbb{C}$ is a complex vector space
			together with a bilinear, associative multiplication
			$A \times A \mapsto A$ 
			satisfying 
			$\lambda(xy) = (\lambda{x})y = x(\lambda{y})$
			for $x, y \in A$ and $\lambda \in \mathbb{C}$.
		\item An algebra $A$ is \textbf{unital} if it contains an element $1 \in A$, called the \textbf{unit}, such that $1x = x1 = x$ for all $x \in A$.
		\item A subset $B \subseteq A$ is called a \textbf{subalgebra}
			if $B$ is an algebra with respect to the same operations.
		\item A subset $I \subseteq A$ is called an \textbf{ideal} 
			if $I$ is a subalgebra of $A$ that is invariant with respect
			to both left and right multiplication by elements of $A$.
		\item A \textbf{normed algebra} is an algebra $A$ that has a norm 
			$\norm{\cdot} : A \rightarrow \mathbb{C}$
			such that it is submultiplicative: 
			$\norm{xy} \leq \norm{x}\norm{y}$ for all $x, y \in A$.
		\item A \textbf{Banach algebra} is a normed algebra $A$ which is complete 
			with respect to the metric induced by the norm $\norm{\cdot}$.
		\item An \textbf{involution} on an algebra $A$ is an antilinear map
			$^*: A \rightarrow A$ such that $(x^*)^* = x$ and $(xy)^* = y^*x^*$ for all $x, y \in A$.
		\item A $\bm{^*}$\textbf{-algebra} is an algebra $A$ that has an involution. Similarly, we define $\bm{^*}$\textbf{-subalgebra} and $\bm{^*}$\textbf{-ideal}.
		\item A \textbf{Banach} $\bm{^*}$\textbf{-algebra} is a Banach algebra $A$ that has an involution.
	\end{itemize}
\end{definition}

Consequently, we have that

\begin{definition}
	A $\bm{C^{*}}$\textbf{-algebra} is a Banach $^*$-algebra $A$ 
	that satisfies the $\bm{C^{*}}$\textbf{-identity}: 
	$\norm{x^*x} = \norm{x}^2$ for all $x \in A$.
\end{definition}

Note that a closed $^*$-subalgebra of a $C^*$-algebra is also a $C^*$-algebra.

\begin{definition}
	We say that an element $a$ of a $C^{*}$-algebra 
	$A$ is \textbf{self-adjoint}
	if $a^* = a$, and, if $A$ is unital, we say that 
	$a$ is \textbf{unitary} if $a^*a = aa^* = 1$.
\end{definition}

\begin{example} \label{exCstaralgebras}
	We give some examples of $C^{*}$-algebras.
	\begin{enumerate}
		\item The complex numbers $\mathbb{C}$ form a unital $C^{*}$-algebra.
		\item Let $\mathcal{H}$ be a complex Hilbert space with inner product denoted by
			$\langle \cdot, \cdot \rangle$.
			The collection of bounded linear operators $\mathcal{B}(\mathcal{H})$
			is a unital $C^{*}$-algebra, 
			where the norm is the operator norm
			\begin{equation}
				\norm{T}
				\coloneqq
				\sup\{\norm{Tv}_\mathcal{H} \mid v \in \mathcal{H}, \norm{v}_\mathcal{H} \leq 1\}
			\end{equation}
			for all $T \in \mathcal{B}(\mathcal{H})$.
			Note that multiplication is given by the composition of operators,
			and that the involution is the adjoint, defined with respect
			to the inner product by 
			$\langle T^*v,w \rangle = \langle v, Tw \rangle$
			for all $T \in \mathcal{B}(\mathcal{H})$ and $v, w \in \mathcal{H}$.
			Also, we see that $\mathcal{B}(\mathcal{H})$ is non-commutative.
		\item For any positive integer $n$, let $M_n(\mathbb{C})$ denote the set
			of $n \times n$ matrices with elements in $\mathbb{C}$.
			Then $M_n(\mathbb{C})$ is a unital $C^{*}$-algebra,
			where the norm is the operator norm
			\begin{equation}
				\norm{M}
				\coloneqq
				\sup\{\norm{Mv}_2 \mid v \in \mathbb{C}^n, \norm{v}_2 \leq 1\}
			\end{equation}
			where $\norm{\cdot}_2$ is the Euclidean norm on $\mathbb{C}^n$,
			and the involution is given by the conjugate transpose.
			This is a special case of the previous example, where we 
			have chosen $\mathcal{H} = \mathbb{C}^n$ and picked a basis to represent 
			linear transformations as matrices.
		\item Let $X$ be a compact Hausdorff 
			space. 
			Then the space of
			continuous functions on $X$,
			\begin{equation}
				C(X) \coloneqq \{f : X \rightarrow \mathbb{C} \mid f \text{ is continuous}\}
			\end{equation}
			is a unital $C^{*}$-algebra 
			where the norm is the supremum norm
			\begin{equation}
				\norm{f}_\infty 
				\coloneqq
				\sup\{\lvert f(x) \rvert \mid x \in X\}
			\end{equation}
			such that 
			$(f+g)(x) \coloneqq f(x) + g(x)$,
			$(\lambda{f})(x) \coloneqq \lambda{f(x)}$,
			$(fg)(x) \coloneqq f(x)g(x)$,
			$1(x) \coloneqq 1$ 
			and
			$f^*(x) \coloneqq \overline{f(x)}$
			for all $f, g \in C(X), x \in X$ and $\lambda \in \mathbb{C}$.
			Note, in particular, that $C(X)$ is commutative.
			The supremum norm is indeed a norm since 
			the compactness of $X$
			implies that every continuous function
			$f : X \rightarrow \mathbb{C}$ is bounded, and so
			$\norm{f}_\infty < \infty$
			for all $f \in C(X)$. 
			The norm is also submultiplicative because
			\begin{equation}
				\lvert fg(x) \rvert
				=
				\lvert f(x)g(x) \rvert
				=
				\lvert f(x) \rvert \lvert g(x) \rvert
				\leq
				\norm{f}_\infty \norm{g}_\infty 
			\end{equation}
			which implies that
			$\norm{fg}_\infty \leq \norm{f}_\infty \norm{g}_\infty$.
			Similar reasoning on the fact that
			$\lvert f(x)\overline{f(x)} \rvert = \lvert f(x) \rvert^2$
			shows that the supremum norm satisfies the $C^*$-identity.
			Finally, the supremum norm is complete because 
			one can show that any Cauchy sequence of continuous 
			functions on $X$
			converges uniformly on $X$ such that the limit
			is also a continuous function on $X$.
	\end{enumerate}
\end{example}

An important theorem for the $C^*$-algebra of continuous functions
on a compact Hausdorff space $X$
is the Stone-Weierstrass theorem.
Its proof can be found in \citet[Theorem 3.3]{speich}.
First we need the following definition.

\begin{definition}
	Let $A$ be an algebra of functions from some set $X$ to the complex numbers 
	$\mathbb{C}$. We say that $A$ \textbf{separates points} in $X$
	if, for all $x \neq y \in X$, there is a function $f \in A$ 
	such that $f(x) \neq f(y)$.
\end{definition}

Consequently, we have that

\begin{theorem}[Stone-Weierstrass]
	\label{stoneweierstrass}
	Let $X$ be a compact Hausdorff space.
	If $A$ is a unital $^*$-subalgebra of $C(X)$ 
	such that $A$ separates points in $X$, then $A$ is dense in $C(X)$.
	In particular, if $A$ is closed, then $A = C(X)$.
\end{theorem}

We also have homomorphisms on $C^*$-algebras:

\begin{definition}
	If $A, B$ are algebras, then a linear map $\phi : A \rightarrow B$ such that
	$\phi(xy) = \phi(x)\phi(y)$ for all $x,y \in A$ is called an \textbf{algebra homomorphism}.

	If $A, B$ are $\bm{^*}$-algebras, then if $\phi : A \rightarrow B$ is an algebra homomorphism that preserves the involution operation,
	that is, $\phi(x^*) = \phi(x)^*$ for all $x \in A$, then $\phi$ is called a $\bm{^*}$\textbf{-homomorphism}.

	If $A, B$ are $\bm{^*}$-algebras, then 
	a bijective $^*$-homomorphism
	$\phi : A \rightarrow B$
	is called a $\bm{^*}$\textbf{-isomorphism}.
\end{definition}

We use the properties of $\mathcal{B}(\mathcal{H})$ 
to obtain representations of
$C^*$-algebras:

\begin{definition}
	A \textbf{representation} of an algebra $A$ is a choice of complex Hilbert space
	$\mathcal{H}$ and an algebra homomorphism
	$\pi : A \rightarrow \mathcal{B}(\mathcal{H})$.

	A $\bm{^*}$-\textbf{representation} of a $^*$-algebra $A$
	is a representation of $A$ that preserves the involution operation.

	A representation is said to be \textbf{faithful} if it is injective.
\end{definition}

There is another important theorem relating to 
the unital $C^*$-algebra of continuous functions on a compact Hausdorff space $X$.
This is the Gelfand--Naimark Theorem that we referenced in Section \ref{noncommgeometry},
and
since it is the first
of two theorems by Gelfand and Naimark that appears in this section, 
we will call it Gelfand--Naimark I.
Its proof can be found in \citet[Theorem II.2.2.4]{blackadar2006operator}.


\begin{theorem}[Gelfand--Naimark I]
	Let $A$ be a commutative, unital $C^*$-algebra.
	Then $A$ is $^*$-isomorphic to $C(X)$, where
	\begin{equation}
		X 
		\coloneqq
		\{\phi : A \rightarrow \mathbb{C} \mid 
			\phi \text{ is a non-zero } ^*\text{-homomorphism}\}
	\end{equation}
	is a compact Hausdorff space. 
	The $^*$-isomorphism is given by
	$f : A \rightarrow C(X)$ where $f(a) : X \rightarrow \mathbb{C}$ 
	for all $a \in A$ is such that $f(a)[\phi] = \phi(a)$ for all $\phi \in X$.
\end{theorem}

We now look at tensor products of $C^*$-algebras.
First, we have the following definition.
\begin{definition}
	Let $A, B$ be $^*$-algebras.
	The \textbf{algebraic tensor product}, written $A \odot B$,
	is defined to be the algebraic tensor product of the vector spaces
	$A$ and $B$ 
	together with a multiplication operation that is given by
	\begin{equation}
		(a_1 \odot b_1)(a_2 \odot b_2)
		=
		a_1a_2 \odot b_1b_2
	\end{equation}
	and an involution operation that is given by
	\begin{equation}
		(a \odot b)^*
		=
		a^* \odot b^*
	\end{equation}
	In particular, $A \odot B$ is a $^*$-algebra.
\end{definition}
To define a tensor product on $C^*$-algebras $A$ and $B$
such that the result is also a $C^*$-algebra, we need to be able to define
a norm $\norm{\cdot}$ that satisfies the $C^*$-identity on $A \odot B$,
and then complete $(A \odot B, \norm{\cdot})$ to obtain a $C^*$-algebra.
Any norm that satisfies the $C^*$-identity
is said to be a $\bm{C^*}$\textbf{-norm}.
Although it can be shown,
as a consequence of the Gelfand--Naimark--Segal Theorem,
that $C^*$-norms on algebraic tensor products
of $C^*$-algebras always exist 
--- see the definition of the spatial norm below ---
one of the main difficulties behind this construction 
is that in many cases there can be more than one
$C^*$-norm that can be associated with $A \odot B$
\citep{courtney2023notes},
and so in those cases
there are many possible definitions of a tensor product
on $C^*$-algebras $A$ and $B$ such that
the result is a $C^*$-algebra.
Hence, in what follows, we choose only to outline the main ideas 
that lead to the construction of a particular tensor product called the 
\textbf{minimal tensor product}
since this tensor product will be used in the definition 
of a compact matrix quantum group.

We start with another theorem by Gelfand--Naimark,
which is technically a corollary of the Gelfand--Naimark--Segal Theorem.
Its proof can be found in \citet[Theorem 5.19]{speich}.

\begin{theorem}[Gelfand--Naimark II]
	Every $C^*$-algebra $A$ has a faithful $^*$-representation 
	$\pi : A \rightarrow \mathcal{B}(\mathcal{H})$
	on some Hilbert space $\mathcal{H}$.
\end{theorem}

Hence, for $C^*$-algebras $A$ and $B$, we have faithful $^*$-representations
$\pi_A : A \rightarrow \mathcal{B}(\mathcal{H}_A)$
and
$\pi_B : B \rightarrow \mathcal{B}(\mathcal{H}_B)$.
One can show 
\citep[Corollary 11.15]{courtney2023notes}
that this induces a faithful $^*$-representation
$\pi_A \odot \pi_B : A \odot B \rightarrow \mathcal{B}(\mathcal{H}_A \otimes \mathcal{H}_B)$
such that 
$(\pi_A \odot \pi_B)(a \odot b) = \pi_A(a) \otimes \pi_B(b)$
for all $a \in A$ and $b \in B$.

Hence we can define a norm $\norm{\cdot}_*$ on $A \odot B$, 
known as the \textbf{spatial norm},
that is given by
\begin{equation}
	\norm{x}_* = \norm{(\pi_A \odot \pi_B)(x)}_{\mathcal{B}(\mathcal{H}_A \otimes \mathcal{H}_B)}
\end{equation}
which is a $C^*$-norm by the fact that the norm on $\mathcal{B}(\mathcal{H}_A \otimes \mathcal{H}_B)$
is a $C^*$-norm together with the injectivity of $\pi_A \odot \pi_B$.

\citet{Takesaki2002} showed
that the spatial norm is independent
of the faithful $^*$-representation, and that it is the minimal
$C^*$-norm on $A \odot B$; that is, for all $C^*$-norms $\gamma$ 
on $A \odot B$,
we have that
\begin{equation}
	\norm{x}_* \leq \gamma(x)
\end{equation}
for all $x \in A \odot B$.
Hence we call $\norm{x}_*$ the \textbf{minimum norm}, and write it as
$\norm{x}_{\min}$.


Consequently, the completion of $A \odot B$, for $C^*$-algebras $A$ and $B$,
with respect to the minimal norm $\norm{x}_{\min}$, 
is called the \textbf{minimal tensor product}
and is denoted by $A \otimes_{\min} B$.

One can also show 
\citep[Corollary 11.29]{courtney2023notes}
that 
for a pair of $^*$-homomorphisms 
$\phi_1 : A_1 \rightarrow B_1$
and
$\phi_2 : A_2 \rightarrow B_2$
on $^*$-algebras $A_1, A_2, B_1, B_2$,
the algebraic tensor product
$\phi_1 \odot \phi_2 : A_1 \odot A_2 \rightarrow B_1 \odot B_2$
extends to a $^*$-homomorphism
$\phi_1 \otimes_{\min} \phi_2 : A_1 \otimes_{\min} A_2 \rightarrow B_1 \otimes_{\min} B_2$.

We now focus on $M_n(A)$, the set of $n \times n$ matrices with entries in 
a $C^*$-algebra $A$, that is
\begin{equation}
	M_n(A)
	\coloneqq
	\{(a_{i,j}) \mid a_{i,j} \in A, 1 \leq i,j \leq n\}
\end{equation}
This comes with a natural involution that is given by
$(a_{i,j})^* = (a_{j,i}^*)$ for all $(a_{i,j}) \in M_n(A)$.

We use the following definition to create a $C^*$-norm on $M_n(A)$.

\begin{definition} 
	Suppose that $\phi: A \rightarrow B$ is a linear map
	between $^*$-algebras $A$ and $B$.
	Then, for all positive integers $n$, we define the linear map
	$\phi^{(n)} : M_n(A) \rightarrow M_n(B)$ to be
	$\phi^{(n)}((a_{i,j})) = (\phi(a_{i,j}))$.
	The map $\phi^{(n)}$ is called a \textbf{matrix amplification} of $\phi$.
\end{definition}

We create the $C^*$-norm on $M_n(A)$ as follows.
By Gelfand--Naimark II,
we have a faithful $^*$-representation 
$\pi : A \rightarrow \mathcal{B}(\mathcal{H})$ 
for some Hilbert space $\mathcal{H}$.
This induces a faithful
$^*$-representation
$\pi^{(n)} : M_n(A) \rightarrow \mathcal{B}(\mathcal{H}^n)$
since one can show 
\citet[Section 6.1]{rordam2000}
that
$M_n(\mathcal{B}(\mathcal{H}))$ is $^*$-isomorphic to $\mathcal{B}(\mathcal{H}^n)$.
Then we can define a function $\norm{\cdot}: M_n(A) \rightarrow \mathbb{C}$ by
\begin{equation} \label{normMnA}
	\norm{(a_{i,j})} 
	\coloneqq
	\norm{\pi^{(n)}((a_{i,j}))}_{\mathcal{B}(\mathcal{H}^n)}
\end{equation}
which is a $C^*$-norm 
since $\pi^{(n)}$ is injective and
the norm on $\mathcal{B}(\mathcal{H}^n)$
is a $C^*$-norm. 

One can show 
the following inequality,
whose proof is outlined in 
\citet[Exercise 1.13]{rordam2000}.

\begin{lemma} \label{BHinequality}
	Let $(T_{i,j}) \in M_n(\mathcal{B}(H))$.
	Then
	\begin{equation}
		\norm{T_{i,j}}_{\mathcal{B}(\mathcal{H})}
		\leq
		\norm{(T_{i,j})}_{M_n(\mathcal{B}(\mathcal{H}))}
		\leq
		\sum_{i,j} \norm{T_{i,j}}_{\mathcal{B}(\mathcal{H})}
	\end{equation}
\end{lemma}

Lemma \ref{BHinequality} implies that the image of $M_n(A)$ under $\pi^{(n)}$,
which is $^*$-isomorphic to $M_n(A)$ since $\pi^{(n)}$ is injective,
is a closed $^*$-subalgebra of the $C^*$-algebra $\mathcal{B}(\mathcal{H}^n)$.
Hence $M_n(A)$ is itself a $C^*$-algebra.
Moreover, the norm on $M_n(A)$ defined in (\ref{normMnA}) is unique
because a $^*$-algebra admits at most one norm 
that makes it into a $C^*$-algebra
\citet[Corollary 1.2.8]{lin2001}.

Expressing this result as a theorem, we have that

\begin{theorem}
	If $A$ is a $C^*$-algebra, then $M_n(A)$ is also a $C^*$-algebra
	whose norm is defined in (\ref{normMnA}).
\end{theorem}

One can also show the following result.

\begin{lemma}
	Suppose that $A$ is a $C^*$-algebra and let $n$ be any positive integer.
	Define a map $\phi : M_n(A) \rightarrow M_n(\mathbb{C}) \odot A$ by
	\begin{equation} \label{phiMnCAiso}
		\phi((a_{i,j}))
		=
		\sum_{i,j = 1}^n E_{i,j} \odot a_{i,j}
	\end{equation}
	where $E_{i,j}$ is the matrix unit with a $1$ in the $(i,j)$-entry, and $0$ otherwise.
	Then $\phi$ is a $^*$-isomorphism.
\end{lemma}

We use the $^*$-isomorphism between $M_n(A)$ and $M_n(\mathbb{C}) \odot A$
to obtain the following isomorphism.

\begin{proposition}
	Suppose that $A$ is a $C^*$-algebra and let $n$ be any positive integer.
	Then there is a unique $C^*$-norm on the algebraic tensor product
	$M_n(\mathbb{C}) \odot A$.
	Hence we write $M_n(\mathbb{C}) \otimes A$, and have that
	\begin{equation}
		M_n(\mathbb{C}) \otimes A 	
		=
		M_n(\mathbb{C}) \odot A 	
		\cong
		M_n(A)
	\end{equation}
\end{proposition}

\begin{proof}
	The $^*$-isomorphism together with the $C^*$-norm given on $M_n(A)$
	defines the existence of a $C^*$-norm on $M_n(\mathbb{C}) \odot A$,
	namely
	\begin{equation}
		\norm{(\lambda_{i,j}) \odot a}
		=
		\norm{(\lambda_{i,j}a)}
	\end{equation}
	$M_n(\mathbb{C}) \odot A$
	is now $^*$-isomorphic to a closed $^*$-subalgebra
	of the $C^*$-algebra $M_n(A)$, namely $M_n(A)$ itself,
	hence 
	it is a $C^*$-algebra.
	The result is now immediate 
	because a $^*$-algebra admits at most one norm making it into a $C^*$-algebra.
\end{proof}

This means that we can use the $^*$-isomorphism $\phi$
to identify elements of $M_n(A)$
with elements of $M_n(\mathbb{C}) \otimes A$, namely
\begin{equation}
	(a_{i,j}) \leftrightarrow \sum_{i,j = 1}^n E_{i,j} \otimes a_{i,j}
\end{equation}

We now turn to the concept of \textbf{universal} $\bm{C^*}$\textbf{-algebras}.
We know that we can create $C^*$-algebras by starting with
a $^*$-algebra, finding a $C^*$-norm for it, and then completing the pair.
We would like to be able to create $C^*$-algebras 
using sets of generators and relations.
We follow the presentation that is given in \citet[Chapter 6]{speich}.

We begin by creating a $^*$-algebra as follows.

\begin{definition}
	Let $E = \{x_i \mid i \in I\}$ be a set of symbols over an index set $I$.
	Let $E^* = \{x_i^* \mid i \in I\}$ be another set of symbols 
	that is disjoint from $E$. Then
	\begin{itemize}
		\item	A \textbf{non-commutative monomial} in $E \cup E^*$ is a word
			of non-commuting symbols
			$x_{i_1}^{\alpha_1}\cdots{x_{i_m}^{\alpha_m}}$
			for some positive integer $m$
			such that $i_1, \dots, i_m \in I$
			and $\alpha_j \in \{\pm 1\}$ for all $i \in [m]$.
		\item	A \textbf{non-commutative polynomial} in $E \cup E^*$ is
			a formal linear combination of non-commutative monomials
			in $E \cup E^*$ with coefficients in $\mathbb{C}$.
		\item	We can define a \textbf{concatenation operation} on 
			non-commutative monomials in $E \cup E^*$:
			\begin{equation} \label{freeconcat}
				(x_{i_1}^{\alpha_1}\cdots{x_{i_m}^{\alpha_m}})	
				\cdot
				(x_{j_1}^{\beta_1}\cdots{x_{j_m}^{\beta_n}})	
				=
				x_{i_1}^{\alpha_1}\cdots{x_{i_m}^{\alpha_m}}
				x_{j_1}^{\beta_1}\cdots{x_{j_m}^{\beta_n}}
			\end{equation}
			for any non-commutative monomials
			$x_{i_1}^{\alpha_1}\cdots{x_{i_m}^{\alpha_m}}$
			and
			$x_{j_1}^{\beta_1}\cdots{x_{j_m}^{\beta_n}}$
			in $E \cup E^*$.
			This operation is extended linearly to 
			to non-commutative polynomials in $E \cup E^*$.
		\item	We can also define an \textbf{involution operation} on 
			non-commutative monomials in $E \cup E^*$:
			\begin{equation} \label{freeinvolution}
				(x_{i_1}^{\alpha_1}\cdots{x_{i_m}^{\alpha_m}})^*
				=
				x_{i_1}^{{\alpha_1}^*}\cdots{x_{i_m}^{{\alpha_m}^*}}
			\end{equation}
			where ${{\alpha_j}^*}$ is $1$ if $\alpha_j = *$ and is $*$
			if $\alpha_j = 1$ for all $j \in [m]$.
			This operation is extended antilinearly 
			to non-commutative polynomials in $E \cup E^*$.
	\end{itemize}
\end{definition}

Then, we have the following definition.

\begin{definition}
	Let $E, E^*$ be as before.
	The \textbf{free} $\bm{^*}$\textbf{-algebra} 
	\textbf{on the generator set} $\bm{E}$, $P(E)$,
	is the set of non-commutative polynomials in $E \cup E^*$
	where addition and scalar multiplication are the standard operations,
	and
	multiplication and involution are given by the concatenation and involution
	operations given in (\ref{freeconcat}) and (\ref{freeinvolution}), respectively.

	The $^*$-algebra $P(E)$ is free in the sense that there are no relations
	amongst the polynomials in $P(E)$.
\end{definition}

For any subset $R$ of polynomials in $P(E)$, 
we can create the $^*$-ideal
$I(R)$ that is generated by $R$.
Considering $P(E)$ and $I(R)$ as algebras, we can certainly form 
the quotient algebra $P(E)/I(R)$.
Now, since $P(E)$ and $I(R)$ are $^*$-algebras,
the involution operation on $P(E)/I(R)$ given by
\begin{equation}
	(a + I)^* \coloneqq a^* + I
\end{equation}
for all $a \in P(E)$ is well-defined, and so in this case
$P(E)/I(R)$ is also a (quotient) $^*$-algebra.
Hence, we have that
\begin{definition}
	The \textbf{universal} $\bm{^*}$\textbf{-algebra
	with generators} $\bm{E}$ \textbf{and relations} $\bm{R}$,
	$A(E \mid R)$,
	is the quotient $^*$-algebra 
	$P(E)/I(R)$.
\end{definition}



We now aim to create a $C^*$-norm on a $^*$-algebra,
although this $^*$-algebra isn't quite $A(E \mid R)$.
For this, we first need the definition of a $C^*$-seminorm.

\begin{definition}
	A $\bm{C^*}$\textbf{-seminorm} on a $^*$-algebra $A$
	is a map $\gamma : A \rightarrow \mathbb{R}$ such that 
	\begin{enumerate}
		\item $\gamma$ is a seminorm, that is
			\begin{itemize}
				\item $\gamma(x) \geq 0$ for all $x \in A$
				\item $\gamma(\lambda{x}) = \lvert \lambda \rvert \gamma(x)$ for all $x \in A$ and $\lambda \in \mathbb{C}$
				\item $\gamma(x + y) \leq \gamma(x) + \gamma(y)$ for all $x,y \in A$
			\end{itemize}
		\item $\gamma$ is submultiplicative, that is, 
			$\gamma(xy) \leq \gamma(x)\gamma(y)$ for all $x,y \in A$,
			and
		\item $\gamma$ satisfies the $C^*$-identity, that is
			$\gamma(x^*x) = \gamma(x)^2$.
	\end{enumerate}
\end{definition}

Note that a norm on an algebra $A$ is a seminorm $\gamma : A \rightarrow \mathbb{R}$ 
such that if $\gamma(x) = 0$ then $x = 0$.
Hence a $C^*$-norm is a $C^*$-seminorm satisfying this additional property.

We make the following definition.

\begin{definition}
	We define a function $\norm{\cdot}_{A(E \mid R)} : A(E \mid R) \rightarrow \mathbb{R}$ by
	\begin{equation}
		\norm{x}_{A(E \mid R)} 
		\coloneqq
		\sup\{\gamma(x) \mid \gamma \text{ is a } C^*\text{-seminorm on } A(E \mid R)\}
	\end{equation}
\end{definition}
It is possible that 
$\norm{x}_{A(E \mid R)}$ is infinity for some $x \in A(E \mid R)$,
so we only allow ourselves to consider those $E$ and $R$ such that
$\norm{x}_{A(E \mid R)} < \infty$ for all $A(E \mid R)$.
In this case, one can show that $\norm{x}_{A(E \mid R)}$ is a $C^*$-seminorm:
for example, for all $x, y \in A(E \mid R)$ and any $C^*$-seminorm 
$\gamma$ on $A(E \mid R)$, we have that
\begin{equation}
	\gamma(xy) \leq \gamma(x)\gamma(y) \leq \norm{x}_{A(E \mid R)}\norm{y}_{A(E \mid R)}  
\end{equation}
and so we see that $\norm{\cdot}_{A(E \mid R)}$ is submultiplicative.
Also, we see that 
\begin{equation}
	I_0 \coloneqq \{x \in A(E \mid R) \mid \norm{x}_{A(E \mid R)} = 0\}
\end{equation}
is a $^*$-ideal.
Hence, $A(E \mid R)/I_0$ is a quotient $^*$-algebra having a $C^*$-norm
$\norm{x}_{A(E \mid R)/I_0}$
that is obtained from the $C^*$-seminorm on $A(E \mid R)$ 
under the quotient map $A(E \mid R) \rightarrow A(E \mid R)/I_0$.
Hence, we have that

\begin{definition}
	The \textbf{universal} $\bm{C^*}$\textbf{-algebra
	with generators} $\bm{E}$ \textbf{and relations} $\bm{R}$,
	$C^*(E \mid R)$,
	is the completion of $A(E \mid R)/I_0$ 
	with respect to $\norm{x}_{A(E \mid R)/I_0}$.
	It is indeed a $C^*$-algebra, by construction.
\end{definition}

The following lemma is very useful for showing the existence of the universal 
$C^*$-algebra $C^*(E \mid R)$.

\begin{lemma} \label{uniCstarexist}
	Let $E = \{x_i \mid i \in I\}$ (with $E^*$ understood)
	be a set of generators and let $R \subseteq P(E)$ be a set of relations.
	If there is a constant $C > 0$ such that the value of $\gamma$ on any
	generator is less than or equal to $C$
	and all $C^*$-seminorms $\gamma$ on $A(E \mid R)$, then $\norm{x} < \infty$
	for all $x \in A(E \mid R)$.
	Hence by the above construction $C^*(E \mid R)$ exists.
\end{lemma}

\begin{proof}
	Let $\gamma$ be any $C^*$-seminorm on $A(E \mid R)$.
	Then, since any monomial in $E \cup E^*$ has a positive integer length,
	we can pick an arbitrary monomial in $E \cup E^*$ having some arbitrary length
	$m$.
	Then, by the submultiplicative property, we have that the value of $\gamma$
	on this monomial is less than or equal to $C^m$, and so by taking the supremum
	over all $C^*$-seminorms we see that any element of $A(E \mid R)$ has a	norm
	that is bounded, as required.
\end{proof}

One can show 
\citet[Proposition 6.7]{speich}
that every universal $C^*$-algebra $C^*(E \mid R)$
satisfies the following so-called universal property.

\begin{proposition}[Universal Property] \label{unipropCstar}
	Let $E = \{x_i \mid i \in I\}$ 
	be a set of generators and let $R \subseteq P(E)$ be a set of relations 
	such that the universal $C^*$-algebra $C^*(E \mid R)$ exists.
	Suppose that $A$ is a $C^*$-algebra having a subset
	$E' = \{y_i \mid i \in I\}$ that \textbf{satisfies the relations} $R$,
	that is, all polynomials in $R$ become zero when we replace each $x_i$
	by $y_i$ for all $i \in I$.
	Then there exists a unique $^*$-homomorphism $\phi : C^*(E \mid R) \rightarrow A$
	such that $x_i \mapsto y_i$ for all $i \in I$.
\end{proposition}

\begin{example}
	For $n \geq 2$, we have a $^*$-isomorphism of $C^*$-algebras
	\begin{equation}
		M_n(\mathbb{C})
		\cong
		C^*(u_{i,j}, 1 \leq i,j \leq n \mid
		u_{i,j}^* = u_{j,i},
		u_{i,k}u_{l,j} = \delta_{k,l}u_{i,j}
		)
	\end{equation}
\end{example}

\begin{proof}
	Let $C^*(E \mid R) \coloneqq C^*(u_{i,j}, 1 \leq i,j \leq n \mid
		u_{i,j}^* = u_{j,i},
		u_{i,k}u_{l,j} = \delta_{k,l}u_{i,j}
		)$.
	We first need to show that $C^*(E \mid R)$ exists, that is, 
	that $\norm{x} < \infty$ for all $x \in A(E \mid R)$.
	This is immediate from
	Lemma \ref{uniCstarexist}
	if we can show that there is some constant $C$
	such that $\gamma(u_{i,j}) \leq C$ for all $1 \leq i,j \leq n$ and for all
	$C^*$-seminorms $\gamma$ on $A(E \mid R)$.

	Let $\gamma$ be a $C^*$-seminorm on $A(E \mid R)$.
	By the $C^*$-identity and the relations $R$ we have that
	\begin{equation}
		\gamma(u_{j,j})^2 
		= 
		\gamma(u_{j,j}^*u_{j,j})
		=
		\gamma(u_{j,j}u_{j,j})
		=
		\gamma(\delta_{j,j}u_{j,j})
		=
		\gamma(u_{j,j})
	\end{equation}
	Hence $\gamma({u_{j,j}})$ is either $0$ or $1$ for all $j \in [n]$ and so,
	again by the $C^*$-identity and the relations $R$, we have that
	\begin{equation}
		\gamma(u_{i,j})^2
		=
		\gamma(u_{i,j}^*u_{i,j})
		=
		\gamma(u_{j,i}u_{i,j})
		=
		\gamma(\delta_{i,i}u_{j,j})
		=
		\gamma(u_{j,j})
	\end{equation}
	Hence $\gamma({u_{i,j}})$ is bounded above by $1$ for all $1 \leq i,j \leq n$
	and all $C^*$-seminorms $\gamma$ on $A(E \mid R)$, as required.

	We now prove the $^*$-isomorphism $M_n(\mathbb{C}) \cong C^*(E \mid R)$.
	The matrix units $E_{i,j} \in M_n(\mathbb{C})$
	satisfy the relations $R$, since $E_{i,j}^* = E_{i,j}^\top = E_{j,i}$
	and $E_{i,k}E_{l,j} = \delta_{k,l}E_{i,j}$, hence, by the universal property,
	there exists a surjective $^*$-homomorphism 
	$\phi: C^*(E \mid R) \rightarrow M_n(\mathbb{C})$ such that 
	$u_{i,j} \mapsto E_{i,j}$ for all $i,j \in [n]$.
	Looking at the relations $R$ we see that the monomials in $C^*(E \mid R)$
	can be at most the elements $u_{i,j}$, but since $\phi$ is surjective
	the monomials are exactly the elements $u_{i,j}$. Hence, as a vector space,
	$C^*(E \mid R)$ has dimension $n^2$ and so $\phi$ is also injective.
	Hence $\phi$ is a $^*$-isomorphism, as required.
\end{proof}


\section{Missing Proofs and Supplementary Content} \label{missingproofs}

We include proofs of results that appeared in the main paper, and 
provide the full statement of Woronowicz--Tannaka--Krein Duality
\citep{woronowicz1988}, the first part of which appeared in
Theorem \ref{tannakakrein}.

\begin{proof}[Proof of Proposition \ref{compmatgroup}]
	Let $G(n) \subseteq GL(n)$ be a compact matrix group. 
	Consider $C(G(n))$, the $C^{*}$-algebra of continuous functions 
	$G(n) \rightarrow \mathbb{C}$, which is commutative.
	Firstly, we can define functions
	$u_{i,j} : G(n) \rightarrow \mathbb{C}$ 
	for $i,j \in [n]$
	such that $u_{i,j}(g) = g_{i,j}$. 
	By the Stone-Weierstrass Theorem (Theorem \ref{stoneweierstrass}) 
	and the compactness of $G(n)$,
	the $u_{i,j}$ generate $C(G(n))$ 
	and the matrices $u \coloneqq (u_{i,j})$ and $u^\top$ are invertible
	since $u^{-1}(g) = u(g^{-1})$.

	Moreover, we have that the comultiplication map $\Delta$ for $C(G(n))$
	is given by 
	\begin{equation} \label{deltaCGndefn}
		\Delta(u_{i,j})(g,h) = u_{i,j}(gh)
	\end{equation}
since
	\begin{equation} 
	\left(
	\sum_k 
	u_{i,k} \otimes u_{k,j}
	\right)
	(g, h)
	=
	\sum_k 
	u_{i,k}(g)
	u_{k,j}(h) 
	=
	\sum_k 
	g_{i,k}
	h_{k,j}
	=
	(gh)_{i,j}
	=
	u_{i,j}(gh)
	\end{equation}
that is, $\Delta$ coincides with matrix multiplication in $G(n)$.
	
	The definition of $\Delta$ that is given in (\ref{deltaCGndefn}) 
	is a $*$-homomorphism because
	\begin{align}
		\Delta(u_{i,j}u_{k,l})(g,h)
		& =
		u_{i,j}u_{k,l}(gh) \\
		& =
		u_{i,j}(gh)u_{k,l}(gh) \label{CGmultrule} \\
		& =
		\Delta(u_{i,j})(g,h)\Delta(u_{k,l})(g,h) \\
		& =
		\Delta(u_{i,j})\Delta(u_{k,l})(g,h) \label{CGmultrule2}
	\end{align}
	and
	\begin{equation}
		\Delta(u_{i,j}^*)(g,h) 
		= 
		u_{i,j}^*(gh)
		=
		\overline{u_{i,j}(gh)}
		=
		\overline{\Delta(u_{i,j})(g,h)}
		=
		\Delta(u_{i,j})^*(g,h) 
	\end{equation}
	Note that in (\ref{CGmultrule}) we have used that $C(G(n))$ has a pointwise
	multiplication, and likewise, in (\ref{CGmultrule2}), we have used
	that $C(G(n)) \times C(G(n))$ has a pointwise multiplication.
\end{proof}

\begin{proof}[Proof of Proposition \ref{tensconjcmqg}]
	We have that $v \otimes w$ is a representation of $G(n)$ because
	\begin{equation} \label{qgtenprodrepproof}
		\Delta(v_{i,j}w_{k,l})
		=
		\Delta(v_{i,j})\Delta(w_{k,l})
		=
		\left(
			\sum_x 
			v_{i,x} \otimes v_{x,j}
		\right)
		\left(
			\sum_y 
			w_{k,y} \otimes w_{y,l}
		\right)
		=
		\sum_{x,y} v_{i,x}w_{k,y} \otimes v_{x,j}w_{y,l}
	\end{equation}
	and $\bar{v}$ is a representation of $G(n)$ because
	\begin{equation} \label{qgconjproof}
		\Delta(v_{i,j}^*)
		=
		\Delta(v_{i,j})^*
		=
		\left(
			\sum_k 
			v_{i,k} \otimes v_{k,j}
		\right)^*
		=
			\sum_k 
			v_{i,k}^* \otimes v_{k,j}^*
	\end{equation}
	where in both (\ref{qgtenprodrepproof}) and (\ref{qgconjproof})
	we have used that $\Delta$ is a $*$-homomorphism.
\end{proof}

\begin{proof}[Proof of Lemma \ref{conjunitary}]
	We have that
	\begin{equation}
		\bar{u}^*
		=
		(u_{i,j}^*)^*
		=
		((u_{j,i}^*)^*)
		=
		(u_{j,i})
		=
		u^\top
	\end{equation}
	and also
	\begin{equation}
		(u^\top)^*
		=
		(u_{j,i})^*
		=
		(u_{i,j}^*)
		=
		\bar{u}
	\end{equation}
	Hence
	\begin{equation}
		u^\top(u^\top)^* 
		= 
		1 
		= 
		(u^\top)^*{u^\top}
		\iff
		\bar{u}{\bar{u}^*}
		=
		1
		=
		{\bar{u}^*}\bar{u}
	\end{equation}
	as required.	
\end{proof}

\begin{proof}[Proof of Theorem \ref{cmqgtwocoloured}]
	It is easy to show that $\FundRep_{G(n)}$ satisfies the
	first four axioms of Definition \ref{twocolourrepcat}.
	For example, the second axiom says that if 
	$\phi_1 \in \FundRep_{G(n)}(w_k, w_l)$ and
	$\phi_2 \in \FundRep_{G(n)}(w_l, w_m)$, then we have that
	\begin{equation}
		(\phi_2 \circ \phi_1){u^{\otimes w_k}}
		=
		\phi_2{u^{\otimes w_l}}\phi_1
		=
		{u^{\otimes w_m}}(\phi_2 \circ \phi_1)
	\end{equation}
	and so $\phi_2 \circ \phi_1 \in \FundRep_{G(n)}(w_k, w_m)$.

	For the fifth axiom, we need to show that the operator
	$R : 1 \mapsto \sum_{i=1}^{n} e_i \otimes e_i$
	is in both $\FundRep_{G(n)}(\varnothing, \circ\bullet)$
	and $\FundRep_{G(n)}(\varnothing, \bullet\circ)$.

	For the first case, we have that 
	\begin{equation} \label{twocolourdualeq1}
	R \in \FundRep_{G(n)}(\varnothing, \circ\bullet)
		\iff
	R = (u \otimes \bar{u})R
	\end{equation}
	By applying both sides to 
	$1$, we have that (\ref{twocolourdualeq1}) holds if and only if
	\begin{equation}
		\sum_{i=1}^{n} e_i \otimes e_i 
		=
		\sum_{i=1}^{n} ue_i \otimes \bar{u}e_i
		=
		\sum_{i=1}^{n} \sum_{j,k=1}^{n} e_j \otimes e_k \otimes u_{j,i}u_{k,i}^{*}
		=	
		\sum_{j,k=1}^{n} e_j \otimes e_k \otimes (uu^*)_{j,k}
	\end{equation}
	Hence, by equating coefficients, we see that
	$R \in \FundRep_{G(n)}(\varnothing, \circ\bullet)$ if and only if
	$uu^* = 1$.
	For the second case, a similar calculation shows that
	$R \in \FundRep_{G(n)}(\varnothing, \bullet\circ)$ if and only if
	$\bar{u}\bar{u}^* = 1$.
	Since $u$ and $\bar{u}$ are both unitary, we obtain the result.
\end{proof}

As promised in Remark \ref{woronremark},
we provide the full statement of Woronowicz--Tannaka--Krein Duality
\citep{woronowicz1988} here, the first part of which already appeared in
Theorem \ref{tannakakrein}.

\begin{theorem} \label{tannakakreinfull}
	[\textbf{Woronowicz--Tannaka--Krein Duality}]
	Let $\mathcal{C}$ be a two-coloured representation category.
	Then there exists a unique compact matrix quantum group $(G(n), u)$ such that
	$\FundRep_{G(n)} = \mathcal{C}$.
	If $E = \{x_{i,j} \mid 1 \leq i,j \leq n\}$ is a set of generators for the
	free $^*$-algebra $P(E)$, then the $^*$-ideal $I_{G(n)} \subseteq P(E)$
	that determines $G(n)$ is given by
	\begin{equation}
		\text{span}
		\left\{ 
		[\phi x^{\otimes w_k} - x^{\otimes w_l} \phi]_{I,J}
		\;\middle\vert\;
			\phi \in \mathcal{C}(w_k, w_l), I \in [n]^l, J \in [n]^k
		\right\}
	\end{equation}
\end{theorem}

In Definition \ref{cattwopart},
we also said that we would provide definitions of 
the following three $\mathbb{C}$-(bi)linear operations on two-coloured set partition diagrams
\begin{align}
	\text{\textbf{composition:}   } & \bullet: P_{w_l}^{w_m}(n) \times P_{w_k}^{w_l}(n) \rightarrow P_{w_k}^{w_m}(n) \\
	\text{\textbf{tensor product:}   } & \otimes: P_{w_k}^{w_l}(n) \times P_{w_q}^{w_m}(n) \rightarrow P_{w_k \cdot w_q}^{w_l \cdot w_m}(n) \\
	\text{\textbf{involution:}   } & *: P_{w_k}^{w_l}(n)
	\rightarrow P_{w_l}^{w_k}(n)
\end{align}
We define these operations as follows.

\begin{definition}[Composition] \label{qgroupcomposition}

Let $d_{\pi_1} \in P_{w_k}^{w_l}(n)$ and $d_{\pi_2} \in P_{w_l}^{w_m}(n)$. 
		First, we concatenate the diagrams, written $d_{\pi_2} \circ d_{\pi_1}$, by putting $d_{\pi_1}$ below $d_{\pi_2}$, concatenating the edges in the middle row of vertices, and then removing all connected components that lie entirely in the middle row of the concatenated diagrams. 
Let $c(d_{\pi_2}, d_{\pi_1})$ be the number of connected components that are removed from the middle row in $d_{\pi_2} \circ d_{\pi_1}$.
		Then the composition is defined, using infix notation, as
\begin{equation} 
	d_{\pi_2} \bullet d_{\pi_1} 
	\coloneqq 
	n^{c(d_{\pi_2}, d_{\pi_1})} (d_{\pi_2} \circ d_{\pi_1})
\end{equation}
\end{definition}

\begin{example}
If $d_{\pi_2}$ is the two-coloured 
$(\circ \circ \bullet \bullet \circ \, \bullet, 
\circ \bullet \bullet \, \circ)$--partition diagram 
\begin{equation}
\begin{aligned}
	\scalebox{0.5}{\begin{tikzpicture}
	\begin{pgfonlayer}{nodelayer}
		\node [style={filled_node}] (1) at (6, 0) {};
		\node [style=none] (3) at (-0.5, 0.5) {};
		\node [style=none] (4) at (-0.5, -4.5) {};
		\node [style=node] (7) at (3, 0) {};
		\node [style=node] (8) at (0, -4) {};
		\node [style={filled_node}] (10) at (9, -4) {};
		\node [style=none] (13) at (15.5, -4.5) {};
		\node [style=none] (14) at (15.5, 0.5) {};
		\node [style=node] (15) at (12, 0) {};
		\node [style=node] (17) at (12, -4) {};
		\node [style={filled_node}] (19) at (15, -4) {};
		\node [style={filled_node}] (20) at (9, 0) {};
		\node [style={filled_node}] (21) at (6, -4) {};
		\node [style=node] (22) at (3, -4) {};
	\end{pgfonlayer}
	\begin{pgfonlayer}{edgelayer}
		\draw [style={dash_2}] (3.center) to (4.center);
		\draw [style={dash_2}] (13.center) to (14.center);
		\draw [style={dash_1}] (4.center) to (13.center);
		\draw [style={dash_1}] (14.center) to (3.center);
		\draw [style=thick] (19) to (15);
		\draw [style=thick] (8) to (1);
		\draw [style=thick, bend left=45, looseness=0.75] (22) to (10);
		\draw [style=thick] (21) to (7);
		\draw [style=thick] (17) to (20);
	\end{pgfonlayer}
\end{tikzpicture}}
\end{aligned}
\end{equation}
and
$d_{\pi_1}$ is the two-coloured 
$(\bullet \circ \bullet \circ \circ \, \circ, 
\circ \circ \bullet \bullet \circ \, \bullet)$--partition diagram 
\begin{equation}
\begin{aligned}
	\scalebox{0.5}{\begin{tikzpicture}
	\begin{pgfonlayer}{nodelayer}
		\node [style={filled_node}] (0) at (0, -4) {};
		\node [style={filled_node}] (1) at (6, -4) {};
		\node [style=none] (3) at (-0.5, 0.5) {};
		\node [style=none] (4) at (-0.5, -4.5) {};
		\node [style=node] (5) at (3, -4) {};
		\node [style=node] (6) at (0, 0) {};
		\node [style={filled_node}] (8) at (9, 0) {};
		\node [style=none] (9) at (15.5, -4.5) {};
		\node [style=none] (10) at (15.5, 0.5) {};
		\node [style=node] (11) at (12, -4) {};
		\node [style=node] (13) at (12, 0) {};
		\node [style={filled_node}] (14) at (15, 0) {};
		\node [style={filled_node}] (16) at (6, 0) {};
		\node [style=node] (17) at (3, 0) {};
		\node [style=node] (18) at (9, -4) {};
		\node [style=node] (20) at (15, -4) {};
	\end{pgfonlayer}
	\begin{pgfonlayer}{edgelayer}
		\draw [style={dash_2}] (3.center) to (4.center);
		\draw [style={dash_2}] (9.center) to (10.center);
		\draw [style={dash_1}] (4.center) to (9.center);
		\draw [style={dash_1}] (10.center) to (3.center);
		\draw [style=thick] (20) to (14);
		\draw [style=thick, bend left=300] (11) to (18);
		\draw [style=thick] (0) to (16);
		\draw [style=thick, bend right=45, looseness=0.75] (17) to (8);
		\draw [style=thick] (5) to (6);
		\draw [style=thick] (1) to (13);
	\end{pgfonlayer}
\end{tikzpicture}}
\end{aligned}
\end{equation}
then
we have that $d_{\pi_2} \circ d_{\pi_1}$ is 
the 
$(\bullet \circ \bullet \circ \circ \, \circ,
\circ \bullet \bullet \, \circ)$--partition diagram
\begin{equation} \label{compdiagram}
\begin{aligned}
	\scalebox{0.5}{\begin{tikzpicture}
	\begin{pgfonlayer}{nodelayer}
		\node [style={filled_node}] (1) at (6, 0) {};
		\node [style={filled_node}] (2) at (6, -4) {};
		\node [style=none] (3) at (-0.5, 0.5) {};
		\node [style=none] (4) at (-0.5, -4.5) {};
		\node [style=node] (7) at (3, 0) {};
		\node [style=none] (13) at (15.5, -4.5) {};
		\node [style=none] (14) at (15.5, 0.5) {};
		\node [style=node] (15) at (12, 0) {};
		\node [style=node] (16) at (3, -4) {};
		\node [style=node] (17) at (12, -4) {};
		\node [style={filled_node}] (20) at (9, 0) {};
		\node [style={filled_node}] (21) at (0, -4) {};
		\node [style=node] (22) at (9, -4) {};
		\node [style=node] (23) at (15, -4) {};
	\end{pgfonlayer}
	\begin{pgfonlayer}{edgelayer}
		\draw [style={dash_2}] (3.center) to (4.center);
		\draw [style={dash_2}] (13.center) to (14.center);
		\draw [style={dash_1}] (4.center) to (13.center);
		\draw [style={dash_1}] (14.center) to (3.center);
		\draw [style=thick] (21) to (7);
		\draw [style=thick] (16) to (1);
		\draw [style=thick] (23) to (15);
		\draw [style=thick, bend left=300] (17) to (22);
		\draw [style=thick] (2) to (20);
	\end{pgfonlayer}
\end{tikzpicture}}
\end{aligned}
\end{equation}
and so $d_{\pi_2} \bullet d_{\pi_1}$ is the diagram (\ref{compdiagram}) multiplied by $n$, since one connected component was removed from the middle row of $d_{\pi_2} \circ d_{\pi_1}$.
\end{example}
		
\begin{definition}[Tensor Product] \label{qgrouptensorprod}
Let $d_{\pi_1} \in P_{w_k}^{w_l}(n)$ and $d_{\pi_2} \in P_{w_q}^{w_m}(n)$. Then $d_{\pi_1} \otimes d_{\pi_2}$ is defined to be the $(w_k \cdot w_q, w_l \cdot w_m)$--partition diagram obtained by horizontally placing $d_{\pi_1}$ to the left of $d_{\pi_2}$ without any overlapping of vertices.
\end{definition}

\begin{example}
If $d_{\pi_1}$ is the two-coloured 
$(\circ \, \bullet, 
\bullet \circ \bullet)$--partition diagram 
\begin{equation}
\begin{aligned}
	\scalebox{0.5}{\begin{tikzpicture}
	\begin{pgfonlayer}{nodelayer}
		\node [style={filled_node}] (64) at (0, 0) {};
		\node [style={filled_node}] (66) at (6, 0) {};
		\node [style={filled_node}] (75) at (4.5, -4) {};
		\node [style=none] (104) at (-0.5, 0.5) {};
		\node [style=none] (105) at (-0.5, -4.5) {};
		\node [style=none] (106) at (6.5, -4.5) {};
		\node [style=none] (107) at (6.5, 0.5) {};
		\node [style=node] (108) at (3, 0) {};
		\node [style=node] (112) at (1.5, -4) {};
	\end{pgfonlayer}
	\begin{pgfonlayer}{edgelayer}
		\draw [style={dash_2}] (105.center) to (106.center);
		\draw [style={dash_2}] (106.center) to (107.center);
		\draw [style={dash_2}] (107.center) to (104.center);
		\draw [style={dash_2}] (104.center) to (105.center);
		\draw [style=thick] (112) to (64);
		\draw [style=thick] (112) to (108);
		\draw [style=thick] (75) to (66);
	\end{pgfonlayer}
\end{tikzpicture}}
\end{aligned}
\end{equation}
and
$d_{\pi_2}$ is the two-coloured 
$(\circ \bullet \circ \, \bullet, 
\circ \circ \bullet)$--partition diagram 
\begin{equation}
\begin{aligned}
	\scalebox{0.5}{\begin{tikzpicture}
	\begin{pgfonlayer}{nodelayer}
		\node [style={filled_node}] (66) at (7.5, 0) {};
		\node [style={filled_node}] (75) at (3, -4) {};
		\node [style=none] (104) at (-0.5, 0.5) {};
		\node [style=none] (105) at (-0.5, -4.5) {};
		\node [style=none] (106) at (9.5, -4.5) {};
		\node [style=none] (107) at (9.5, 0.5) {};
		\node [style=node] (108) at (4.5, 0) {};
		\node [style=node] (112) at (0, -4) {};
		\node [style=node] (113) at (6, -4) {};
		\node [style=node] (114) at (1.5, 0) {};
		\node [style={filled_node}] (115) at (9, -4) {};
	\end{pgfonlayer}
	\begin{pgfonlayer}{edgelayer}
		\draw [style={dash_2}] (105.center) to (106.center);
		\draw [style={dash_2}] (106.center) to (107.center);
		\draw [style={dash_2}] (107.center) to (104.center);
		\draw [style={dash_2}] (104.center) to (105.center);
		\draw [style=thick, bend left=315] (115) to (113);
		\draw [style=thick, bend right=45] (114) to (108);
		\draw [style=thick, bend left=45] (112) to (75);
		\draw [style=thick] (75) to (66);
		\draw [style=thick] (113) to (108);
	\end{pgfonlayer}
\end{tikzpicture}}
\end{aligned}
\end{equation}
then
we have that $d_{\pi_1} \otimes d_{\pi_2}$ is the 
$(\circ \bullet \circ \bullet \circ \, \bullet,
\bullet \circ \bullet \circ \circ \, \bullet)$--partition diagram
\begin{equation}
\begin{aligned}
	\scalebox{0.5}{\begin{tikzpicture}
	\begin{pgfonlayer}{nodelayer}
		\node [style={filled_node}] (0) at (0, 0) {};
		\node [style={filled_node}] (1) at (6, 0) {};
		\node [style={filled_node}] (2) at (3, -4) {};
		\node [style=none] (3) at (-0.5, 0.5) {};
		\node [style=none] (4) at (-0.5, -4.5) {};
		\node [style=node] (7) at (3, 0) {};
		\node [style=node] (8) at (0, -4) {};
		\node [style={filled_node}] (9) at (15, 0) {};
		\node [style={filled_node}] (10) at (9, -4) {};
		\node [style=none] (13) at (15.5, -4.5) {};
		\node [style=none] (14) at (15.5, 0.5) {};
		\node [style=node] (15) at (12, 0) {};
		\node [style=node] (16) at (6, -4) {};
		\node [style=node] (17) at (12, -4) {};
		\node [style=node] (18) at (9, 0) {};
		\node [style={filled_node}] (19) at (15, -4) {};
	\end{pgfonlayer}
	\begin{pgfonlayer}{edgelayer}
		\draw [style={dash_2}] (3.center) to (4.center);
		\draw [style=thick] (8) to (0);
		\draw [style=thick] (8) to (7);
		\draw [style=thick] (2) to (1);
		\draw [style={dash_2}] (13.center) to (14.center);
		\draw [style=thick, bend left=315] (19) to (17);
		\draw [style=thick, bend right=45] (18) to (15);
		\draw [style=thick, bend left=45] (16) to (10);
		\draw [style=thick] (10) to (9);
		\draw [style=thick] (17) to (15);
		\draw [style={dash_1}] (4.center) to (13.center);
		\draw [style={dash_1}] (14.center) to (3.center);
	\end{pgfonlayer}
\end{tikzpicture}}
\end{aligned}
\end{equation}
\end{example}

\begin{definition}[Involution] \label{qgroupinvolution}
Let $d_\pi \in P_{w_k}^{w_l}(n)$.
Then $d_\pi^{*} \in P_{w_l}^{w_k}(n)$
is defined by reflecting $d_\pi$ in the horizontal axis.
\end{definition}

\begin{example}
Recalling the two-coloured 
$(\circ \bullet \bullet \circ \bullet \, \circ, 
		\bullet \circ \bullet \bullet \circ)$--partition diagram $d_\pi$
from Example \ref{twocolsetpart}
(reprinted below for ease)
\begin{equation}
\begin{aligned}
	\scalebox{0.5}{\begin{tikzpicture}
	\begin{pgfonlayer}{nodelayer}
		\node [style={filled_node}] (64) at (1.5, 0) {};
		\node [style={filled_node}] (66) at (7.5, 0) {};
		\node [style={filled_node}] (67) at (10.5, 0) {};
		\node [style={filled_node}] (75) at (3, -4) {};
		\node [style={filled_node}] (76) at (6, -4) {};
		\node [style={filled_node}] (78) at (12, -4) {};
		\node [style=none] (104) at (-0.5, 0.5) {};
		\node [style=none] (105) at (-0.5, -4.5) {};
		\node [style=none] (106) at (15.5, -4.5) {};
		\node [style=none] (107) at (15.5, 0.5) {};
		\node [style=node] (108) at (4.5, 0) {};
		\node [style=node] (109) at (13.5, 0) {};
		\node [style=node] (110) at (9, -4) {};
		\node [style=node] (111) at (15, -4) {};
		\node [style=node] (112) at (0, -4) {};
	\end{pgfonlayer}
	\begin{pgfonlayer}{edgelayer}
		\draw [style={dash_2}] (105.center) to (106.center);
		\draw [style={dash_2}] (106.center) to (107.center);
		\draw [style={dash_2}] (107.center) to (104.center);
		\draw [style={dash_2}] (104.center) to (105.center);
		\draw [style=thick, bend left=45, looseness=0.50] (75) to (78);
		\draw [style=thick] (112) to (64);
		\draw [style=thick, bend right=60] (108) to (66);
		\draw [style=thick] (109) to (110);
		\draw [style=thick] (76) to (67);
	\end{pgfonlayer}
\end{tikzpicture}}
\end{aligned}
\end{equation}
we have that $d_\pi^{*}$ is the 
$(\bullet \circ \bullet \bullet \circ,
\circ \bullet \bullet \circ \bullet \, \circ)$--partition diagram
\begin{equation}
\begin{aligned}
	\scalebox{0.5}{\begin{tikzpicture}
	\begin{pgfonlayer}{nodelayer}
		\node [style={filled_node}] (64) at (1.5, 1) {};
		\node [style={filled_node}] (66) at (7.5, 1) {};
		\node [style={filled_node}] (67) at (10.5, 1) {};
		\node [style={filled_node}] (75) at (3, 5) {};
		\node [style={filled_node}] (76) at (6, 5) {};
		\node [style={filled_node}] (78) at (12, 5) {};
		\node [style=none] (104) at (-0.5, 0.5) {};
		\node [style=none] (105) at (-0.5, 5.5) {};
		\node [style=none] (106) at (15.5, 5.5) {};
		\node [style=none] (107) at (15.5, 0.5) {};
		\node [style=node] (108) at (4.5, 1) {};
		\node [style=node] (109) at (13.5, 1) {};
		\node [style=node] (110) at (9, 5) {};
		\node [style=node] (111) at (15, 5) {};
		\node [style=node] (112) at (0, 5) {};
	\end{pgfonlayer}
	\begin{pgfonlayer}{edgelayer}
		\draw [style={dash_2}] (105.center) to (106.center);
		\draw [style={dash_2}] (106.center) to (107.center);
		\draw [style={dash_2}] (107.center) to (104.center);
		\draw [style={dash_2}] (104.center) to (105.center);
		\draw [style=thick, bend right=45, looseness=0.50] (75) to (78);
		\draw [style=thick] (112) to (64);
		\draw [style=thick, bend left=60] (108) to (66);
		\draw [style=thick] (109) to (110);
		\draw [style=thick] (76) to (67);
	\end{pgfonlayer}
\end{tikzpicture}}
\end{aligned}
\end{equation}
\end{example}

With this, we can now provide the following proof of Proposition \ref{twocolmonoidal}.

\begin{proof}[Proof of Proposition \ref{twocolmonoidal}]
	$\mathcal{P}(n)$ is a strict monoidal category
	because the bifunctor (horizontal operation) 
	on objects reduces to the concatenation operation on 
	words, which is associative, and
	the bifunctor on morphisms is 
	the tensor product of linear combinations of two-coloured partition diagrams
	that is given in Definition \ref{qgrouptensorprod},
	which is associative by definition.

	$\mathcal{P}(n)$ is $\mathbb{C}$--linear because the morphism space 
	between any two objects is by definition a vector space,
	and the composition of morphisms is $\mathbb{C}$-bilinear 
	by definition.
	For the same reason, the bifunctor is also $\mathbb{C}$--bilinear.	
\end{proof}

\begin{proof}[Proof of Theorem \ref{quantumpartfunctor}]
	We prove this theorem in a series of steps.
	We first consider the image $\mathcal{C}_{\mathcal{P}}(n)$ of
	the two-coloured partition category $\mathcal{P}(n)$ 
	under $d_\pi \mapsto D_\pi$.

	\textbf{Step 1}: We need to show that $\mathcal{C}_{\mathcal{P}}(n)$
	is actually a category. 

	The objects and the morphisms have been defined in the statement of the theorem.
	For every triple of objects, the composition rule is given by matrix 
	multiplication, which is associative.
	Furthermore, the identity morphism for every object $w_k$ of length $k$ is the 
	$n^k \times n^k$ identity matrix.
	Note also that the unit object is $\varnothing$, the empty word.

	\textbf{Step 2}: We need to show that $\Theta : \mathcal{P}(n) \rightarrow \mathcal{C}_{\mathcal{P}}(n)$ is a functor.

	\begin{itemize}
		\item $\Theta$ maps the objects of $\mathcal{P}(n)$ to the objects of
			$\mathcal{C}_{\mathcal{P}}(n)$ since 
			$\Theta$ maps each word $w_k$ to itself, by definition.
		\item It is enough to show that $\Theta(g)\Theta(f) = \Theta(g \bullet f)$ on arbitrary basis elements of arbitrary morphism spaces where
			the codomain of $f$ is the domain of $g$
			because the morphism spaces are vector spaces. 

			Let $f = d_{\pi_1}$ be a $(w_k,w_l)$--partition diagram, and
			let $g = d_{\pi_2}$ be a $(w_l,w_m)$--partition diagram.
			Then, by Definition \ref{qgroupcomposition}, we have that
			\begin{equation}
				d_{\pi_2} \bullet d_{\pi_1} 
				=
				n^{c(d_{\pi_2}, d_{\pi_1})} d_{\pi_3}
			\end{equation}
			where $d_{\pi_3}$ be the composition $d_{\pi_2} \circ d_{\pi_1}$
			expressed as a $(w_k,w_m)$--partition diagram.
			Then 
			\begin{equation} \label{partbulletmatmult}
				\Theta(g \bullet f) 
				= 
				\Theta(d_{\pi_2} \bullet d_{\pi_1}) 
				=
				n^{c(d_{\pi_2}, d_{\pi_1})} D_{\pi_3}
			\end{equation}
			We also have that
			\begin{align}
				\Theta(g)\Theta(f)
				=
				D_{\pi_2}D_{\pi_1}
				& =
				\left(
				\sum_{I \in [n]^m, K \in [n]^l}
				\delta_{\pi_2, (I,K)}
				E_{I,K}
				\right)
				\left(
				\sum_{L \in [n]^l, J \in [n]^k}
				\delta_{\pi_1, (L,J)}
				E_{L,J} 
				\right)
				\\
				& =
				\sum_{I \in [n]^m, K \in [n]^l, J \in [n]^k}
				\delta_{\pi_2, (I,K)}
				\delta_{\pi_1, (K,J)}
				E_{I,K}E_{K,J} \label{partcompmultmat}
			\end{align}

			For fixed $I, J$, consider
			\begin{equation} \label{partcompdeltas}
				\sum_{K \in [n]^l}
				\delta_{\pi_2, (I,K)}
				\delta_{\pi_1, (K,J)}
			\end{equation}
			This is equal to
			\begin{equation}
				n^{c(d_{\pi_2}, d_{\pi_1})}
				\delta_{\pi_3, (I,J)}
			\end{equation}
			Indeed, for fixed $I,J$, 
			$\delta_{\pi_3, (I,J)}$ is $1$
			if and only if 
			both $\delta_{\pi_2, (I,K)}$ and 
			$\delta_{\pi_1, (K,J)}$ are $1$
			for some $K \in [n]^l$
			since $d_{\pi_3}$ is the composition
			$d_{\pi_2} \circ d_{\pi_1}$.
			The number of such $K$ is determined
			only by the vertices that appear in connected components
			that are removed from the middle row of 
			$d_{\pi_2} \circ d_{\pi_1}$, since, for fixed $I,J$,
			only these vertices can be freely chosen
			if we want both $\delta_{\pi_2, (I,K)}$ and
			$\delta_{\pi_1, (K,J)}$ to be $1$.
			However, since the entries in each connected component must 
			be the same for both $\delta_{\pi_2, (I,K)}$ and 
			$\delta_{\pi_1, (K,J)}$ to be $1$,
			this implies that the number of $K \in [n]^l$ such that
			both $\delta_{\pi_2, (I,K)}$ and 
			$\delta_{\pi_1, (K,J)}$ are $1$
			is 
			$n^{c(d_{\pi_2}, d_{\pi_1})}$.

			Hence (\ref{partcompmultmat}) becomes
			\begin{equation}
				\sum_{I \in [n]^m, J \in [n]^k}
				n^{c(d_{\pi_2}, d_{\pi_1})}
				\delta_{\pi_3, (I,J)}
				E_{I,J}
				=
				n^{c(d_{\pi_2}, d_{\pi_1})} D_{\pi_3}
			\end{equation}
			and so, by (\ref{partbulletmatmult}),
			we have that $\Theta(g)\Theta(f) = \Theta(g \bullet f)$, as required.
		\item 
			For each object $w_k$ in $\mathcal{P}(n)$,
			the identity morphism $1_{w_k}$ is the 
			two-coloured $(w_k, w_k)$--partition diagram $d_\pi$
			where $\pi$ is the set partition
			\begin{equation}
				\{1, k+1 \mid 2, k+2, \mid \dots \mid k, 2k\}
			\end{equation}
			of $[2k]$, and its image under $d_\pi \mapsto D_\pi$ is
			the $n^k \times n^k$ identity matrix, as required. 
	\end{itemize}

	\textbf{Step 3}: We need to show that $\mathcal{C}_{\mathcal{P}}(n)$
	is a strict $\mathcal{C}$-linear monoidal category.

	We first consider the tensor product of arbitrary basis elements 
	of arbitrary morphism spaces.
	Let $d_\pi$ be a $(w_k,w_l)$--partition diagram,
	and let $d_\tau$ be a $(w_m,w_q)$--partition diagram.
	Then, defining $d_{\omega}$ to be the $(w_k \cdot w_q, w_l \cdot w_m)$--partition
	diagram, where $\omega$ is the set partition $\pi \cup \tau$ of $[l+m+k+q]$,
	we see that
	
	\begin{align} \label{imagestensorprod}
	D_\pi \otimes D_\tau 
		& =
		\left(
		\sum_{I \in [n]^l, J \in [n]^k}
		\delta_{\pi, (I,J)}
		E_{I,J}
		\right)
		\otimes
		\left(
		\sum_{X \in [n]^m, Y \in [n]^q}
		\delta_{\tau, (X,Y)}
		E_{X,Y} 
		\right)
		\\
		& =
		\sum_{(I,X) \in [n]^{l+m}, (J,Y) \in [n]^{k+q}}
		\delta_{\omega, (I,X),(J,Y))}
		E_{(I,X),(J,Y)} \label{deltaprod} \\
		& = D_\omega
	\end{align}
	since $S_\pi((I,J)) \cup S_\tau((X,Y)) = S_\omega((I,X),(J,Y))$.

	The bifunctor on objects of $\mathcal{C}_{\mathcal{P}}(n)$ 
	reduces to the concatenation
	operation on words, which is associative. 
	Also, since the bifunctor on
	morphisms in $\mathcal{C}_{\mathcal{P}}(n)$ 
	is the tensor product of linear combinations of images
	of two-coloured partition diagrams, it is associative because
	the tensor product of images of two-coloured partition diagrams is associative,
	both by the above calculation and because taking unions of set partitions 
	is associative.
	Hence, $\mathcal{C}_{\mathcal{P}}(n)$ is a strict monoidal category.

	$\mathcal{C}_{\mathcal{P}}(n)$ is $\mathbb{C}$--linear because 
	the morphism space between any two objects is by definition a vector space,
	and the composition of morphisms is $\mathbb{C}$-bilinear 
	by definition.
	For the same reason, the bifunctor is also $\mathbb{C}$--bilinear.

	\textbf{Step 4}: We need to show that $\Theta : \mathcal{P}(n) \rightarrow \mathcal{C}_{\mathcal{P}}(n)$ is a strict $\mathcal{C}$-linear monoidal functor.

	\begin{enumerate}
		\item $\Theta$ preserves the tensor product on objects, 
			since $\Theta(w_k) = w_k$ for any object in $\mathcal{P}(n)$, 
			and the tensor product in both categories is given by 
			the concatenation of words.
		\item It is enough to show that $\Theta(f) \otimes \Theta(g) = \Theta(f \otimes g)$ on arbitrary basis elements of arbitrary morphism spaces as the morphism spaces are vector spaces. This result is now immediate from the calculation given in 
		(\ref{imagestensorprod}).
		\item It is clear that
			$\Theta$ sends the unit object $\varnothing$ 
			in $\mathcal{P}(n)$ to $\varnothing$,
			which is the unit object in $\mathcal{C}_{\mathcal{P}}(n)$.
		\item This is immediate because the map $d_\pi \mapsto D_\pi$
			is $\mathbb{C}$--linear by Definition \ref{partmatfunctor}.
	\end{enumerate}

	\textbf{Step 5:} We need to show that $\mathcal{C}_{\mathcal{P}}(n)$ is a
	two-coloured representation category.
	
	The tensor product, concatenation and identity axioms have already
	been shown in previous steps.
	For the involution, if $D_\pi$ is an element of
	$\Hom_{\mathcal{C}_{\mathcal{P}}(n)}(w_k,w_l)$ and
	is the image of a $(w_k, w_l)$--partition diagram
	$d_\pi$ under $\Theta$,
	then
	\begin{equation}
		(D_\pi)^*
		=
		\left(
		\sum_{I \in [n]^l, J \in [n]^k}
		\delta_{\pi, (I,J)}
		E_{I,J}
		\right)^*
		= 
		\sum_{J \in [n]^k, I \in [n]^l}
		\delta_{\pi^*, (J,I)}
		E_{J,I}
		=
		D_{\pi^*}
	\end{equation}
	As $D_{\pi^*}$ is the image of the $(w_l, w_k)$--partition diagram
	$d_{\pi^*}$ under $\Theta$, this implies
	that $(D_\pi)^*$ is in $\Hom_{\mathcal{C}_{\mathcal{P}}(n)}(w_l,w_k)$,
	as required.

	Finally, we see that the map 
	$R : 1 \mapsto \sum_{i=1}^{n} e_i \otimes e_i$
	is in both 
	$\Hom_{\mathcal{C}_{\mathcal{P}}(n)}(\varnothing, \circ\bullet)$ 
	and
	$\Hom_{\mathcal{C}_{\mathcal{P}}(n)}(\varnothing, \bullet\circ)$
	since it is the image under $d_\pi \mapsto D_\pi$
	of the top-row pair partition diagram corresponding to the set
	partition $\{1,2\}$ of $\{1,2\}$ superimposed, on the one hand,
	with the word $w_2 \coloneqq \circ \bullet$,
	and on the other hand,
	with the word $w_2 \coloneqq \bullet\circ$.

	\textbf{Step 6:} We need to show that $d_\pi \mapsto D_\pi$ defines a 
	strict $\mathbb{C}$--linear monoidal functor
	$\Theta : \mathcal{K}(n) \rightarrow \mathcal{C}_{\mathcal{K}}(n)$
	such that $\mathcal{C}_{\mathcal{K}}(n)$ is a 
	two-coloured representation category.

	This is now immediate by Definition \ref{twocolcatpart}, 
	the definition of a two-coloured category of partitions, 
	and since the functor on $\mathcal{P}(n)$ 
	is strict $\mathbb{C}$--linear monoidal
	whose image is a two-coloured representation category.
	In particular, the morphism spaces for $\mathcal{K}(n)$ are closed under
	composition, tensor product and involution, and so these are inherited
	by their image under $\Theta$.
	Also, the map $R : 1 \mapsto \sum_{i=1}^{n} e_i \otimes e_i$ is in 
	$\Hom_{\mathcal{C}_{\mathcal{K}}(n)}(\varnothing, \circ\bullet)$ 
	and
	$\Hom_{\mathcal{C}_{\mathcal{K}}(n)}(\varnothing, \bullet\circ)$
	since the top row pair partition diagram corresponding to the set partition
	$\{1,2\}$ of $\{1,2\}$ superimposed either with $\circ\bullet$ or with
	$\bullet\circ$ is in both
	$\Hom_{\mathcal{K}(n)}(\varnothing, \circ\bullet)$ 
	and
	$\Hom_{\mathcal{K}(n)}(\varnothing, \bullet\circ)$.
\end{proof}


\section{Additional Examples} \label{addexamples}

We provide additional examples and content for the examples that were given in
Section \ref{CMQGExamples}.
We begin with one-coloured partition categories, which have been fully characterised
by \citet{raum2016}.

\subsection{One-Coloured Partition Categories}

One-coloured partition categories can be divided into four cases: 
group, non-crossing, half-liberated, and the rest. 
The two most important cases are the group and non-crossing ones: we refer the reader 
either to \citet{banica2010} or to \citet{gromada2020}
for results relating to the two other cases.
The corresponding compact matrix quantum groups $G(n)$ are known in the literature as 
\textbf{easy orthogonal compact matrix quantum groups}, 
or sometimes \textbf{Banica--Speicher quantum groups}, 
because $S_n \subseteq G(n) \subseteq O(n)^{+}$. 

\textbf{Group}: 
\citet{banica}
provided a full characterisation: 
there are only six compact matrix groups
such that $S_n \subseteq G(n) \subseteq O(n)^{+}$. 
Consequently, by Corollary \ref{cmgfromcmqgnn},
we obtain characterisations of six compact matrix group equivariant neural networks, of which two -- the symmetric group $S_n$
\citep{ravanbakhsh17a, maron2018, pearcecrump, godfrey}
and the orthogonal group $O(n)$
\citep{pearcecrumpB}
-- 
were known previously to the machine learning community.
Said differently, we have found characterisations of the weight matrices 
that appear in four new compact matrix group equivariant neural networks,
two of which we discussed in Section \ref{CMQGExamples}.
We state the result of \citet{banica} below.

\begin{theorem}[\citet{banica} {[Theorem 2.8]}]
	\label{sixeasyorthog}
There are only six easy orthogonal compact matrix groups 
	such that $S_n \subseteq G(n) \subseteq O(n)^{+}$
	-- 
the groups and the spanning sets that determine the weight matrices of 
Theorem \ref{easycstar}
are given below. In what follows, when we refer to the image of a $(k,l)$--partition diagram, we mean the image under the map $d_\pi \mapsto D_\pi$ given in Definition
\ref{partmatfunctor}.

	\begin{itemize}
\item The symmetric group $S_n$: the image of all $(k,l)$--partition diagrams.

\item The orthogonal group $O(n)$: the image of all $(k,l)$--partition diagrams whose blocks come in pairs.

	\item The hyperoctahedral group $H_n$ (the symmetry group of the hypercube): the image of all $(k,l)$--partition diagrams having blocks of even size.

	\item The bistochastic group $B_n$ (the group of orthogonal matrices having sum $1$ in each row and in each column): the image of all $(k,l)$--partition diagrams having blocks of size one or two.

	\item The modified symmetric group $S_n' \coloneqq \mathbb{Z}_2 \times S_n$: the image of all $(k,l)$--partition diagrams where the number of blocks of odd size is even.

	\item The modified bistochastic group $B_n' \coloneqq \mathbb{Z}_2 \times B_n$: the image of all $(k,l)$--partition diagrams having an even number of blocks of size one and any number of blocks of size two.
	\end{itemize}
\end{theorem}

\begin{remark}
	We know from \citet{godfrey} that
	we can improve upon the spanning set for
	the symmetric group $S_n$ given in Theorem \ref{sixeasyorthog}
	by removing all $(k,l)$--partition diagrams 
	that have more than $n$ blocks 
	from the set of all $(k,l)$--partition diagrams
	before taking their image.
	The resulting set of matrices is the diagram basis for $S_n$.
\end{remark}

\textbf{Non-Crossing}: 
In order to state the classifications for the non-crossing case,
we first need the following definition.

\begin{definition}
	A set partition diagram $d_\pi$ corresponding to a set partition $\pi$ of $[l+k]$ is said to be \textbf{crossing}
	if there exist four integers
	$1 \leq x_1 < x_2 < x_3 < x_4 \leq l+k$ 
	satisfying:
	\begin{enumerate}
		\item $x_1$ and $x_3$ are in the same block 
		\item $x_2$ and $x_4$ are in the same block, and
		\item $x_1$ and $x_2$ are not in the same block.
	\end{enumerate}
	Otherwise, $d_\pi$ is said to be \textbf{non-crossing}.
\end{definition}

The six easy orthogonal compact matrix groups that were found by
\citet{banica}
were \textbf{liberated} by the same authors
to create six 
compact matrix quantum groups -- which are true quantum groups -- 
that fall into the non-crossing case. 
Liberation means that crossings are not allowed in any of the set partition diagrams.
\citet{weber2013}
found a seventh using this method.
We call these compact matrix quantum groups \textbf{free} since their categories of partitions $\mathcal{K}(n)$
only contain non-crossing set partition diagrams.
We have the following result.


\begin{figure*}[tb]
	\begin{tcolorbox}[colback=white!02, colframe=black]
	\begin{center}
		\scalebox{0.6}{\input{phi2,4.tikz}}
	\end{center}
	\end{tcolorbox}
	\caption{
		The fifteen $(2,2)$--partition diagrams.
	}
	\label{22diagrams}
\end{figure*}


\begin{theorem} \label{seveneasyorthog}
There are only seven free easy orthogonal compact matrix quantum groups. 
We list the quantum groups and the spanning sets that determine the weight matrices of 
Theorem \ref{easycstar} below.

The first six are 
	the symmetric quantum group $S_n^{+}$, 
	the orthogonal quantum group $O(n)^{+}$, 
	the hyperoctahedral quantum group $H_n^{+}$,
	the bistochastic quantum group $B_n^{+}$,
	the modified symmetric quantum group $S_n'^{+}$, and
	the modified bistochastic quantum group $B_n'^{+}$,
	whose spanning sets are 
	determined by the same sets of set partition diagrams
	that were given in Theorem \ref{sixeasyorthog}
	for their ``sister" groups,
	but with all crossing set partition diagrams removed.
		
The seventh is
	the freely modified bistochastic quantum group $B_n^{\#{+}}$.
	To obtain its corresponding spanning set, 
	we consider the same set of set partition diagrams as for $B_n'^{+}$,
	but now temporarily label each vertex with a colour from
	$\{\circ, \bullet\}$ in an alternating fashion, starting with $\bullet$ in the top right corner and moving counterclockwise, ultimately finishing in the bottom right corner.
	We only retain those set partition diagrams where the blocks of size two pair one $\circ$ with one $\bullet$, and take its image under $d_\pi \mapsto D_\pi$, removing the colours in the process.
%
	%

	%
	%

\end{theorem}

In fact, we obtain a stronger classification for the non-crossing case 
due to the following theorem.

\begin{theorem}[\citet{banica} {[Theorem 3.8]}]
	The spanning sets for the seven compact matrix quantum groups 
	in the non-crossing case are bases when $n \geq 4$. 
\end{theorem}

\begin{remark} \label{partweightconstruction}
By Theorem \ref{easycstar},
note that each easy compact matrix quantum group 
$(G(n), u)$ that appears in 
Theorems \ref{sixeasyorthog} and \ref{seveneasyorthog}
is \textit{defined by}
the two-coloured category of partitions
$\mathcal{K}(n)$ that is 
associated with it.
	That is, \citet{banica} actually
	showed that there are only six one-coloured category of partitions
	that contain the swap partition diagram, that is, the $(2,2)$--partition diagram
	\begin{equation}
		\begin{aligned}
			\scalebox{0.5}{\begin{tikzpicture}
	\begin{pgfonlayer}{nodelayer}
		\node [style=none] (94) at (0, 1.25) {\Large$1$};
		\node [style=none] (95) at (4, 1.25) {\Large$2$};
		\node [style=none] (96) at (0, -5.25) {\Large$3$};
		\node [style=none] (97) at (4, -5.25) {\Large$4$};
		\node [style=none] (104) at (-0.5, 0.5) {};
		\node [style=none] (105) at (-0.5, -4.5) {};
		\node [style=none] (106) at (4.5, -4.5) {};
		\node [style=none] (107) at (4.5, 0.5) {};
		\node [style=node] (108) at (0, 0) {};
		\node [style=node] (109) at (4, 0) {};
		\node [style=node] (111) at (4, -4) {};
		\node [style=node] (112) at (0, -4) {};
	\end{pgfonlayer}
	\begin{pgfonlayer}{edgelayer}
		\draw [style={dash_2}] (105.center) to (106.center);
		\draw [style={dash_2}] (106.center) to (107.center);
		\draw [style={dash_2}] (107.center) to (104.center);
		\draw [style={dash_2}] (104.center) to (105.center);
		\draw [style=thick] (108) to (111);
		\draw [style=thick] (112) to (109);
	\end{pgfonlayer}
\end{tikzpicture}}
		\end{aligned}
	\end{equation}
	and together with \citet{weber2013} they showed that there are only
	seven one-coloured category of partitions that are non-crossing.
	With these restrictions on the possible one-coloured category of partitions
	in place
	they could immediately apply Woronowicz--Tannaka--Krein duality to construct the 
	easy compact matrix quantum groups from each of these 
	one-coloured category of partitions.
	In fact, we have, again by Theorem \ref{easycstar}, that
	each compact matrix quantum group is precisely the universal $C^*$-algebra
\begin{equation}
	C(G(n)) = C^{*}(u_{i,j}, 1 \leq i,j \leq n \mid
	D_\pi{u^{\otimes w_k}} = {u^{\otimes w_l}}D_\pi 
	\text{ for all } d_\pi \in \mathcal{K}(n)(w_k,w_l))
\end{equation}
\end{remark}


\begin{example}
	We calculate the weight matrices that appear in the compact matrix 
	quantum group equivariant neural networks for all of the easy 
	orthogonal compact matrix quantum groups 
	that appear in the group and non-crossing case
	when $n = k = l = 2$.
Recall from Theorems 
\ref{sixeasyorthog}
and
\ref{seveneasyorthog}
that the weight matrices are formed as a weighted linear combination of the image under $d_\pi \mapsto D_\pi$ of some subset of all of the possible $(2,2)$--partition diagrams.

There are $B(4) = 15$ $(2,2)$--partition diagrams in total, 
where $B(4)$ is the fourth Bell number. 
They are shown in Figure \ref{22diagrams}.
We number the diagrams from $i = 1$ to $15$, going from left to right 
and then top to bottom, assigning a weight $w_i$ to each diagram.


\begin{figure*}[tb]
	\begin{tcolorbox}[colback=blue!10, colframe=blue!30, coltitle=black, 
		title={\bfseries Procedure 2: How to Calculate
		the $(I,J)$-entry of each Equivariant Spanning Set Matrix $D_\pi$
		from 
		$((\mathbb{C}^{n})^{\otimes k}, u^{\otimes w_{k}})
		\rightarrow 
		((\mathbb{C}^{n})^{\otimes l}, u^{\otimes w_{l}})$
		for an Easy Compact Matrix Quantum Group $(G(n), u)$.},
		fonttitle=\bfseries]
		We assume that $d_\pi$ is a two-coloured $(w_k, w_l)$--partition diagram 
		in the two-coloured category of partitions $\mathcal{K}(n)$ that 
		uniquely determines $(G(n), u)$.
		We perform the following steps:
	\begin{enumerate}
	\item Place the indices $I$ on the top row of $d_\pi$ and the indices $J$ on the bottom row of $d_\pi$.
	\item If all of the vertices in each block in $d_\pi$ have been overlaid with the same number, then the $(I, J)$ entry of $D_\pi$ is $1$, otherwise it is $0$.
	\end{enumerate}
	\end{tcolorbox}
  	\label{qgrouppowermethod}
\end{figure*}



\begin{figure}[p]
\begin{center}
\begin{tblr}{
  colspec = {X[c]X[c]X[c]X[c]},
  stretch = 0,
  rowsep = 5pt,
  hlines = {1pt},
  vlines = {1pt},
}
	{$d_\pi$} & 
	{$D_\pi$ } &
	{$d_\pi$} & 
	{$D_\pi$ } \\
	\scalebox{0.6}{\begin{tikzpicture}
	\begin{pgfonlayer}{nodelayer}
		\node [style=none] (0) at (0.5, -2.75) {\Large$3$};
		\node [style=node] (1) at (0.5, 1.5) {};
		\node [style=node] (2) at (0.5, -1.5) {};
		\node [style=node] (3) at (3.5, 1.5) {};
		\node [style=node] (4) at (3.5, -1.5) {};
		\node [style=none] (5) at (3.5, -2.75) {\Large$4$};
		\node [style=none] (6) at (0.5, 2.75) {\Large$1$};
		\node [style=none] (7) at (3.5, 2.75) {\Large$2$};
		\node [style=none] (8) at (0, 2) {};
		\node [style=none] (9) at (0, -2) {};
		\node [style=none] (10) at (4, -2) {};
		\node [style=none] (11) at (4, 2) {};
	\end{pgfonlayer}
	\begin{pgfonlayer}{edgelayer}
		\draw [style={dash_2}] (8.center) to (9.center);
		\draw [style={dash_2}] (9.center) to (10.center);
		\draw [style={dash_2}] (10.center) to (11.center);
		\draw [style={dash_2}] (11.center) to (8.center);
		\draw [style=thick] (2) to (1);
		\draw [style=thick] (2) to (4);
		\draw [style=thick] (4) to (3);
		\draw [style=thick] (3) to (1);
		\draw [style=thick] (2) to (3);
		\draw [style=thick] (4) to (1);
	\end{pgfonlayer}
\end{tikzpicture}} & 
	\scalebox{0.65}{
	$
	\NiceMatrixOptions{code-for-first-row = \scriptstyle \color{blue},
                   	   code-for-first-col = \scriptstyle \color{blue}
	}
	\begin{bNiceArray}{*{2}{c}|*{2}{c}}[first-row,first-col]
				& 1,1 	& 1,2	& 2,1 	& 2,2 	\\
		1,1		& 1	& 0	& 0	& 0   	\\
		1,2		& 0	& 0	& 0	& 0	\\
		\cline{1-4}
		2,1		& 0	& 0	& 0	& 0   	\\
		2,2		& 0	& 0	& 0	& 1	
	\end{bNiceArray}
	$} &
	\scalebox{0.6}{\begin{tikzpicture}
	\begin{pgfonlayer}{nodelayer}
		\node [style=none] (0) at (0.5, -2.75) {\Large$3$};
		\node [style=node] (1) at (0.5, 1.5) {};
		\node [style=node] (2) at (0.5, -1.5) {};
		\node [style=node] (3) at (3.5, 1.5) {};
		\node [style=node] (4) at (3.5, -1.5) {};
		\node [style=none] (5) at (3.5, -2.75) {\Large$4$};
		\node [style=none] (6) at (0.5, 2.75) {\Large$1$};
		\node [style=none] (7) at (3.5, 2.75) {\Large$2$};
		\node [style=none] (8) at (0, 2) {};
		\node [style=none] (9) at (0, -2) {};
		\node [style=none] (10) at (4, -2) {};
		\node [style=none] (11) at (4, 2) {};
	\end{pgfonlayer}
	\begin{pgfonlayer}{edgelayer}
		\draw [style={dash_2}] (10.center) to (11.center);
		\draw [style={dash_2}] (11.center) to (8.center);
		\draw [style={dash_2}] (8.center) to (9.center);
		\draw [style={dash_2}] (9.center) to (10.center);
		\draw [style=thick] (1) to (3);
	\end{pgfonlayer}
\end{tikzpicture}} & 
	\scalebox{0.65}{
	$
	\NiceMatrixOptions{code-for-first-row = \scriptstyle \color{blue},
                   	   code-for-first-col = \scriptstyle \color{blue}
	}
	\begin{bNiceArray}{*{2}{c}|*{2}{c}}[first-row,first-col]
				& 1,1 	& 1,2	& 2,1 	& 2,2 	\\
		1,1		& 1	& 1	& 1	& 1   	\\
		1,2		& 0	& 0	& 0	& 0	\\
		\cline{1-4}
		2,1		& 0	& 0	& 0	& 0   	\\
		2,2		& 1	& 1	& 1	& 1	
	\end{bNiceArray}
	$} \\
	\scalebox{0.6}{\begin{tikzpicture}
	\begin{pgfonlayer}{nodelayer}
		\node [style=none] (0) at (0.5, -2.75) {\Large$3$};
		\node [style=node] (1) at (0.5, 1.5) {};
		\node [style=node] (2) at (0.5, -1.5) {};
		\node [style=node] (3) at (3.5, 1.5) {};
		\node [style=node] (4) at (3.5, -1.5) {};
		\node [style=none] (5) at (3.5, -2.75) {\Large$4$};
		\node [style=none] (6) at (0.5, 2.75) {\Large$1$};
		\node [style=none] (7) at (3.5, 2.75) {\Large$2$};
		\node [style=none] (8) at (0, 2) {};
		\node [style=none] (9) at (0, -2) {};
		\node [style=none] (10) at (4, -2) {};
		\node [style=none] (11) at (4, 2) {};
	\end{pgfonlayer}
	\begin{pgfonlayer}{edgelayer}
		\draw [style={dash_2}] (8.center) to (9.center);
		\draw [style={dash_2}] (9.center) to (10.center);
		\draw [style={dash_2}] (10.center) to (11.center);
		\draw [style={dash_2}] (11.center) to (8.center);
		\draw [style=thick] (4) to (3);
		\draw [style=thick] (3) to (1);
		\draw [style=thick] (4) to (1);
	\end{pgfonlayer}
\end{tikzpicture}} & 
	\scalebox{0.65}{
	$
	\NiceMatrixOptions{code-for-first-row = \scriptstyle \color{blue},
                   	   code-for-first-col = \scriptstyle \color{blue}
	}
	\begin{bNiceArray}{*{2}{c}|*{2}{c}}[first-row,first-col]
				& 1,1 	& 1,2	& 2,1 	& 2,2 	\\
		1,1		& 1	& 0	& 1	& 0   	\\
		1,2		& 0	& 0	& 0	& 0	\\
		\cline{1-4}
		2,1		& 0	& 0	& 0	& 0   	\\
		2,2		& 0	& 1	& 0	& 1	
	\end{bNiceArray}
	$} &
	\scalebox{0.6}{\begin{tikzpicture}
	\begin{pgfonlayer}{nodelayer}
		\node [style=none] (0) at (0.5, -2.75) {\Large$3$};
		\node [style=node] (1) at (0.5, 1.5) {};
		\node [style=node] (2) at (0.5, -1.5) {};
		\node [style=node] (3) at (3.5, 1.5) {};
		\node [style=node] (4) at (3.5, -1.5) {};
		\node [style=none] (5) at (3.5, -2.75) {\Large$4$};
		\node [style=none] (6) at (0.5, 2.75) {\Large$1$};
		\node [style=none] (7) at (3.5, 2.75) {\Large$2$};
		\node [style=none] (8) at (0, 2) {};
		\node [style=none] (9) at (0, -2) {};
		\node [style=none] (10) at (4, -2) {};
		\node [style=none] (11) at (4, 2) {};
	\end{pgfonlayer}
	\begin{pgfonlayer}{edgelayer}
		\draw [style={dash_2}] (10.center) to (11.center);
		\draw [style={dash_2}] (11.center) to (8.center);
		\draw [style={dash_2}] (8.center) to (9.center);
		\draw [style={dash_2}] (9.center) to (10.center);
		\draw [style=thick] (3) to (4);
	\end{pgfonlayer}
\end{tikzpicture}} & 
	\scalebox{0.65}{
	$
	\NiceMatrixOptions{code-for-first-row = \scriptstyle \color{blue},
                   	   code-for-first-col = \scriptstyle \color{blue}
	}
	\begin{bNiceArray}{*{2}{c}|*{2}{c}}[first-row,first-col]
				& 1,1 	& 1,2	& 2,1 	& 2,2 	\\
		1,1		& 1	& 0	& 1	& 0   	\\
		1,2		& 0	& 1	& 0	& 1	\\
		\cline{1-4}
		2,1		& 1	& 0	& 1	& 0   	\\
		2,2		& 0	& 1	& 0	& 1	
	\end{bNiceArray}
	$} \\
	\scalebox{0.6}{\begin{tikzpicture}
	\begin{pgfonlayer}{nodelayer}
		\node [style=none] (0) at (0.5, -2.75) {\Large$3$};
		\node [style=node] (1) at (0.5, 1.5) {};
		\node [style=node] (2) at (0.5, -1.5) {};
		\node [style=node] (3) at (3.5, 1.5) {};
		\node [style=node] (4) at (3.5, -1.5) {};
		\node [style=none] (5) at (3.5, -2.75) {\Large$4$};
		\node [style=none] (6) at (0.5, 2.75) {\Large$1$};
		\node [style=none] (7) at (3.5, 2.75) {\Large$2$};
		\node [style=none] (8) at (0, 2) {};
		\node [style=none] (9) at (0, -2) {};
		\node [style=none] (10) at (4, -2) {};
		\node [style=none] (11) at (4, 2) {};
	\end{pgfonlayer}
	\begin{pgfonlayer}{edgelayer}
		\draw [style={dash_2}] (8.center) to (9.center);
		\draw [style={dash_2}] (9.center) to (10.center);
		\draw [style={dash_2}] (10.center) to (11.center);
		\draw [style={dash_2}] (11.center) to (8.center);
		\draw [style=thick] (2) to (3);
		\draw [style=thick] (3) to (1);
		\draw [style=thick] (2) to (1);
	\end{pgfonlayer}
\end{tikzpicture}} & 
	\scalebox{0.65}{
	$
	\NiceMatrixOptions{code-for-first-row = \scriptstyle \color{blue},
                   	   code-for-first-col = \scriptstyle \color{blue}
	}
	\begin{bNiceArray}{*{2}{c}|*{2}{c}}[first-row,first-col]
				& 1,1 	& 1,2	& 2,1 	& 2,2 	\\
		1,1		& 1	& 1	& 0	& 0   	\\
		1,2		& 0	& 0	& 0	& 0	\\
		\cline{1-4}
		2,1		& 0	& 0	& 0	& 0   	\\
		2,2		& 0	& 0	& 1	& 1	
	\end{bNiceArray}
	$} &
	\scalebox{0.6}{\begin{tikzpicture}
	\begin{pgfonlayer}{nodelayer}
		\node [style=none] (0) at (0.5, -2.75) {\Large$3$};
		\node [style=node] (1) at (0.5, 1.5) {};
		\node [style=node] (2) at (0.5, -1.5) {};
		\node [style=node] (3) at (3.5, 1.5) {};
		\node [style=node] (4) at (3.5, -1.5) {};
		\node [style=none] (5) at (3.5, -2.75) {\Large$4$};
		\node [style=none] (6) at (0.5, 2.75) {\Large$1$};
		\node [style=none] (7) at (3.5, 2.75) {\Large$2$};
		\node [style=none] (8) at (0, 2) {};
		\node [style=none] (9) at (0, -2) {};
		\node [style=none] (10) at (4, -2) {};
		\node [style=none] (11) at (4, 2) {};
	\end{pgfonlayer}
	\begin{pgfonlayer}{edgelayer}
		\draw [style={dash_2}] (9.center) to (10.center);
		\draw [style={dash_2}] (10.center) to (11.center);
		\draw [style={dash_2}] (11.center) to (8.center);
		\draw [style={dash_2}] (8.center) to (9.center);
		\draw [style=thick] (1) to (4);
	\end{pgfonlayer}
\end{tikzpicture}} & 
	\scalebox{0.65}{
	$
	\NiceMatrixOptions{code-for-first-row = \scriptstyle \color{blue},
                   	   code-for-first-col = \scriptstyle \color{blue}
	}
	\begin{bNiceArray}{*{2}{c}|*{2}{c}}[first-row,first-col]
				& 1,1 	& 1,2	& 2,1 	& 2,2 	\\
		1,1		& 1	& 0	& 1	& 0   	\\
		1,2		& 1	& 0	& 1	& 0	\\
		\cline{1-4}
		2,1		& 0	& 1	& 0	& 1   	\\
		2,2		& 0	& 1	& 0	& 1	
	\end{bNiceArray}
	$} \\
	\scalebox{0.6}{\begin{tikzpicture}
	\begin{pgfonlayer}{nodelayer}
		\node [style=none] (0) at (0.5, -2.75) {\Large$3$};
		\node [style=node] (1) at (0.5, 1.5) {};
		\node [style=node] (2) at (0.5, -1.5) {};
		\node [style=node] (3) at (3.5, 1.5) {};
		\node [style=node] (4) at (3.5, -1.5) {};
		\node [style=none] (5) at (3.5, -2.75) {\Large$4$};
		\node [style=none] (6) at (0.5, 2.75) {\Large$1$};
		\node [style=none] (7) at (3.5, 2.75) {\Large$2$};
		\node [style=none] (8) at (0, 2) {};
		\node [style=none] (9) at (0, -2) {};
		\node [style=none] (10) at (4, -2) {};
		\node [style=none] (11) at (4, 2) {};
	\end{pgfonlayer}
	\begin{pgfonlayer}{edgelayer}
		\draw [style={dash_2}] (8.center) to (9.center);
		\draw [style={dash_2}] (9.center) to (10.center);
		\draw [style={dash_2}] (10.center) to (11.center);
		\draw [style={dash_2}] (11.center) to (8.center);
		\draw [style=thick] (2) to (3);
		\draw [style=thick] (3) to (4);
		\draw [style=thick] (4) to (2);
	\end{pgfonlayer}
\end{tikzpicture}} & 
	\scalebox{0.65}{
	$
	\NiceMatrixOptions{code-for-first-row = \scriptstyle \color{blue},
                   	   code-for-first-col = \scriptstyle \color{blue}
	}
	\begin{bNiceArray}{*{2}{c}|*{2}{c}}[first-row,first-col]
				& 1,1 	& 1,2	& 2,1 	& 2,2 	\\
		1,1		& 1	& 0	& 0	& 0   	\\
		1,2		& 0	& 0	& 0	& 1	\\
		\cline{1-4}
		2,1		& 1	& 0	& 0	& 0   	\\
		2,2		& 0	& 0	& 0	& 1	
	\end{bNiceArray}
	$} &
	\scalebox{0.6}{\begin{tikzpicture}
	\begin{pgfonlayer}{nodelayer}
		\node [style=none] (0) at (0.5, -2.75) {\Large$3$};
		\node [style=node] (1) at (0.5, 1.5) {};
		\node [style=node] (2) at (0.5, -1.5) {};
		\node [style=node] (3) at (3.5, 1.5) {};
		\node [style=node] (4) at (3.5, -1.5) {};
		\node [style=none] (5) at (3.5, -2.75) {\Large$4$};
		\node [style=none] (6) at (0.5, 2.75) {\Large$1$};
		\node [style=none] (7) at (3.5, 2.75) {\Large$2$};
		\node [style=none] (8) at (0, 2) {};
		\node [style=none] (9) at (0, -2) {};
		\node [style=none] (10) at (4, -2) {};
		\node [style=none] (11) at (4, 2) {};
	\end{pgfonlayer}
	\begin{pgfonlayer}{edgelayer}
		\draw [style={dash_2}] (9.center) to (10.center);
		\draw [style={dash_2}] (10.center) to (11.center);
		\draw [style={dash_2}] (11.center) to (8.center);
		\draw [style={dash_2}] (8.center) to (9.center);
		\draw [style=thick] (2) to (4);
	\end{pgfonlayer}
\end{tikzpicture}} & 
	\scalebox{0.65}{
	$
	\NiceMatrixOptions{code-for-first-row = \scriptstyle \color{blue},
                   	   code-for-first-col = \scriptstyle \color{blue}
	}
	\begin{bNiceArray}{*{2}{c}|*{2}{c}}[first-row,first-col]
				& 1,1 	& 1,2	& 2,1 	& 2,2 	\\
		1,1		& 1	& 0	& 0	& 1   	\\
		1,2		& 1	& 0	& 0	& 1	\\
		\cline{1-4}
		2,1		& 1	& 0	& 0	& 1   	\\
		2,2		& 1	& 0	& 0	& 1	
	\end{bNiceArray}
	$} \\
	\scalebox{0.6}{\begin{tikzpicture}
	\begin{pgfonlayer}{nodelayer}
		\node [style=none] (0) at (0.5, -2.75) {\Large$3$};
		\node [style=node] (1) at (0.5, 1.5) {};
		\node [style=node] (2) at (0.5, -1.5) {};
		\node [style=node] (3) at (3.5, 1.5) {};
		\node [style=node] (4) at (3.5, -1.5) {};
		\node [style=none] (5) at (3.5, -2.75) {\Large$4$};
		\node [style=none] (6) at (0.5, 2.75) {\Large$1$};
		\node [style=none] (7) at (3.5, 2.75) {\Large$2$};
		\node [style=none] (8) at (0, 2) {};
		\node [style=none] (9) at (0, -2) {};
		\node [style=none] (10) at (4, -2) {};
		\node [style=none] (11) at (4, 2) {};
	\end{pgfonlayer}
	\begin{pgfonlayer}{edgelayer}
		\draw [style={dash_2}] (8.center) to (9.center);
		\draw [style={dash_2}] (9.center) to (10.center);
		\draw [style={dash_2}] (10.center) to (11.center);
		\draw [style={dash_2}] (11.center) to (8.center);
		\draw [style=thick] (1) to (2);
		\draw [style=thick] (2) to (4);
		\draw [style=thick] (4) to (1);
	\end{pgfonlayer}
\end{tikzpicture}} & 
	\scalebox{0.65}{
	$
	\NiceMatrixOptions{code-for-first-row = \scriptstyle \color{blue},
                   	   code-for-first-col = \scriptstyle \color{blue}
	}
	\begin{bNiceArray}{*{2}{c}|*{2}{c}}[first-row,first-col]
				& 1,1 	& 1,2	& 2,1 	& 2,2 	\\
		1,1		& 1	& 0	& 0	& 0   	\\
		1,2		& 1	& 0	& 0	& 0	\\
		\cline{1-4}
		2,1		& 0	& 0	& 0	& 1   	\\
		2,2		& 0	& 0	& 0	& 1	
	\end{bNiceArray}
	$} &
	\scalebox{0.6}{\begin{tikzpicture}
	\begin{pgfonlayer}{nodelayer}
		\node [style=none] (0) at (0.5, -2.75) {\Large$3$};
		\node [style=node] (1) at (0.5, 1.5) {};
		\node [style=node] (2) at (0.5, -1.5) {};
		\node [style=node] (3) at (3.5, 1.5) {};
		\node [style=node] (4) at (3.5, -1.5) {};
		\node [style=none] (5) at (3.5, -2.75) {\Large$4$};
		\node [style=none] (6) at (0.5, 2.75) {\Large$1$};
		\node [style=none] (7) at (3.5, 2.75) {\Large$2$};
		\node [style=none] (8) at (0, 2) {};
		\node [style=none] (9) at (0, -2) {};
		\node [style=none] (10) at (4, -2) {};
		\node [style=none] (11) at (4, 2) {};
	\end{pgfonlayer}
	\begin{pgfonlayer}{edgelayer}
		\draw [style={dash_2}] (9.center) to (10.center);
		\draw [style={dash_2}] (10.center) to (11.center);
		\draw [style={dash_2}] (11.center) to (8.center);
		\draw [style={dash_2}] (8.center) to (9.center);
		\draw [style=thick] (2) to (1);
	\end{pgfonlayer}
\end{tikzpicture}} & 
	\scalebox{0.65}{
	$
	\NiceMatrixOptions{code-for-first-row = \scriptstyle \color{blue},
                   	   code-for-first-col = \scriptstyle \color{blue}
	}
	\begin{bNiceArray}{*{2}{c}|*{2}{c}}[first-row,first-col]
				& 1,1 	& 1,2	& 2,1 	& 2,2 	\\
		1,1		& 1	& 1	& 0	& 0   	\\
		1,2		& 1	& 1	& 0	& 0	\\
		\cline{1-4}
		2,1		& 0	& 0	& 1	& 1   	\\
		2,2		& 0	& 0	& 1	& 1	
	\end{bNiceArray}
	$} \\
	\scalebox{0.6}{\begin{tikzpicture}
	\begin{pgfonlayer}{nodelayer}
		\node [style=none] (0) at (0.5, -2.75) {\Large$3$};
		\node [style=node] (1) at (0.5, 1.5) {};
		\node [style=node] (2) at (0.5, -1.5) {};
		\node [style=node] (3) at (3.5, 1.5) {};
		\node [style=node] (4) at (3.5, -1.5) {};
		\node [style=none] (5) at (3.5, -2.75) {\Large$4$};
		\node [style=none] (6) at (0.5, 2.75) {\Large$1$};
		\node [style=none] (7) at (3.5, 2.75) {\Large$2$};
		\node [style=none] (8) at (0, 2) {};
		\node [style=none] (9) at (0, -2) {};
		\node [style=none] (10) at (4, -2) {};
		\node [style=none] (11) at (4, 2) {};
	\end{pgfonlayer}
	\begin{pgfonlayer}{edgelayer}
		\draw [style={dash_2}] (8.center) to (9.center);
		\draw [style={dash_2}] (9.center) to (10.center);
		\draw [style={dash_2}] (10.center) to (11.center);
		\draw [style={dash_2}] (11.center) to (8.center);
		\draw [style=thick] (1) to (3);
		\draw [style=thick] (2) to (4);
	\end{pgfonlayer}
\end{tikzpicture}} & 
	\scalebox{0.65}{
	$
	\NiceMatrixOptions{code-for-first-row = \scriptstyle \color{blue},
                   	   code-for-first-col = \scriptstyle \color{blue}
	}
	\begin{bNiceArray}{*{2}{c}|*{2}{c}}[first-row,first-col]
				& 1,1 	& 1,2	& 2,1 	& 2,2 	\\
		1,1		& 1	& 0	& 0	& 1   	\\
		1,2		& 0	& 0	& 0	& 0	\\
		\cline{1-4}
		2,1		& 0	& 0	& 0	& 0   	\\
		2,2		& 1	& 0	& 0	& 1	
	\end{bNiceArray}
	$} &
	\scalebox{0.6}{\begin{tikzpicture}
	\begin{pgfonlayer}{nodelayer}
		\node [style=none] (0) at (0.5, -2.75) {\Large$3$};
		\node [style=node] (1) at (0.5, 1.5) {};
		\node [style=node] (2) at (0.5, -1.5) {};
		\node [style=node] (3) at (3.5, 1.5) {};
		\node [style=node] (4) at (3.5, -1.5) {};
		\node [style=none] (5) at (3.5, -2.75) {\Large$4$};
		\node [style=none] (6) at (0.5, 2.75) {\Large$1$};
		\node [style=none] (7) at (3.5, 2.75) {\Large$2$};
		\node [style=none] (8) at (0, 2) {};
		\node [style=none] (9) at (0, -2) {};
		\node [style=none] (10) at (4, -2) {};
		\node [style=none] (11) at (4, 2) {};
	\end{pgfonlayer}
	\begin{pgfonlayer}{edgelayer}
		\draw [style={dash_2}] (9.center) to (10.center);
		\draw [style={dash_2}] (10.center) to (11.center);
		\draw [style={dash_2}] (11.center) to (8.center);
		\draw [style={dash_2}] (8.center) to (9.center);
		\draw [style=thick] (2) to (3);
	\end{pgfonlayer}
\end{tikzpicture}} & 
	\scalebox{0.65}{
	$
	\NiceMatrixOptions{code-for-first-row = \scriptstyle \color{blue},
                   	   code-for-first-col = \scriptstyle \color{blue}
	}
	\begin{bNiceArray}{*{2}{c}|*{2}{c}}[first-row,first-col]
				& 1,1 	& 1,2	& 2,1 	& 2,2 	\\
		1,1		& 1	& 1	& 0	& 0   	\\
		1,2		& 0	& 0	& 1	& 1	\\
		\cline{1-4}
		2,1		& 1	& 1	& 0	& 0   	\\
		2,2		& 0	& 0	& 1	& 1	
	\end{bNiceArray}
	$} \\
	\scalebox{0.6}{\begin{tikzpicture}
	\begin{pgfonlayer}{nodelayer}
		\node [style=none] (0) at (0.5, -2.75) {\Large$3$};
		\node [style=node] (1) at (0.5, 1.5) {};
		\node [style=node] (2) at (0.5, -1.5) {};
		\node [style=node] (3) at (3.5, 1.5) {};
		\node [style=node] (4) at (3.5, -1.5) {};
		\node [style=none] (5) at (3.5, -2.75) {\Large$4$};
		\node [style=none] (6) at (0.5, 2.75) {\Large$1$};
		\node [style=none] (7) at (3.5, 2.75) {\Large$2$};
		\node [style=none] (8) at (0, 2) {};
		\node [style=none] (9) at (0, -2) {};
		\node [style=none] (10) at (4, -2) {};
		\node [style=none] (11) at (4, 2) {};
	\end{pgfonlayer}
	\begin{pgfonlayer}{edgelayer}
		\draw [style={dash_2}] (8.center) to (9.center);
		\draw [style={dash_2}] (9.center) to (10.center);
		\draw [style={dash_2}] (10.center) to (11.center);
		\draw [style={dash_2}] (11.center) to (8.center);
		\draw [style=thick] (1) to (4);
		\draw [style=thick] (2) to (3);
	\end{pgfonlayer}
\end{tikzpicture}} & 
	\scalebox{0.65}{
	$
	\NiceMatrixOptions{code-for-first-row = \scriptstyle \color{blue},
                   	   code-for-first-col = \scriptstyle \color{blue}
	}
	\begin{bNiceArray}{*{2}{c}|*{2}{c}}[first-row,first-col]
				& 1,1 	& 1,2	& 2,1 	& 2,2 	\\
		1,1		& 1	& 0	& 0	& 0   	\\
		1,2		& 0	& 0	& 1	& 0	\\
		\cline{1-4}
		2,1		& 0	& 1	& 0	& 0   	\\
		2,2		& 0	& 0	& 0	& 1	
	\end{bNiceArray}
	$} &
	\scalebox{0.6}{\begin{tikzpicture}
	\begin{pgfonlayer}{nodelayer}
		\node [style=none] (0) at (0.5, -2.75) {\Large$3$};
		\node [style=node] (1) at (0.5, 1.5) {};
		\node [style=node] (2) at (0.5, -1.5) {};
		\node [style=node] (3) at (3.5, 1.5) {};
		\node [style=node] (4) at (3.5, -1.5) {};
		\node [style=none] (5) at (3.5, -2.75) {\Large$4$};
		\node [style=none] (6) at (0.5, 2.75) {\Large$1$};
		\node [style=none] (7) at (3.5, 2.75) {\Large$2$};
		\node [style=none] (8) at (0, 2) {};
		\node [style=none] (9) at (0, -2) {};
		\node [style=none] (10) at (4, -2) {};
		\node [style=none] (11) at (4, 2) {};
	\end{pgfonlayer}
	\begin{pgfonlayer}{edgelayer}
		\draw [style={dash_2}] (9.center) to (10.center);
		\draw [style={dash_2}] (10.center) to (11.center);
		\draw [style={dash_2}] (11.center) to (8.center);
		\draw [style={dash_2}] (8.center) to (9.center);
	\end{pgfonlayer}
\end{tikzpicture}} & 
	\scalebox{0.65}{
	$
	\NiceMatrixOptions{code-for-first-row = \scriptstyle \color{blue},
                   	   code-for-first-col = \scriptstyle \color{blue}
	}
	\begin{bNiceArray}{*{2}{c}|*{2}{c}}[first-row,first-col]
				& 1,1 	& 1,2	& 2,1 	& 2,2 	\\
		1,1		& 1	& 1	& 1	& 1   	\\
		1,2		& 1	& 1	& 1	& 1	\\
		\cline{1-4}
		2,1		& 1	& 1	& 1	& 1   	\\
		2,2		& 1	& 1	& 1	& 1	
	\end{bNiceArray}
	$} \\
	\scalebox{0.6}{\begin{tikzpicture}
	\begin{pgfonlayer}{nodelayer}
		\node [style=none] (0) at (0.5, -2.75) {\Large$3$};
		\node [style=node] (1) at (0.5, 1.5) {};
		\node [style=node] (2) at (0.5, -1.5) {};
		\node [style=node] (3) at (3.5, 1.5) {};
		\node [style=node] (4) at (3.5, -1.5) {};
		\node [style=none] (5) at (3.5, -2.75) {\Large$4$};
		\node [style=none] (6) at (0.5, 2.75) {\Large$1$};
		\node [style=none] (7) at (3.5, 2.75) {\Large$2$};
		\node [style=none] (8) at (0, 2) {};
		\node [style=none] (9) at (0, -2) {};
		\node [style=none] (10) at (4, -2) {};
		\node [style=none] (11) at (4, 2) {};
	\end{pgfonlayer}
	\begin{pgfonlayer}{edgelayer}
		\draw [style={dash_2}] (8.center) to (9.center);
		\draw [style={dash_2}] (9.center) to (10.center);
		\draw [style={dash_2}] (10.center) to (11.center);
		\draw [style={dash_2}] (11.center) to (8.center);
		\draw [style=thick] (1) to (2);
		\draw [style=thick] (4) to (3);
	\end{pgfonlayer}
\end{tikzpicture}} & 
	\scalebox{0.65}{
	$
	\NiceMatrixOptions{code-for-first-row = \scriptstyle \color{blue},
                   	   code-for-first-col = \scriptstyle \color{blue}
	}
	\begin{bNiceArray}{*{2}{c}|*{2}{c}}[first-row,first-col]
				& 1,1 	& 1,2	& 2,1 	& 2,2 	\\
		1,1		& 1	& 0	& 0	& 0   	\\
		1,2		& 0	& 1	& 0	& 0	\\
		\cline{1-4}
		2,1		& 0	& 0	& 1	& 0   	\\
		2,2		& 0	& 0	& 0	& 1	
	\end{bNiceArray}
	$} &
\end{tblr}
	\caption{For $n=2$, we display the images under the map $d_\pi \mapsto D_\pi$
	for each of the $(2,2)$--partition diagrams.}
	\label{easyorthogqc}
	\end{center}
\end{figure}

Figure \ref{easyorthogqc} shows the images of the fifteen $(2,2)$--partition diagrams under the map $d_\pi \mapsto D_\pi$. 
We apply the characterisations of Theorems
\ref{sixeasyorthog}
and
\ref{seveneasyorthog}
to obtain the following results.

For the group case, we have that
\begin{itemize}
	\item
		\textbf{The Symmetric Group $S_n$}:
		The weight matrix is a weighted linear combination of the images of
		all fifteen $(2,2)$--partition diagrams.
	\item 
		\textbf{The Orthogonal Group $O(n)$}:
		The weight matrix is a weighted linear combination of the images of
		diagrams six, seven and eight.
	\item
		\textbf{The Hyperoctahedral Group $H_n$}:
		The weight matrix is (coincidentally) the same as for the orthogonal 
		group in this case, since the $(2,2)$-partition diagrams with blocks 
		of even size are the same as the $(2,2)$-partition diagrams whose 
		blocks come in pairs.
	\item
		\textbf{The Bistochastic Group $B_n$}:
		The weight matrix is a weighted linear combination of the images of
		diagrams six through fifteen, inclusive.
	\item
		\textbf{The Modified Symmetric Group $S_n'$}:
		The weight matrix is (coincidentally) the same as for the symmetric 
		group, since the number of blocks of odd size in any 
		$(2,2)$--partition diagram is either $0, 2$ or $4$.
	\item
		\textbf{The Modified Bistochastic Group $B_n'$}:
		The weight matrix is (coincidentally) the same as for the bistochastic 
		group, since the number of blocks of odd size in any 
		$(2,2)$--partition diagram is either $0, 2$ or $4$.
\end{itemize}
and for the non-crossing case, we have that
\begin{itemize}
	\item
		\textbf{The Symmetric Quantum Group $S_n^{+}$}:
		The weight matrix is a weighted linear combination of the images of
		all diagrams except for seven, which is the only crossing diagram.
	\item
		\textbf{The Orthogonal Quantum Group $O(n)^{+}$}:
		The weight matrix is a weighted linear combination of the images of
		diagrams six and eight.
	\item
		\textbf{The Hyperoctahedral Quantum Group $H_n^{+}$}:
		The weight matrix is the same as for $O(n)^{+}$.
	\item
		\textbf{The Bistochastic Quantum Group $B_n^{+}$}:
		The weight matrix is a weighted linear combination of the images of
		diagrams six and eight through fifteen, inclusive.	
	\item
		\textbf{The Modified Symmetric Quantum Group $S_n'^{+}$}:
		The weight matrix is the same as for $S_n^{+}$.
	\item
		\textbf{The Modified Bistochastic Quantum Group $B_n'^{+}$}:
		The weight matrix is the same as for $B_n^{+}$.
	\item
		\textbf{The Freely Modified Bistochastic Quantum Group $B_n^{\#{+}}$}:
		We colour in the diagrams that appear in the weight matrix 
		for $B_n^{+}$, and retain only the diagrams where the blocks 
		of size two pair one $\circ$ with one $\bullet$.
		Consequently, we remove diagrams eleven and fourteen from the set, 
		and so the weight matrix is a weighted linear combination of 
		the images of diagrams six, eight, nine, ten, twelve, 
		thirteen and fifteen.
\end{itemize}
\end{example}

\subsection{Two-Coloured Partition Categories}

Two-coloured partition categories are much richer than one-coloured partition categories 
because the generators $u_{i,j}$ of the corresponding compact matrix quantum groups are 
no longer self-adjoint.
They were first investigated by \citet{freslonweber2016}.
Unfortunately, at the time of writing, a full characterisation of these categories 
is unknown. However, 
\citet{tarrago2016, tarrago2018}
classified all two-coloured partition categories in the group and non-crossing case.
In particular, there are seven series of two-coloured partition categories in the group
case and twelve series of two-coloured partition categories in the non-crossing case.
The corresponding compact matrix quantum groups $G(n)$ are known as 
\textbf{easy unitary compact matrix quantum groups}, since they satisfy 
$S_n \subseteq G(n) \subseteq U(n)^{+}$. 
We focus on the  characterisation for only the two most important series:
the unitary group $U(n)$ and the unitary quantum group $U(n)^+$,
referring the reader to \citet[Theorem 5.3]{tarrago2016}
for the characterisations of the other series. 

\citet{tarrago2016, tarrago2018} showed that the two-coloured category of partitions
for the unitary group $U(n)$ has, for fixed words $w_k$ and $w_l$, a morphism space that
is spanned by all two-coloured $(w_k, w_l)$--partition diagrams 
whose blocks come in pairs such that if two vertices of a block are in the same row, 
then they have different colours, otherwise they have the same colours.
\citet{tarrago2016, tarrago2018} also showed that the two-coloured category of partitions
for the unitary quantum group $U(n)^+$ has, for words $w_k$ and $w_l$, a morphism space that
is spanned by all \textit{non-crossing} two-coloured $(w_k, w_l)$--partition diagrams 
satisfying the same conditions as those for the unitary group.
As before, we obtain the $n^l \times n^k$ weight matrices of 
Theorem \ref{easycstar}
for these compact matrix quantum groups
by taking the images of the two-coloured partition diagrams that span the respective morphism spaces
under the map $d_\pi \mapsto D_\pi$ given in Definition \ref{partmatfunctor}.

This introduces the rather interesting point for the unitary group $U(n)$.
If we pick two words $w_k$ and $w_l$, then the $n^l \times n^k$ weight matrix 
corresponding to these words
that appears in a
compact matrix quantum group equivariant neural network
for the unitary group $U(n)$ 
is the $\mathbb{C}$-linear span of the image of all
$(w_k, w_l)$--partition diagrams that live
in the two-coloured partition category for $U(n)$.
However, if we consider compact matrix \textit{group} equivariant neural networks for $U(n)$ instead,
as the words can only be formed from white points, 
then not only must we have that $k$ equals $l$,
by the characterisation of the two-coloured partition category for $U(n)$,
but also that
the $n^k \times n^k$ weight matrix is 
the image of all permutations in the symmetric group $S_k$, 
expressed as $(k,k)$--partition diagrams. 
We see that this is the classic version of Schur--Weyl duality.

\begin{example}
We calculate the weight matrices that appear in compact matrix quantum group 
equivariant neural networks for the unitary group $U(n)$
and unitary quantum group $U(n)^{+}$
when $n = k = l = 2$.

The easiest way to do this is to consider the one-coloured partitions that come in pairs and then superimpose colours onto each row of vertices to see which two-coloured partitions match the classification result.

The one-coloured $(2,2)$--partition diagrams that come in pairs are
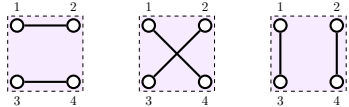
\begin{equation}
	\begin{aligned}
		\scalebox{0.6}{\begin{tikzpicture}
	\begin{pgfonlayer}{nodelayer}
		\node [style=none] (0) at (14.5, -4.25) {\Large$3$};
		\node [style=node] (1) at (14.5, 0) {};
		\node [style=node] (2) at (14.5, -3) {};
		\node [style=node] (3) at (17.5, 0) {};
		\node [style=node] (4) at (17.5, -3) {};
		\node [style=none] (5) at (17.5, -4.25) {\Large$4$};
		\node [style=none] (6) at (14.5, 1.25) {\Large$1$};
		\node [style=none] (7) at (17.5, 1.25) {\Large$2$};
		\node [style=none] (8) at (0.5, -4.25) {\Large$3$};
		\node [style=node] (9) at (0.5, 0) {};
		\node [style=node] (10) at (0.5, -3) {};
		\node [style=node] (11) at (3.5, 0) {};
		\node [style=node] (12) at (3.5, -3) {};
		\node [style=none] (13) at (3.5, -4.25) {\Large$4$};
		\node [style=none] (14) at (0.5, 1.25) {\Large$1$};
		\node [style=none] (15) at (3.5, 1.25) {\Large$2$};
		\node [style=none] (16) at (7.5, -4.25) {\Large$3$};
		\node [style=node] (17) at (7.5, 0) {};
		\node [style=node] (18) at (7.5, -3) {};
		\node [style=node] (19) at (10.5, 0) {};
		\node [style=node] (20) at (10.5, -3) {};
		\node [style=none] (21) at (10.5, -4.25) {\Large$4$};
		\node [style=none] (22) at (7.5, 1.25) {\Large$1$};
		\node [style=none] (23) at (10.5, 1.25) {\Large$2$};
		\node [style=none] (24) at (14, 0.5) {};
		\node [style=none] (25) at (14, -3.5) {};
		\node [style=none] (26) at (18, -3.5) {};
		\node [style=none] (27) at (18, 0.5) {};
		\node [style=none] (28) at (7, 0.5) {};
		\node [style=none] (29) at (7, -3.5) {};
		\node [style=none] (30) at (11, -3.5) {};
		\node [style=none] (31) at (11, 0.5) {};
		\node [style=none] (32) at (0, 0.5) {};
		\node [style=none] (33) at (0, -3.5) {};
		\node [style=none] (34) at (4, -3.5) {};
		\node [style=none] (35) at (4, 0.5) {};
	\end{pgfonlayer}
	\begin{pgfonlayer}{edgelayer}
		\draw [style={dash_2}] (24.center) to (25.center);
		\draw [style={dash_2}] (25.center) to (26.center);
		\draw [style={dash_2}] (26.center) to (27.center);
		\draw [style={dash_2}] (27.center) to (24.center);
		\draw [style={dash_2}] (28.center) to (29.center);
		\draw [style={dash_2}] (29.center) to (30.center);
		\draw [style={dash_2}] (30.center) to (31.center);
		\draw [style={dash_2}] (31.center) to (28.center);
		\draw [style={dash_2}] (32.center) to (33.center);
		\draw [style={dash_2}] (33.center) to (34.center);
		\draw [style={dash_2}] (34.center) to (35.center);
		\draw [style={dash_2}] (35.center) to (32.center);
		\draw [style=thick] (9) to (11);
		\draw [style=thick] (10) to (12);
		\draw [style=thick] (17) to (20);
		\draw [style=thick] (18) to (19);
		\draw [style=thick] (1) to (2);
		\draw [style=thick] (4) to (3);
	\end{pgfonlayer}
\end{tikzpicture}}
	\end{aligned}
\end{equation}
We label the diagrams as $1$ to $3$ from left to right.
Note that the matrices obtained under $d_\pi \mapsto D_\pi$ are the same no matter the colours on the vertices,
hence we can reuse the matrices for diagrams six, seven and eight from Figure \ref{easyorthogqc} in order to calculate the weight matrices.

We now consider pairs of words. As there are only $4$ words of length $2$, this implies that there are $16$ such pairs of words.
We obtain the following result for the weight matrices for $U(2)$, where for each cell, we take weighted linear combinations of the matrices $D_\pi$ for each set partition diagram in the cell.
\begin{equation} \label{unclassification}
	\NiceMatrixOptions{code-for-first-row = \scriptstyle \color{black},
			   code-for-first-col = \scriptstyle \color{black}
	}
	\renewcommand{\arraystretch}{1.5}
	\begin{bNiceArray}{*{1}{c}|*{1}{c}|*{1}{c}|*{1}{c}}[first-row,first-col]
		(I,J) & \circ\circ 	& \circ\bullet & \bullet\circ & \bullet\bullet 	\\
		\circ\circ		& 2,3 	& \varnothing	& \varnothing	& 2,3	\\
		\cline{1-4}
		\circ\bullet		& \varnothing	& 1,3	& 1,2	& \varnothing	\\
		\cline{1-4}
		\bullet\circ		& \varnothing	& 1,2	& 1,3	& \varnothing	\\
		\cline{1-4}
		\bullet\bullet		& 2,3	& \varnothing	& \varnothing	& 2,3
	\end{bNiceArray}
\end{equation}
Similarly, for $U(2)^{+}$, the weight matrices are weighted linear combinations of the images of the following set partition diagrams, where we have removed the second diagram from (\ref{unclassification}).
\begin{equation}
	\NiceMatrixOptions{code-for-first-row = \scriptstyle \color{black},
			   code-for-first-col = \scriptstyle \color{black}
	}
	\renewcommand{\arraystretch}{1.5}
	\begin{bNiceArray}{*{1}{c}|*{1}{c}|*{1}{c}|*{1}{c}}[first-row,first-col]
		(I,J) & \circ\circ 	& \circ\bullet & \bullet\circ & \bullet\bullet 	\\
		\circ\circ		& 3 	& \varnothing	& \varnothing	& 3	\\
		\cline{1-4}
		\circ\bullet		& \varnothing	& 1,3	& 1	& \varnothing	\\
		\cline{1-4}
		\bullet\circ		& \varnothing	& 1	& 1,3	& \varnothing	\\
		\cline{1-4}
		\bullet\bullet		& 3	& \varnothing	& \varnothing	& 3
	\end{bNiceArray}
\end{equation}
Note that the top left hand cell of (\ref{unclassification}) provides a classification of the weight matrix appearing in the compact matrix \textit{group} equivariant neural network for $U(2)$ when $n = k = l = 2$.
\end{example}


\end{document}